\documentclass[12pt,draftcls,onecolumn]{IEEEtran}
\usepackage{CJK}
\usepackage{amssymb,amsmath,amsopn,amsfonts,graphicx,color,amsthm,mathrsfs,color}
\usepackage[sort&compress, numbers]{natbib}
\usepackage{graphicx}
\usepackage{mathtools}
\usepackage{float}
\usepackage{multirow}
\usepackage{textcomp}
\usepackage{epstopdf}
\usepackage[justification=centering]{caption}
\usepackage{subfigure}
\usepackage{array}
\usepackage{longtable}
\usepackage{algorithm}

\allowdisplaybreaks[4]
\usepackage{algpseudocode}

\renewcommand{\algorithmicrequire}{\textbf{Input:}}  
\renewcommand{\algorithmicensure}{\textbf{Output:}}

\newtheorem{remark}{Remark}[section]
\newtheorem{theorem}{Theorem}[section]
\newtheorem{lemma}{Lemma}[section]
\newtheorem{proposition}{Proposition}[section]
\newtheorem{condition}{Condition}[section]
\newtheorem{assumption}{Assumption}[section]
\newtheorem{definition}{Definition}[section]
\newtheorem{corollary}{Corollary}[section]

\renewcommand\thesection{\arabic{section}}
\renewcommand\thesubsectiondis{\thesection.\arabic{subsection}}

\allowdisplaybreaks
\def\1{\boldsymbol{1}}

\def\B{\mathscr B}
\def\({\Big(}
\def\){\Big)}
\def\<{\langle}
\def\>{\rangle}

\def\dd{\,\text{d}}
\def\G{\mathcal G}
\def\L{\mathcal L}
\def\E{\mathbb E}
\def\G{\mathcal G}
\def\LL{\mathscr L}
\def\Gg{\mathscr G}
\def\P{\mathbb{P}}
\def\H{\mathcal H}
\def\HH{\mathscr H}
\def\F{\mathcal{F}}
\def\Ff{\mathscr F}
\def\N{\mathcal N}

\def\a{\alpha}
\def\b{\beta}

\def\ba{\begin{array}}
\def\ea{\end{array}}
\def\ban{\begin{eqnarray*}}
\def\ean{\end{eqnarray*}}
\def\bann{\begin{eqnarray*}}
\def\eann{\end{eqnarray*}}
\def\bnaa{\begin{eqnarray}}
\def\enaa{\end{eqnarray}}
\def\bd{\begin{description}}
\def\ed{\end{description}}
\def\be{\begin{equation}}
\def\ee{\end{equation}}
\def\bna{\begin{eqnarray}}
\def\ena{\end{eqnarray}}

\def\F{{\cal F}}

\def\X{\mathscr X}
\def\Y{\mathscr Y}
\def\Z{\mathscr Z}

\def\N{{\cal N}}

\ifCLASSINFOpdf
\else
\fi
\begin{document}
\begin{CJK}{GBK}{song}

\title{Decentralized Online Learning for Random Inverse Problems Over Graphs}
%
%
%

\author{Xiwei Zhang, Tao Li,  Yan Chen and Qianyuan Long
\thanks{This work was supported  by the National Natural Science Foundation
of China under Grant 62261136550. This paper was presented in
part at the  14th Asian Control Conference, Dalian, China, July 5-8, 2024. Corresponding author: Tao Li  (email: litao@amss.ac.cn).}
\thanks{Xiwei Zhang is with the No.2 High School of East China Normal University, Shanghai, China and he was with the School of Mathematical Sciences, East China Normal University.}
\thanks{Tao Li is with the Key Laboratory of Management, Decision and Information Systems, Institute of Systems Science, Academy of Mathematics and Systems Science, Chinese Academy of Sciences,  Beijing 100190, China, and also with School of
Mathematical Sciences, University of Chinese Academy of Sciences, Beijing 100149, China.}
\thanks{Yan Chen and Qianyuan Long were with the School of Mathematical Sciences, East China Normal University, Shanghai, China.}
}

%
%

\markboth{Journal of \LaTeX\ Class Files, June~2023}%
{Shell \MakeLowercase{\textit{et al.}}: Bare Demo of IEEEtran.cls for IEEE Journals}
%



\maketitle

\begin{abstract}
We propose a decentralized online learning algorithm for distributed random inverse problems over network graphs with online measurements, and  unify the distributed parameter estimation in Hilbert spaces and the least mean square problem in reproducing kernel Hilbert spaces (RKHS-LMS). We transform the convergence of the algorithm into the asymptotic stability of a class of inhomogeneous random difference equations in Hilbert spaces with $L_{2}$-bounded martingale difference terms and develop the $L_2$-asymptotic stability theory in Hilbert spaces. We show that if the network graph is connected and the sequence of forward operators satisfies the \emph{infinite-dimensional spatio-temporal persistence of excitation} condition, then the estimates of all nodes are mean square and almost surely strongly consistent. Moreover, we propose a decentralized online learning algorithm in RKHS based on non-stationary online data streams, and prove that the algorithm is mean square and almost surely strongly consistent if the operators induced by the random input data satisfy the \emph{infinite-dimensional spatio-temporal persistence of excitation} condition.
\end{abstract}


\section{Introduction}
\label{sec:introduction}
\IEEEPARstart{I}{nverse}  problems have wide applications such as medical imaging, geophysics and oil exploration (\cite{bertero}). An inverse problem is to determine the system input (cause) from the  system output (result).
In reality, measurements are usually affected by external disturbances, and the  inverse problems with noisy measurements have been widely studied including the cases with deterministic noises  (\cite{Tikhonov123}) and those with Gaussian white noises   (\cite{Bissantz}).
It is of practical significance to consider inverse problems with both randomly time-varying forward operators and random measurement noises. For example, consider the online learning problem in RKHS. Let $\X\subseteq \mathbb R^n$ be the input space and $(\HH_K,\langle \cdot,\cdot \rangle _K)$ be the Hilbert space with Mercer kernel $K:\X\times \X \to \mathbb R$. At time instant $k$, the random (with unknown distribution) input data $x(k)\in \X$ and the output data $y(k)\in \mathbb R$ satisfy the measurement equation $y(k)=f_0(x(k))+v(k),\ k\ge 0$, where $f_0\in \HH_K$ is the unknown function, and $v(k)\in \mathbb R$ is the random measurement noise. The online learning problem in $\HH_K$ is to reconstruct $f_0$ based on the online data stream $\{(x(k),y(k))\}_{k=0}^{\infty}$. By the reproducing property of RKHS, the above measurement equation can be written as
\begin{equation}\label{0}
y(k)=H(k)f_0+v(k),~k\ge 0,
\end{equation}
where $H(k)$ is the randomly time-varying forward operator induced by the input data $x(k)$,
satisfying $H(k)f:=\langle f,K(x(k),\cdot)\rangle _K$, $\forall\ f\in \HH_K$.
Thus, the online learning problem in RKHS  comes down to an inverse problem associated with the measurement model (\ref{0}). Besides, Dynamic inverse problems in \cite{Burger} , including  dynamic computerized tomography, dynamic load monitoring and magnetic particle imaging are the inverse problems with time-invariant forward operators.
Most of the existing  works on statistical inverse problems assumed the forward operators to be deterministic and time-invariant, which can not cover the measurement model (\ref{0})  (\cite{Bissantz}, \cite{Lu}-\cite{JinB}).

In addition to online data streams, many practical problems are required to be solved in a decentralized or distributed information structure. The overall large amounts of data are usually divided into several data sets, and the learning process is performed with multiple parallel processors (\cite{Rosenblatt}).
Decentralized online learning algorithms
for finite-dimensional parameter estimation have been widely studied. Pioneering works on the decentralized online parameter estimation in finite-dimensional spaces were achieved in \cite{Lopes}-\cite{kar2011} and fruitful results were obtained in \cite{Lopes}-\cite{Ishihara}.
Specifically, the decentralized online learning algorithms with randomly spatio-temporal independent observation matrices were proposed in \cite{Lopes}-\cite{Cattivelli} via the collaborative strategy of diffusion. Kar and Moura   established a distributed observability condition in  \cite{kar2011}.

Infinite-dimensional supervised online learning in RKHS is another important topic of random inverse problems (\cite{Lyaqini}). Based on the systematic study of batch learning in \cite{Poggio}, rich results of online learning algorithms based on i.i.d. online data streams were obtained in \cite{Ying}-\cite{Deng}.
Ying and Pontil (\cite{Ying}) considered the least-square online gradient descent
algorithm in RKHS and presented a novel capacity independent approach to derive error bounds and convergence results for this algorithm.
Tarr$\grave{\text{e}}$s and Yao (\cite{Tarres}) showed that the online regularized algorithm can achieve the strong convergence rate of batch learning, and the weak convergence rate is optimal in the sense that it reaches the minimax and individual lower rates.
Dieuleveut and Bach (\cite{Dieuleveut}) considered the random-design least-squares regression problem within the RKHS framework, and showed that the averaged unregularized LMS algorithm with a given sufficient large step-size can attain optimal rates of convergence for a variety of regimes for the smoothnesses of the optimal prediction function in RKHS.
Shin \emph{et al}. (\cite{Shin}) proposed decentralized adaptive learning algorithms over graphs in RKHS in a deterministic framework. Deng \emph{et al}. (\cite{Deng}) used the multiplicative operator in the saddle point problem to carve out the communication structure of decentralized networks and proposed a distributed consensus-based online learning algorithm with  i.i.d. measurements. The nonlinear online learning problems in RKHS with spatio-temporal independent measurements were studied in  \cite{Mitra}-\cite{SmaleYao}.

Up to now,
the existing theory of inverse problems in statistical and stochastic frameworks is far from mature and the existing related results may be divided into three categories: (i) statistical inverse problems based on deterministic time-invariant compact forward operators in Hilbert spaces (\cite{Bissantz},\cite{Lu}-\cite{JinB}); (ii) distributed parameter estimation in finite-dimensional spaces (\cite{Lopes}-\cite{Ishihara}); (iii) decentralized learning based on stationary data, e.g., i.i.d. measurements in RKHS (\cite{Shin}-\cite{Bouboulis}). Some basic problems are still open, such as
\begin{itemize}
\item inverse problems with randomly time-varying forward operators;
\item to establish a unified framework for random inverse problems in infinite-dimensional Hilbert spaces, distributed parameter estimation problems in finite-dimensional spaces and online learning problems in RKHS;
\item to develop a decentralized RKHS learning theory based on non-stationary data streams.
\end{itemize}

Motivated by the above problems, we consider a class of random inverse problems over graphs, establish a unified framework to deal with the above three types of problems, and propose a decentralized online learning algorithm in Hilbert spaces. The learning algorithm of each node in the network takes the form of consensus plus innovation as in \cite{kar20132}. The innovation term is to update the node's estimate by using the node's own measurement data, and the consensus term is a weighted sum of its own estimate and the estimates of its neighboring nodes. The forward operator of the measurement of each node is randomly time-varying and is not required to satisfy special statistical properties, such as temporal independence (the forward operator of each node over the graph is independent with respect to time), spatial independence (the forward operators of different nodes are independent of each other at each moment) and stationarity, and the random measurement noise is no longer restricted to  Gaussian white noise.

We weaken the constraints on the forward operators compared with the most of the existing studies on inverse problems (\cite{Bissantz}-\cite{JinB}). We consider the cases with general bounded linear forward operators instead of compact ones. Besides, we allow the forward operators to be randomly time-varying, which are not restricted to be deterministic and time-invariant. These general settings bring essential difficulties to the convergence  analysis of the algorithm.
Tarr$\grave{\text{e}}$s and Yao (\cite{Tarres}), Smale and Yao (\cite{SmaleYao}) transformed the online learning problem with i.i.d. data streams in RKHS into the inverse problem with the deterministic time-invariant Hilbert-Schmidt forward operator, and then obtained the convergence result by using the singular value decomposition (SVD) of the linear compact operators in the Hilbert space.
By using the i.i.d. properties of the data, Dieuleveut and Bach (\cite{Dieuleveut}) transformed the estimation error equation into the random difference equation equivalently, where the homogeneous part is deterministically time-invariant and the inhomogeneous part is the martingale difference sequence in the Hilbert space, from which the mean-square convergence of the algorithm was obtained by means of the spectral decomposition property of the compact operator. Note that the SVD of linear compact operators in Hilbert spaces is only applicable for the inverse problems with deterministic and time-invariant linear compact forward operators. The existing  methodologies in  \cite{Math1}-\cite{Jonesfg}, \cite{Tarres}-\cite{Dieuleveut} and \cite{SmaleYao} are no longer applicable for our problems.

Note that for the decentralized online learning in finite-dimensional spaces over random graphs based on non-stationary data, we  have established the stochastic spatio-temporal persistence of excitation (SSTPE) condition to ensure convergence of the algorithm in \cite{WLZ}.  However, for this kind of finite-dimensional excitation conditions, the information matrices are all required to be positive definite, i.e., the eigenvalues of the matrix have strictly positive lower bounds. Obviously, the SSTPE  condition  is not applicable for inverse problems in infinite-dimensional Hilbert spaces. It is known that even for a strictly positive compact operator,  the infimum of its eigenvalues is zero, and then the SSTPE  condition can not hold.

To this end, by means of measurability and integration theory of mappings with values in Banach spaces, spectral decomposition theory of bounded self-adjoint operators, and martingale convergence methods, we investigate the $L_p^q$-stability condition on the sequence of operator-valued random elements.
We transform the convergence analysis of the algorithm into the $L_2$-asymptotic stability of time-varying random difference equations in Hilbert spaces. Since the forward operators of inverse problems in infinite-dimensional Hilbert spaces usually do not have bounded inverses, the existing asymptotic stability theory on infinite-dimensional random difference equations with compressive operators in Hilbert spaces  cannot be applied to inverse problems (\cite{Ungureanu1}-\cite{zwzwxcbs}). To this end, we propose the $L_p^q$-stability condition on the sequence of the products of operator-valued random elements, and establish the $L_2$-asymptotic stability theory for a class of inhomogeneous  random difference equations in Hilbert spaces with $L_2$-bounded martingale difference terms. We give sufficient conditions on the stability of a class of operator-valued random sequences composed of forward operators and Laplacian matrices of graphs. We prove that if the graph is connected, and the sequence of forward operators satisfies the \emph{infinite-dimensional spatio-temporal persistence of excitation} condition, then all nodes' estimates are all mean square and almost surely strongly consistent with the unknown function of the inverse problem.

We develop a theory of decentralized online learning in RKHS. Almost the existing literature on online supervised learning in RKHS  are based on i.i.d. data (e.g.,  \cite{GUO}-\cite{Lin}),  while we propose a decentralized online learning algorithm based on non-stationary and non-independent data streams in RKHS. We establish the convergence condition by equivalently transforming the distributed learning problem in RKHS into the random inverse problems over graphs. Especially, if the graph has only one node and the random input data is i.i.d., then our algorithm degenerates to the centralized online learning algorithm without regularization parameters in \cite{Ying} and \cite{Dieuleveut}.

The rest of the paper is organized as follows. In Section II, the online random inverse problems over   graphs in  Hilbert spaces are formulated and the decentralized online learning algorithm is proposed. In Section III, the convergence of the algorithm is proved. In Section IV, the decentralized online learning problems in   RKHS  are studied. In Sections V and VI, simulation results  and  conclusions are given, respectively. The proofs of theorems and corollaries in Sections III and IV are in Appendix A. The proofs of all lemmas in  Sections III and IV  can be found in Appendix C.

The following notations will be used throughout the paper.   For a subset $A$  of a set $X$, if $x \in A$, then $\mathbf{1}_A(x)=1$, and if $x \notin A$, then $\mathbf{1}_A(x)=0$. Denote $f\in \F_0$ if $f$ is measurable with respect to the $\sigma$-algebra $\F_0$. For any given sets $\{A_i,i\in \mathscr I\}$,
where $\mathscr I$ is a set of indices, denote $\sigma\left(\bigcup_{i\in \mathscr I}A_i\right)$ by $\bigvee_{i\in \mathscr I}A_i$. Let $\tau_{\text{N}}(\X)$ be the topology induced by the norm $\|\cdot\|_{\X}$ in a Banach space $\X$. Denote $(\Omega,\F,\mathbb P)$ as a complete probability space.
 Let $L^0(\Omega;\X)$ be a linear space composed of all mappings which take values in $(\X,\tau_{\text{N}}(\X))$ and are strongly $\P$-measurable  with reference to $(\Omega,\F,\mathbb P)$. In particular, for  a sub-$\sigma$-algebra $\Gg$ of $\F$, $L^0(\Omega,\Gg;\X)$ is defined with reference to  $(\Omega,\Gg,\mathbb{P}|_{\Gg})$.
Let $L^p(\Omega;\X)=\{f\in L^0(\Omega;\X): \|f\|_{L^p(\Omega;\X)}<\infty \}$, where
$\|f\|_{L^p(\Omega;\X)}:=\left(\int_{\Omega}\|f\|_{\X}^p\dd\P\right)^{\frac{1}{p}}$, $1\leq p<\infty$.
Let $\mathscr L(\mathscr Y,\mathscr Z)$ be the linear space of all bounded linear operators from the Banach space $\mathscr Y$ to the Banach space $\mathscr Z$, in particular, $\mathscr L(\mathscr Z):=\mathscr L( \mathscr Z,\mathscr Z)$. Let $\tau_{\text{S}}(\mathscr L(\mathscr Y, \mathscr Z))$ be the strong operator topology of $\mathscr L(\mathscr Y,\mathscr Z)$. 

Let $(\X_i,\langle \cdot,\cdot\rangle _{\X_i})$ be a Hilbert space, where the norm induced by the inner product is defined by $\|x_i\|_{\X_i}:=\sqrt{\langle x_i,x_i\rangle _{\X_i}}$, $x_i\in \X_ i$, $i=1,\cdots,n$. The Hilbert direct sum space is denoted by $\bigoplus_{i=1}^n\X_i=\left\{x=(x_1,\cdots,x_n):x_i\in \X_i,1\leq i\leq n\right\}$,
where the inner product is defined by $\langle x, y\rangle _{\bigoplus_{i=1}^n\X_i}:=\sum_{i=1}^n\langle x_i,y_i\rangle _{\X_i}$,$\forall\ x=(x_1,\cdots,x_n),y=(y_1,\cdots,y_n)\in \bigoplus_{i=1}^n\X_i$. Denote $\bigoplus_{i=1}^n\X:=\X^n$. Let $\X$ and $\mathscr Y$ be Hilbert spaces.  Denote the Kronecker product of the vector $\1_n\in \mathbb R^{n}$ and $ f\in \X$ by $\1_n\otimes f:=(f,\cdots,f)\in \X^n$ and the Kronecker product of the matrix $A\in \mathbb R^{n\times m}$ and $B\in \LL(\X,\Y)$ by
\ban
A\otimes B:=\begin{pmatrix}
a_{11}B & \cdots & a_{1m}B \\
\vdots & \ddots & \vdots \\
a_{n1}B & \cdots & a_{nm}B
\end{pmatrix}
\in \mathscr L(\X^m,\mathscr Y^n).
\ean
The operations and properties of operator matrices in Hilbert direct sum spaces can be found in \cite{sl}. For   $T \in \LL(\X,\Y)$,
if  an operator $T^*: \Y \to \X$ satisfies $
\langle Tx,y\rangle_{\Y}=\langle x, T^*y\rangle_{\X}$, $\forall \ x \in \X,y\in \Y$, then $T^*$ is called the adjoint operator of $T$.
For $T\in \LL(\X)$, if $T^*=T$, i.e., $\langle Tx,y \rangle _{\X}=\langle x, Ty \rangle _{\X}$, $\forall\ x,y\in \X$, then $T$ is called a self-adjoint operator. If a self-adjoint operator $T\in \LL(\X)$ satisfies $\langle Tx,x \rangle _{\X}\ge 0$, $\forall\ x\in \X$, then $T$ is  positive self-adjoint, denoted by $T\ge 0$. Especially if $\langle Tx,x \rangle _{\X}> 0$ for any given $x\neq 0$ in $\X$, then $T$ is strictly positive self-adjoint, denoted by $T>0$.


Let $\mathcal{G}=\{\mathcal{V},\mathcal{E}_{\mathcal{G}},\mathcal{A}_{\mathcal{G}}\}$ denote a weighted graph, where $\mathcal{V}=\{1,...,N\}$ is the set of nodes and $\mathcal{E}_{\mathcal{G}}$ is the set of edges. The unordered pair $(j,i)\in\mathcal{E}_\mathcal G$ if and only if there exists an edge between nodes $j$ and $i$. Denote the set of neighboring nodes of node $i$ by $\mathcal{N}_i=\{j\in \mathcal{V}:(j,i)\in
\mathcal{E}_{\mathcal{G}}\}$. The matrix $\mathcal{A}_{\mathcal{G}}=[a_{ij}]\in \mathbb{R}^{N\times N}$ is called the weighted adjacency matrix of $\G$, and for any given  $i,j\in \mathcal V$, $a_{ii}=0$ and $a_{ij}=a_{ji}>0$ if and only if $j\in \mathcal{N}_i$.
The Laplacian matrix of $\G$ is defined by $\mathcal{L}_{\mathcal{G}}=\mathcal{D}_{\mathcal{G}}-\mathcal{A}_{\mathcal{G}}$, where the degree matrix $\mathcal{D}_{\mathcal{G}}=\text{diag}\{\sum_{j=1}^Na_{1j},\sum_{j=1}^Na_{2j},\cdots,\sum_{j=1}^Na_{Nj}\}$.

\section {Online Learning for Random Inverse Problems over Graphs}

\subsection{Online Random Inverse Problems over Graphs}
Consider a distributed communication network modeled by a weighted  graph $\mathcal{G}=\{\mathcal{V},\mathcal{E}_{\mathcal{G}},\mathcal{A}_{\mathcal{G}}\}$ consisting of $N$ nodes. The measurement $y_i(k)$ of node $i$ at instant $k$ is given by
\begin{equation}\label{measuramentmodel}
y_i(k)=H_i(k)f_0+v_i(k),\ k=0,1,2,...\ i=1,\cdots,N,
\end{equation}
where $f_0\in \X$ to be estimated is an unknown element in a separable Hilbert space $\X$, the random forward operator  $H_i(k):\Omega\to \mathscr L(\X,\Y_i)$ is an operator-valued random element with values in $(\LL(\X,\Y_i),\tau_{\text{S}}(\LL(\X,\Y_i)))$, and the measurement noise $v_i(k):\Omega\to \Y_i$ is a random element with values in a separable Hilbert space $(\Y_i,\tau_{\text{N}}(\Y_i))$.

\begin{remark}
\label{remark1adda11}
For the finite-dimensional Euclidean space $\X=\mathbb R^n$, the measurement model (\ref{measuramentmodel}) has been widely studied in \cite{Lopes}-\cite{Ishihara}, where the forward operator $H_i(k)$ degenerates to the random observation matrix and $f_0\in \mathbb R^n$ degenerates to the unknown finite-dimensional parameter vector.  Kar \emph{et al}. (\cite{kar2012}) investigated the decentralized estimation algorithms with nonlinear measurement models, where they introduced separably estimable measurement models that generalize the observability condition in linear centralized estimation to nonlinear decentralized estimation. It is worth pointing out that even for the nonlinear measurement model in \cite{kar2012}, the unknown quantity to be estimated is a parameter vector in a finite-dimensional space, different from which, the unknown quantity $f_0$ in the measurement model (\ref{measuramentmodel}) can be a nonlinear function, which is an element in the infinite-dimensional Hilbert space $\X$. The measurement model (\ref{measuramentmodel}) is essentially different from that in \cite{kar2012}. This will be further clarified in Remark \ref{remarkaddadd6}.
\end{remark}

Denote $y(k)=(y_1(k),\cdots,y_N(k)):\Omega\to \bigoplus_{i=1}^N\Y_i$, $ v(k)=(v_1(k),\cdots,v_N(k)):\Omega\to \bigoplus_{i=1}^N\Y_i$ and $ H(k) = (H_1(k),\cdots,H_N(k)):\Omega\to \mathscr L(\X,\bigoplus_{i=1}^N\Y_i)$.
We can write (\ref{measuramentmodel}) in the compact form
\begin{equation}\label{compactform}
y(k)=H(k)f_0+v(k),\ k\ge 0.
\end{equation}
The online random inverse problem means reconstructing $f_0$ by real-time random measurements $\{y(k), k\geq0\}$.

The measurement equation (\ref{compactform}), which covers the existing models of inverse problems, is more general in the sense that the forward operator can be randomly time-varying. Besides, different from  the existing literature, which has a centralized information structure, here, the reconstruction of $f_0$ is constrained by the information structure of the graph, i.e., there is no centralized fusion center collecting the overall measurements $y(k)$, and at each moment $k$, node $i$ can only use its own observation $y_i(k)$ and its neighbors' estimates $f_j(k), j\in \mathcal N_i$ to give its next estimate $f_i(k+1):\Omega\to \X$ for $f_0$, i.e.,
$$f_i(k)\in \left(\bigvee_{s=0,1,2,...,k-1}\sigma(y_i(s);\tau_{\text{N}}(\Y_i))\right)\bigvee\left(\bigvee_{j\in \mathcal N_i\cup\{i\}}\sigma\left(f_j(k-1);\tau_{\text{N}}(\X)\right)\right),~i\in \mathcal V.$$

Here,  the problem of cooperatively estimating $f_0$ by the nodes over the graph based on each node's local measurements, its own and neighbors' estimates is called the \textbf{\emph{random inverse problem over the graph}}.

\begin{remark}
Most of the existing literature on inverse problems assumes the forward operator to be a deterministic and time-invariant linear compact operator $H$, associated with the measurement equation 
\begin{equation}\label{nklowokkrkr}
y=Hf_0+v.
\end{equation}
Here, different from the existing literature, the forward operator in the measurement equation (\ref{compactform}) is allowed to be random and time-varying. In the classical inverse problem, the noise $v$ in the measurement equation (\ref{nklowokkrkr}) is modeled as a deterministic perturbation  (\cite{Benning}). In the statistical inverse problem, it is modeled as Gaussian white noise (\cite{Lu}-\cite{Math1}).
Based on the stochastic gradient descent (SGD) algorithm, Lu and Math\'{e} (\cite{Lu100}), Jahn and Jin (\cite{Jahn}) and Jin and Lu (\cite{JinB}) obtained the centralized learning strategy by the random discrete sampling of the forward operator and minimizing the loss functional
\bna\label{sunshi}
\widetilde{J}(x)=\frac{1}{2}\E\left[\|y-Hx\|^2\right],\ \forall\ x\in \X.
\ena
Specifically, Jahn and Jin (\cite{Jahn}) and Jin and Lu (\cite{JinB}) investigated the regularization property of SGD with a priori and a posteriori stopping rules, and Lu and Math\'{e} (\cite{Lu100}) gave an upper bound on the estimation error of SGD with discrete  level. Recently, Iglesias \emph{et al.} (\cite{Iglesias1}) and Lu \emph{et al.}  (\cite{Lu1}) solved the statistical inverse problem based on real-time measurements $\{y(k):y(k)=Hf_0+v(k),k\ge 0\}$, where $\{v(k),k\ge 0\}$ is an i.i.d. Gaussian white noise sequence. The statistical inverse problems in \cite{Iglesias1}-\cite{Lu1} are the special cases of the online random inverse problem with random forward operators.
\end{remark}

\subsection{Decentralized Online Learning Algorithm}\label{online}
Denote $\mathcal H(k)=\text{diag}\{H_1(k),\cdots,H_N(k)\}$.  Based on the loss functional (\ref{sunshi}), we consider minimizing the loss functional $J:\X^N \to \mathbb R$ with the Laplacian regularization term given by
\bna
\label{costfunctionLaplacianregularization}
J(f)=\frac{1}{2}\left(\E\Big[\|y(k)-\mathcal H(k)f\|_{\bigoplus_{i=1}^N\Y_i}^2\Big]+\left\langle \left(\L_{\G}\otimes I_{\X}\right )f,f\right\rangle _{\X^N}\right),~\forall\ f\in \X^N,
\ena
where $I_{\X}$ is the identical operator on the Hilbert space $\X$. The loss functional  $J(f)$ consists of two terms: the mean-square estimation error term $\E\left[\|y(k)-\mathcal H(k)f\|_{\bigoplus_{i=1}^N\Y_i}^2\right]$ and the Laplacian regularization term $\langle (\L_{\G}\otimes I_{\X})f,f\rangle _{\X^N}=\frac{1}{2}\sum_{i=1}^N\sum_{j=1}^Na_{ij}\|f_i-f_j\|^2_{\X}$, where $f_i\in \X$ is the $i$-th component of $f$.

Suppose $\H(k)\in L^2(\Omega;\mathscr L(\X^N,\bigoplus_{i=1}^N\Y_i))$ and $v(k)\in L^2(\Omega;\bigoplus_{i=1}^N\Y_i)$.
If the sequences $\{\H(k),k\ge 0\}$ and $\{v(k), k\ge 0\}$ are both i.i.d, then it follows from Definition \ref{fenbudingyi} and Proposition \ref{nlllwwieiie}.(a) that $\{y(k)-\H(k)f,k\ge 0\}$ is a sequence of i.i.d. random elements with values in the Hilbert space $(\bigoplus_{i=1}^N\Y_i,\tau_{\text{N}}(\bigoplus_{i=1}^N\Y_i))$.
It can be verified that
  $\H^*(k)\H(k)\in L^1(\Omega;\LL(\X^N))$. Then, by Proposition \ref{vnwlssfweewwfew}, the gradient operator $\text{grad}~J:\X^N\to \X^N$ is given by
\ban
\text{grad}~J(f)&=&\frac{1}{2}\text{grad}~\left(\E\Big[\|y(k)-\mathcal H(k)f\|_{\bigoplus_{i=1}^N\Y_i}^2\Big]+\left\langle \left(\L_{\G}\otimes I_{\X}\right)f,f\right\rangle _{\X^N}\right)\\
&=&\frac{1}{2}\text{grad}~\E\left[\langle \H^*(k)\H(k)f,f\rangle _{\X^N}\right]-\text{grad}~\E\Big[\langle \H(k)f,y(k) \rangle _{\bigoplus_{i=1}^N\Y_i}\Big]\\ &&+\frac{1}{2}
\text{grad}~\left\langle \left(\L_{\G}\otimes I_{\X}\right)f,f\right\rangle _{\X^N}\\
&=&\frac{1}{2}\text{grad}~\langle \E\left[\H^*(k)\H(k)\right]f,f\rangle _{\X^N}-\text{grad}~\E\left[\langle \H(k)f,y(k) \rangle _{\bigoplus_{i=1}^N\Y_i}\right]\\
&&+\frac{1}{2}
\text{grad}~\left\langle \left(\L_{\G}\otimes I_{\X}\right)f,f\right\rangle _{\X^N}.
\ean
Similarly, we have
\ban
\langle \E[\H^*(k)\H(k)]x,y\rangle _{\X^N}&=&\E[\langle \H^*(k)\H(k)x,y\rangle _{\X^N}]\\ &=&\E[\langle x,\H^*(k)\H(k)y\rangle _{\X^N}]\\
&=&\langle x, \E[\H^*(k)\H(k)]y\rangle _{\X^N},\ \forall\ x, y\in \X^N.
\ean
Therefore, $\E[\H^*(k)\H(k)]:\X^N\to \X^N$ is a self-adjoint operator. Noting that the Laplacian matrix $\L_{\G}$ is positive semi-definite, it follows that $\L_{\G}\otimes I_{\X}$ is a self-adjoint operator. Then
$$
\text{grad}~\langle \E[\H^*(k)\H(k)]f,f\rangle _{\X^N}=2\E[\H^*(k)\H(k)]f,~\forall\ f\in \X^N,
$$
and $$\text{grad}~\langle (\L_{\G}\otimes I_{\X})f,f\rangle _{\X^N}=2(\L_{\G}\otimes I_{\X})f,~\forall\ f\in \X^N.$$
Noting that $\H^*(k)y(k)\in L^1(\Omega;\X^N)$, it follows that
\begin{align}
&\lim_{t\to 0}\frac{1}{t}\left(\E\left[\langle \H(k)(f+tg),y(k) \rangle _{\bigoplus_{i=1}^N\Y_i}\right]-\E\left[\langle \H(k)f,y(k) \rangle _{\bigoplus_{i=1}^N\Y_i}\right]\right)\cr
=&\E[\langle \H^*(k)y(k),g\rangle _{\X^N}]\cr  = & \langle \E[\H^*(k)y(k)],g\rangle _{\X^N},~\forall\ g\in \X^N.\notag
\end{align}
It follows that $\text{grad}~\E[\langle \H(k)f,y(k) \rangle _{\bigoplus_{i=1}^N\Y_i}]=\E[\H^*(k)y(k)]$. Thus, we have
$$
\text{grad}~J(f)=-\E[\H^*(k)(y(k)-\H(k)f)]+(\L_{\G}\otimes I_{\X})f,~\forall\ f \in \X^N.
$$
Then we have the stochastic gradient descent (SGD) algorithm in the Hilbert space
\bna\label{algorithm1}
f(k+1)=f(k)+a(k)\H^*(k)(y(k)-\H(k)f(k))-b(k)(\L_{\G}\otimes I_{\X})f(k),~k\ge 0,
\ena
where $a(k)$ and $b(k)$ are algorithm gains.
Let $f(k)=(f_1(k),\cdots,f_N(k))$. From (\ref{algorithm1}), we obtain the decentralized online learning algorithm
\bna\label{algorithm}
f_i(k+1)&=&f_i(k)+a(k)H_i^*(k)(y_i(k)-H_i(k)f_i(k)) \cr
&&+b(k)\sum_{j\in \N_i}a_{ij}(f_j(k)-f_i(k)),~k\ge 0,~i\in \mathcal V.
\ena

\begin{remark}
The algorithm (\ref{algorithm}) takes a form of ``consensus+innovations''. This kind of decentralized estimation strategies was first proposed in \cite{kar2012}-\cite{kar20132} for estimating the parameters in finite-dimensional spaces. Here,
 for the measurement model (\ref{measuramentmodel}), the learning strategy (\ref{algorithm}), which is a stochastic gradient decent algorithm for the Laplacian regularized loss functional (\ref{costfunctionLaplacianregularization}), is exactly the ``consensus+innovations'' type. The algorithm (\ref{algorithm}) can be regarded as the extention of the ``consensus+innovations'' type algorithm to the case of non-parametric or infinite-dimensional estimation.
\end{remark}

\begin{remark}
To implement the algorithm  (\ref{algorithm}), the terms $H_i^*(k)y_i(k)$ and $H_i^*(k)H_i(k) f_i(k)$ may be computed instead of computing $H_i^*(k)$. For some inverse problems, such as the deconvolution problem in \cite{Kaipio2006}  and recovering the unknown initial value in the heat equation with a random diffusivity coefficient  in \cite{Isakov2006}, these terms can be computed by the numerical integration scheme in \cite{Atkinson1997}. Besides, these  terms can be computed by the cubic spline interpolation method in decentralized online learning problems in RKHS. Generally speaking, the computational complexity of the two terms depends on  specific problems and the computational methods employed.

Besides, noting that the computations for each node  are  all in parallel  and there exists a  constant $d>0$ independent of the total number of nodes $N$ such that $\sup_{i=1,\ldots,N,\ N\geq  1}\left|\mathcal{N}_i\right| \leq  d$, it follows that the computational complexity of the algorithm  (\ref{algorithm}) is $O(1), \ N \rightarrow \infty$.
\end{remark}

\section{Convergence Analysis}

Although the algorithm (\ref{algorithm1}) is designed by assuming that $\{\H(k)\in L^2(\Omega;\mathscr L(\X^N,\bigoplus_{i=1}^N\Y_i)),\\k\ge 0\}$ and $\{v(k)\in L^2(\Omega;\bigoplus_{i=1}^N\Y_i), k\ge 0\}$ are both i.i.d, and $\H(k),\ k\ge 0$ are  random elements with values in $   (\mathscr L(\X^N,\bigoplus_{i=1}^N\Y_i), \tau_{\text{N}}(\mathscr L(\X^N,\bigoplus_{i=1}^N\Y_i)))$.
In fact, in this section, we will show that even for the non-independence and non-stationarity sequence of operator-valued random elements $\{\H(k),k\ge 0\}$ with values in  $ (\mathscr L(\X^N,\bigoplus_{i=1}^N\Y_i), \tau_{\text{S}}(\mathscr L(\X^N,\bigoplus_{i=1}^N\Y_i)))$ and the noise sequence $\{v(k),k\ge 0\}$, the algorithm (\ref{algorithm}) still converges under mild conditions.

Denote the global estimation error by $e(k)=f(k)-\1_N\otimes f_0$. Note that $(\L_{\G}\otimes I_{\X})(\1_N\otimes f_0)=0$ and $\H(k)(\1_N\otimes f_0)=H(k)f_0$. Subtracting $\1_N\otimes f_0$ on both sides of equation (\ref{algorithm1}) yields
\begin{align}\label{error}
&e(k+1)\notag\\
=&(I_{\X^N}-a(k)\H^*(k)\H(k)-b(k)\L_{\G}\otimes I_{\X})f(k)+a(k)\H^*(k)y(k)-\1_N\otimes f_0 \notag\\
 =&(I_{\X^N}-a(k)\H^*(k)\H(k)-b(k)\L_{\G}\otimes I_{\X})(f(k)-\1_N\otimes f_0+\1_N\otimes f_0) \notag\\
&+a(k)\H^*(k)y(k)-\1_N\otimes f_0\notag\\
=&(I_{\X^N}-a(k)\H^*(k)\H(k)-b(k)\L_{\G}\otimes I_{\X})e(k)-(a(k)\H^*(k)\H(k) \notag\\
&+b(k)\L_{\G}\otimes I_{\X})(\1_N\otimes f_0)+a(k)\H^*(k)y(k) \notag\\
=&(I_{\X^N}-a(k)\H^*(k)\H(k)-b(k)\L_{\G}\otimes I_{\X})e(k) +a(k)\H^*(k)(y(k)-\H(k)(\1_N\otimes f_0))\notag\\
=&(I_{\X^N}-a(k)\H^*(k)\H(k)-b(k)\L_{\G}\otimes I_{\X})e(k)+a(k)\H^*(k)v(k).
\end{align}
The estimation error equation (\ref{error}) belongs to the following family of randomly time-varying difference equations
\bna \label{chafen}
x(k+1)=(I_{\X_1}-F(k))x(k)+G(k)u(k),~k\ge 0,
\ena
where $x(k)$ is a random element with values in a separable Hilbert space $(\X_1,\tau_{\text{N}}(\X_1))$, $u(k)$ is a random element with values in a separable Hilbert space $(\X_2,\tau_{\text{N}}(\X_2))$, $F(k):\Omega\to\mathscr L(\X_1)$ and $G(k):\Omega\to\mathscr L(\X_2,\X_1)$ are random elements with values in the topological spaces $(\LL(\X_1),\tau_{\text{S}}(\LL(\X_1)))$ and $(\LL(\X_2,\X_1),\tau_{\text{S}}(\LL(\X_2,\X_1)))$, respectively. To analyze the convergence of the estimation error equation (\ref{error}), we will first develop an asymptotic stability theory of the randomly time-varying difference equation (\ref{chafen}).

\subsection{Asymptotic Stability of Random Difference Equations in Hilbert Spaces}
To rigorously study the asymptotic stability of the random difference equations in the Hilbert space $(\X,\tau_{\text{N}}(\X))$, we introduce the following definitions.

 \vspace{-3mm}
\begin{definition}
If the sequence of random elements $\{x(k),k\ge 0\}$ with values in the Hilbert space $(\X,\tau_{\text{N}}(\X))$ satisfies $\sup_{k\ge 0}\E\left[\|x(k)\|_{\X}^p\right]<\infty$, where $p>0$, then $\{x(k),k\ge 0\}$ is said to be $L_p$-bounded.
\end{definition}
 \vspace{-6mm}
\begin{definition}
If the sequence of random elements $\{x(k),k\ge 0\}$ with values in the Hilbert space $(\X,\tau_{\text{N}}(\X))$ satisfies $\lim_{k\to \infty}\E\left[\|x(k)\|_{\X}^p\right]=0$, where $p>0$, then $\{x(k),k\ge 0\}$ is said to be $L_p$-asymptotically stable.
\end{definition}
\vspace{-6mm}

\begin{definition} \label{dingyi}
Let $\{A(k),k\ge 0\}$ be a sequence of operator-valued random elements with values in $(\LL(\X),\tau_{\text{S}}(\LL(\X)))$ and $\{\mathcal F(k), k\ge 0\}$ be a filter in $(\Omega,\F,\P)$. If for any given $L_q$-bounded adaptive sequence $\{x(k),\F(k), k\ge 0\}$ with values in the Hilbert space $\X$,
$$
\lim_{m\to \infty}\E\left[\left\|\prod_{k=n+1}^mA(k)x(n)\right\|_{\X}^p\right]=0,\ \forall\ n\geq0,\ \text{where}\ p,\ q>0,
$$
then $\{A(k), k\geq 0\}$ is said to be $L_p^q$-stable with respect to the filter $\{\mathcal F(k), k\geq0\}$.
\end{definition}

The strong convergence of the sequence of products of deterministic non-expansive operators has attracted the attentions of many scholars. By assuming strong convergence of operator products, the convergence results on infinite-dimensional deterministic time-varying difference equations in a general metric space were obtained (\cite{Reich3}-\cite{Pustylnik4}). Reich and Zaslavski (\cite{Reich3}) studied deterministic time-varying compressive operators in general metric spaces and obtained strong convergence results on the sequence of operator products; Pustylnik \emph{et al.}  (\cite{Pustylnik4}) studied the strong convergence of the sequence of operator products consisting of finite number of projection operators.

Noting that a sequence of deterministic operator products converging strongly to $0$ can be regarded as a $L_p^q$-stable sequence of operators w.r.t. the trivial filter $\{\F(k)=\{\emptyset,\Omega\},k\in \mathbb Z\}$, the concept of strong convergence for the sequence of operator products in the above literature can be regarded as a special case of Definition \ref{dingyi}. Besides, for the case in finite-dimensional spaces, Guo (\cite{Guo1994}) proposed the concept of $L_p$-exponentially stable random matrix sequence $\{I-B(k)\in \mathbb R^{N\times N},k\geq0\}$, i.e. there exist constants $M>0$ and $\lambda \in (0,1)$ such that
$$\E\left[\left\|\prod_{k=n+1}^m(I-B(k))\right\|_{\LL(\mathbb R^{N})}^p\right]\leq M\lambda^{m-n},\ \forall\ m> n\geq0.$$
From Definition \ref{dingyi}, we know that
 $\{I-B(k)\in \mathbb R^{N\times N},k\geq0\}$ is $L^{s}_{r}$-stable  w.r.t. any given filter $\{\mathcal F(k), k\geq0\}$, where $r=pa^{-1}$, $s=pba^{-1}$ and $a,b$ are positive real numbers with $a^{ -1}+b^{-1}=1$.

\vskip 0.2cm

Denote
$\F(k)=\bigvee_{i=0}^k(\sigma(F(i);\tau_{\text{S}}(\LL(\X_1)))\bigvee \sigma(G(i); \tau_{\text{S}}(\LL(\X_2,\X_1)))\bigvee \sigma(u(i);\tau_{\text{N}}(\X_2)))$, $k\\ \ge 0$ and $\F(-1)=\{\emptyset,\Omega\}$.
\vskip 0.2cm

For the $L_2$-asymptotic stability of the solution sequence of the random difference equation (\ref{chafen}), we have the following lemma.
\vskip 1mm

\begin{lemma}\label{wendingxing}
For the random difference equation (\ref{chafen}), let $\{u(k),\F(k),k\ge 0\}$ be a $L_2$-bounded sequence of martingale differences, and the sequence $\{u(k),k\ge 0\}$ be independent of  $\{F(k),k\ge 0\}$ and $\{G(k),k\ge 0\}$. If (i) $\{I_{\X_1}-F(k),k\ge 0\}$ is $L_2^2$-stable  w.r.t. $\{\F(k),k\ge 0\}$, (ii) there exists a sequence of nonnegative real numbers $\{\gamma(k),k\ge 0\}$ such that
\begin{align}\label{qqqqq}
\E\left[\|I_{\X_1}-F(k)\|_{\LL(\X_1)}^4 \Big|\mathcal F(k-1)\right]\leq 1+\gamma(k)~\text{a.s.},
\end{align}
$\sum_{k=0}^{\infty}\gamma(k) <\infty$, and (iii)
\begin{align}
&\sum_{k=0}^{\infty}\E\left[\|G(k)\|_{\LL(\X_2,\X_1)}^2
\right]<\infty,\label{ssafe}\\
&\sup_{k\ge 0}\E\left[\|G(k)\|_{\LL(\X_2,\X_1)}^4\right]<\infty,\label{ssafe1}
\end{align}
then, the  solution $\{x(k),k\ge 0\}$ of  (\ref{chafen}) is $L_2$-asymptotically stable.
\end{lemma}

\begin{remark}
Systematic results on the stability of randomly time-varying difference equations in finite-dimensional spaces were achieved in \cite{Guo1994}-\cite{GUO444}, while the results on randomly time-varying difference equations in infinite-dimensional spaces remain fragmented. Kubrusly (\cite{Kubrusly}), Vajjha \emph{et al.} (\cite{Vajjha}) and the references therein transformed the analysis of the mean square convergence of stochastic approximation algorithms in Hilbert spaces into the $L_2$-asymptotic stability analysis of randomly time-varying difference equation. Ungureanu \emph{et al.} (\cite{Ungureanu1}), Ungureanu (\cite{Ungureanu3}) and Zhang \emph{et al.} (\cite{zwzwxcbs}) investigated the $L_2$-asymptotic stability of the solution sequence of the randomly time-varying difference equation $x(k+1)=A(k)x(k)+b(k)$ in Hilbert space. The above literature assumed $A(k):\Omega\to \LL(\X)$ to be mean square exponentially stable, i.e., there exist constants $M>0$ and $\lambda\in (0,1)$ such that, for any $m>n\ge 0,~\forall\ x\in \X$,
\bna\label{zhishuwending}
\E\left[\left\|\left(\prod_{k=n+1}^mA(k)\right)x\right\|_{\X}^2\right]\leq M\lambda^{m-n}\|x\|^2.
\ena
For the randomly time-varying difference equation (\ref{chafen}), even if $F(k)\equiv F$, where $F$ is a deterministic self-adjoint compact operator with $\|F\|_{\LL(\X)}\leq 1$, the operator $A(k)\equiv I_{\X}-F$ does not satisfy (\ref{zhishuwending}). In fact, if (\ref{zhishuwending}) holds, then
\bna\label{zhihuhhh}
\left\|\prod_{k=n+1}^{n+2^l}A(k)\right\|_{\LL(\X)}\leq \sqrt{M}\lambda^{2^{l-1}},\ \forall\ l, n\ge 0.
\ena
Noting that $\|I_{\X}-F\|_{\LL(\X)}=\sup\limits_{\|x\|_{\X}=1}|\langle (I_{\X}-F)x,x\rangle_{\X} |=1-\inf_{\|x\|_{\X}=1}\langle Fx,x\rangle_{\X} =1$, thus we have $\E\Big[ \big\|\prod_{k=n+1}^{n+2^l} A(k)
 \big\|_{\LL(\X)}^2\Big]=\big\|(I_{\X}-F)^{2^l}\big
\|_{\LL(\X)}=\big\|I_{\X}-F\big\|_{\LL(\X)} ^{2^l}=1,\ \forall\ l, \ n\ge 0,$
which is in contradiction to (\ref{zhihuhhh}). Thereby, the existing results and methods on the stability of infinite-dimensional random difference equations are not applicable to random inverse problems.
\end{remark}

\subsection{Convergence of the Decentralized Algorithm}
Let
\begin{align}
\F(k)=\bigvee_{i=0}^k\left(\sigma\left(\H(i);\tau_{\text{S}}\left(\LL\left(\X^N,\bigoplus_{j=1}^N\Y_j\right)\right)\right)\bigvee \sigma\left(v(i);\tau_{\text{N}}\left(\bigoplus_{j=1}^N\Y_j\right)\right)\right),~k\ge 0,\label{section3fk}
\end{align}
and $\F(-1)=\{\emptyset,\Omega\}$, where $\H(i)$ and $v(i)$ are given in Section 2.2. We need the following assumptions.

\begin{assumption}\label{assumption1}
  The noises  $\{v(k),k\ge 0\}$ with values in $(\bigoplus_{i=1}^N\Y_i,\tau_{\text{N}}(\bigoplus_{i=1}^N\Y_i))$ and the random forward operator sequence $\{H(k),k\ge 0\}$ with values in $(\mathscr L( \X,\bigoplus_{i=1}^N$ $\Y_i),\tau_{\text{S}}(\mathscr L(\X$, $\bigoplus_{i=1}^N\Y_i)))$ are mutually independent.
\end{assumption}
\begin{assumption}\label{assumption2}
The  noises  $\{v(k),\F(k),k   \ge 0\}$ are martingale differences  and there exists a constant $\b_v>0$ such that $\sup_{k\ge 0}\E\big[\|v(k)\|_{\bigoplus_{i=1}^N\Y_i}^2|\F(k-1)\big]\leq \b_v~\text{a.s.}$
\end{assumption}

For the gains of the algorithm (\ref{algorithm}), we  need the following conditions.

\begin{condition}\label{condition1}
The algorithm gains $\{a(k),k\ge 0\}$ and $\{b(k),k\ge 0\}$ are both monotonically decreasing sequences of positive real numbers.
\end{condition}

\begin{condition}\label{condition2}
$\sum_{k=0}^{\infty}a^2(k)<\infty$ and $\sum_{k=0}^{\infty}b^2(k)<\infty$.
\end{condition}

\begin{condition}\label{condition3}
$\sum_{k=0}^{\infty}a(k)=\infty$ and $\max\{a(k)-a(k+1),b(k)-a(k)\}=\mathcal O(a^2(k)+b^2(k))$.
\end{condition}

We analyze the convergence of the algorithm (\ref{algorithm}) now.
Firstly, from Lemma \ref{wendingxing}, we have the following key theorem.

\begin{theorem}\label{dingliyi1}
For the algorithm (\ref{algorithm}), suppose that Assumptions \ref{assumption1}, \ref{assumption2} and Condition \ref{condition2} hold, there exists a sequence of nonnegative real numbers $\{\gamma(k),k\ge 0\}$ with $\sum_{k=0}^{\infty}\gamma(k)< \infty$, such that
\bna\label{qafgs}
&&~~~\left.\E\left[\|I_{\X^N}-\left(a(k)\H^*(k)\H(k)+b(k)\L_{\G}\otimes I_{\X}\right)\|^4_{\LL\left(\X^N\right)}\right|\F(k-1)\right]\cr
&&\leq 1+\gamma(k)~\text{a.s.},
\ena
and  the sequence of operator-valued random elements $\{I_{\X^N}-(a(k)\H^*(k)\H(k)+b(k)\L_{\G}\otimes I_{\X}),k\ge 0\}$ is $L_2^2$-stable  w.r.t. $\{\F(k),k\ge 0\}$.\\
I. If
$$\sup_{k\ge 0}\E\left[\|\H(k)\|_{\mathscr L\left(\X^N,\bigoplus_{i=1}^N\Y_i\right)}^2\right]<\infty,$$
then the algorithm (\ref{algorithm}) is mean square consistent, i.e., $\lim_{k\to\infty}\E[\|f_i(k)-f_0\|_{\X}^2]=0,~i\in \mathcal V$.\\
II. If there exists a constant $\rho_0>0$ such that
$$\E\left.\left[\|\H(k)\|_{\mathscr L\left(\X^N,\bigoplus_{i=1}^N\Y_i\right)}^2\right|\F(k-1)\right]\leq \rho_0~\text{a.s.},$$
then the algorithm (\ref{algorithm}) is almost surely strongly consistent, i.e., $\lim_{k\to\infty}\|f(k)-f_0\|_{\X}=0~\text{a.s.},~i\in \mathcal V$.
\end{theorem}
%


 We will next give intuitive sufficient conditions on the mean square and almost sure strong consistency of the algorithm (\ref{algorithm}).  At first, we need the following fundamental lemma.


\begin{lemma}\label{jihubiranshoulian}
For the algorithm (\ref{algorithm}), let $\G$ be connected and assume that Assumptions \ref{assumption1}, \ref{assumption2} and Conditions \ref{condition1}-\ref{condition3} hold. If there exist positive self-adjoint operators $\HH_i\in \mathscr L(\mathscr X)$, $i=1,\cdots,N$ satisfying $\sum_{i=1}^N\HH_i>0$, and  a fixed-length time period $h>0$, such that
\bna\label{yinlitiaojian1}
\sum_{j=1}^N\sum_{k=0}^{\infty}\E\left[\left\|\HH_jx(k)-\sum_{i=kh}^{(k+1)h-1}\E\left.\left[H_j^*(i)H_j(i)x(k)\right|\F(kh-1)\right] \right\|_{\X}^2\right]<\infty,
\ena
for any $L_2$-bounded adaptive sequence $\{x(k), \F(kh-1),k\ge 0\}$ with values in the Hilbert space $\X$, and there exists a constant $\rho_0>0$ such that
\bna\label{yinlitiaojian2}
\sup_{k\ge 0}\left(\E\left.\left[\|\H^*(k)\H(k)\|_{\LL\left(\X^N\right)}^{2^{\max\{h,2\}}}\right|\F(k-1)\right]\right)^{\frac{1}{2^ {\max\{h,2\}}}}\leq \rho_0~\text{a.s.},
\ena
then $\{I_{\X^N}-a(k)\H^*(k)\H(k)-b(k)\L_{\G}\otimes I_{\X},k\ge 0\}$ is $L_2^2$-stable  w.r.t. $\{\F(k),k\ge 0\}$.
\end{lemma}



Combining Theorem  \ref{dingliyi1} and Lemma  \ref{jihubiranshoulian}, we have the following theorem.

\begin{theorem}\label{vnknoklfl}
For the algorithm (\ref{algorithm}),
suppose that all the conditions in Lemma \ref{jihubiranshoulian} hold. If there exists a sequence of nonnegative real numbers $\{\Gamma(k),k\ge 0\}$ with $\sum_{k=0}^{\infty}\Gamma(k)<\infty$, such that
\bna\label{dinglitiaojian}
  \E\left[\|I_{\X^N}-4\left(a(k)\H^*(k)\H(k)+b(k)\L_{\G}\otimes I_{\X}\right)\|_{\LL\left(\X^N\right)}\big|\F(k-1)\right]
 \leq 1+\Gamma(k)~\text{a.s.},
\ena
then the algorithm (\ref{algorithm}) is both mean square and almost surely strongly consistent.
\end{theorem}
%

Especially, if $\{\H(k),k\ge 0\}$
and  $\{v(k),k\ge 0\}$
 are both i.i.d. and they are mutually independent, then the following corollary  follows from Theorem \ref{vnknoklfl}.

\begin{corollary}\label{xiaosirendetuilun}
For the algorithm (\ref{algorithm}), assume that $\G$ is connected, Assumptions \ref{assumption1}, \ref{assumption2} and Conditions \ref{condition1}-\ref{condition3} hold, and $\{\H(k),k\ge 0\}$ and $\{v(k),k\ge 0\}$ are both i.i.d. sequences and they are mutually independent. If there exists a constant $\rho_0>0$ such that $\|\H(0)\|\leq \rho_0~\text{a.s.}$ and
\bna\label{tuiluntiaojian} \sum_{j=1}^N\E\left[\|H_j(0)x\|_{\Y_j}^2\right]>0,\ \forall\ x\in \X\setminus\{0\},
\ena
then the algorithm (\ref{algorithm}) is both mean square and almost surely strongly consistent.
\end{corollary}
%

The graph $\G$ describes the communication topology among nodes, and its connectivity ensures the nodes to collaboratively reconstruct the unknown function $f_0$ successfully. The condition (\ref{yinlitiaojian1}) in Lemma \ref{jihubiranshoulian} plays an important role in the convergence analysis of the decentralized algorithm, which we call the \textbf{\emph{infinite-dimensional spatio-temporal persistence of excitation}} condition.
Note that
\begin{align}
&\left|\sum_{j=1}^N\sum_{i=kh}^{(k+1)h-1}\E\left[\|H_j(i)x\|_{\Y_j}^2\right]-\left\langle \sum_{j=1}^N\HH_jx,x\right\rangle _{\X}\right|\cr
 =& \Bigg|\sum_{j=1}^N\Bigg\langle \E\Bigg[\sum_{i=kh}^{(k+1)h-1}
 \E\left.\left[H^*_j(i)H_j(i)x\right|\F(kh-1)\right] -\HH_jx\Bigg],x\Bigg\rangle _{\X}\Bigg|\cr
 \leq & \sum_{j=1}^N\Bigg[\E\Bigg[\Bigg\|\sum_{i=kh}^{(k+1)h-1}\E\left.\left[H_j^*(i)H_j(i)x\right|\F(kh-1)\right]
   -\HH_jx
\Bigg\|_{\X}^2\Bigg]\Bigg]^{\frac{1}{2}}\|x\|_{\X},~x\in \X,~k\ge 0.\notag
\end{align}
Therefore, the infinite-dimensional spatio-temporal persistence of excitation condition is equivalent to the combination of the following  two conditions.
\begin{itemize}
\item For any $x\in \X\setminus\{0\}$, there exists an integer $K(x)>0$ such that
\bna\label{vnkwwoelekel}
\sum_{j=1}^N\sum_{i=kh}^{(k+1)h-1}\E\left[\left\|H_j(i)x\right\|_{\Y_j}^2\right]>0,~k\ge K(x);
\ena
\item there exist deterministic time-invariant operators $\HH_j$, $j=1,\cdots,N$ such that
$$
\sum_{j=1}^N\sum_{k=0}^{\infty} \E\Bigg[\bigg\|\HH_jx(k)\\ -\sum\limits_{i=kh}^{(k+1)h-1} \E\left[H_j^*(i)H_j(i)x(k)|\F(kh-1)\right] \bigg\|_{\X}^2\Bigg]<\infty,
$$
\end{itemize}
for any $L_2$-bounded adaptive sequence $\{x(k),\F(kh-1),k\ge 0\}$ with values in the Hilbert space $\X$.

\vskip 1mm

The \textbf{\emph{spatio-temporal persistence of excitation}} implies that
the non-zero orbits of the random forward operators of all nodes are non-degenerate in the mean square sense for a fixed length time period,  where \textbf{\emph{spatio-temporal}} refers specifically to the temporal and spatial states of the operator orbit $x\mapsto H_j(i)x$ of the random forward operator. By (\ref{vnkwwoelekel}), we know that we neither need the temporal orbit of the forward operator of each node over the graph to be non-degenerate, i.e., $\sum_{i=kh}^{(k+1)h-1}\E \left[\|H_j(i)x\|_{\Y_j}^2\right]>0,\ \forall\ j\in \mathcal V,$ nor need the spatial orbit of the forward operator of all nodes to be non-degenerate at each instant, i.e., $\sum_{j=1}^N\E\left[\|H_j(i)x\|_{\Y_j}^2\right]>0,\ \forall\ i\ge 0.$
Notably, if the sequences of random forward operator  and   noises  are both i.i.d. and they are mutually independent, then the \textbf{\emph{infinite-dimensional spatio-temporal persistence of excitation}} condition degenerates to the case that the spatial orbits of the forward operators of all nodes are non-degenerate at the initial moment $k=0$, i.e., the condition (\ref{tuiluntiaojian}) in the Corollary \ref{xiaosirendetuilun}.

In the past decades, to solve the problems of finite-dimensional parameter estimation and signal tracking with non-stationary and non-independent data, many scholars have proposed excitation conditions based on the conditional expectation of the observation/regression matrix. The stochastic persistence of excitation condition was first proposed by Guo (\cite{Guo1990}) in the analysis of centralized Kalman filtering algorithms.  Xie and Guo (\cite{Xieguo}) proposed the cooperative information condition based on the conditional expectations of the observation matrices for the distributed adaptive filtering algorithm over connected graphs. Wang \emph{et al.} (\cite{WLZ}) proposed the stochastic spatio-temporal persistence of excitation condition for the decentralized online estimation algorithm over randomly time-varying graphs.
Zhang \emph{et al.} (\cite{ZLF}) proved that if the graph is connected and the randomly time-varying regression matrix satisfies the uniformly conditionally spatio-temporally joint observability condition, i.e., there exists an integer $h>0$ and a constant $\theta>0$, respectively, such that
\ban
\inf_{k\ge 0}\lambda_{\text{min}}\bigg(\sum_{j=1}^N\sum_{i=kh}^{(k+1)h-1}
\E\left[H_j^T(i)H_j(i)\big|\F(kh-1)\right]\bigg)\ge \theta \ \text{a.s.},
\ean
then the algorithm achieves mean square and almost sure convergence. For this case, it is not difficult to verify that the random matrix sequence $\{I_{Nn}-a(k)\H^*(k)\H(k)-b(k)\L_{\G}\otimes I_{n},k\ge 0\}$ satisfies the $L_2^2$-stability condition w.r.t. $\{\F(k),k\ge 0\}$, and thus the excitation conditions proposed in \cite{WLZ} and \cite{ZLF} are all special cases of the $L_2^2$-stability condition in Theorem \ref{dingliyi1}.

The persistence of excitation conditions proposed for finite-dimensional systems all require that the conditional expectation of the information matrix consisting of the observation (regression) matrices is positive definite, i.e., the information matrix has strictly positive minimum eigenvalues; however, inverse problems in infinite-dimensional Hilbert spaces are usually ill-posed. Even for a strictly positive linear compact operator, the excitation condition similar to the persistence excitation conditions for finite-dimensional systems can not hold any more since the infimum of the eigenvalues of the compact operator is always zero.

\section{Decentralized Online Learning in Reproducing Kernel Hilbert Space}
We will discuss in this section a special class of online random inverse problems: decentralized online learning problems in  RKHS. Let $\X$ be a non-empty subset of $\mathbb R^n$, $K:\X\times \X\to \mathbb R$ be a Mercer kernel, and $(\mathscr H_K,\langle \cdot,\cdot \rangle _K)$ be a separable RKHS with kernel $K$, which  consists of functions with domain $\X$.
The observation data $y_i(k)$ of the $i$-th node at instant $k$ is given by
\bna\label{xuexi}
y_i(k)=f_0(x_i(k))+v_i(k),\ k\ge 0,\ i\in \mathcal V,
\ena
where $x_i(k):\Omega\to \X$ is a random vector with values in the Hilbert space $(\X,\tau_{\text{N}}(\X))$ at instant $k$, called the random input data, and the observation noise $v_i(k):\Omega\to \mathbb R$ is a random vector with values in the Hilbert space $(\mathbb R,\tau_{\text{N}}(\mathbb R))$,  and  $f_0:\X\to \mathbb R$ is an unknown function in $\HH_K$. The nodes cooperatively estimate $f_0$ by information exchanging among them.

\begin{remark}
\label{remarkaddadd6}
In \cite{kar2012}, the measurement of each node is given by $y_i(k)=f_i(\theta^*)+v_i(k)$, $i=1,2,...,N$, in which the nonlinear mappings $f_i(\cdot)$, $i=1,2,...,N$ are completely known in prior and it is the parameter vector $\theta^*$ in a finite-dimensional space, which is unknown and needs to be estimated. While for the measurement model (\ref{xuexi}), it is the mapping $f_0$ in the infinite-dimensional space $\HH_K$, which is unknown and needs to be estimated.
\end{remark}

For any given $x\in \X$, the function $K_x:\X\to \HH_K$ induced by the Mercer kernel is given by $K_x(y)=K(x,y),\ \forall\ y\in\X.$
Define the random forward operator $H_i(k)$ as
$$
H_i(k)(f):=f(x_i(k)),~f\in \HH_K,~k\ge 0,~i\in \mathcal V,
$$
then the measurement model (\ref{xuexi}) can be represented as the random inverse problem with the measurement equation (\ref{measuramentmodel}). Based on the algorithm (\ref{algorithm}), the decentralized online learning strategy in $\HH_K$ is given by
\bna\label{rkhs}
&&\hspace{-0.8cm}f_i(k+1)=f_i(k)+a(k)(y_i(k)-f_i(k)(x_i(k)))K_{x_i(k)}+b(k)\sum_{j\in \N_i}a_{ij}(f_j(k)-f_i(k)),\cr
&&~~~~~~~~~~~~~~~~~~~~~~~~~~~~~~~~~~~~~~~~~~~~~~~~~~~~~~~~~~~~~~~~~~~~~~~~~~~~k\ge 0,~i\in \mathcal V.
\ena

The learning  process of (\ref{rkhs}) is given by Algorithm \ref{DORKHS}.
Assume that  $\X$ is a  bounded and closed interval of $\mathbb{R}$ in Algorithm \ref{DORKHS}.

\renewcommand{\algorithmicrequire}{\textbf{Input:}}  
\renewcommand{\algorithmicensure}{\textbf{Output:}} 

\begin{algorithm}

         \begin{algorithmic}[1] 
         \caption{Decentralized Online Learning in RKHS}
          \label{DORKHS}
          \State Uniformly sample $L$ points $\{z_{l},\ l=1,\ldots,L\}$ at equal distances  on $\X$.
          \State Input:
           initial values $f_{i}(0)(z_{l}),\ l=1,\ldots,L,\ i=1,...,N$,   iteration step $K_1$, weighted matrix $A=(a_{ij})$,  algorithms gains $\{a(k), b(k), \ k\geq 0\}$, kernel function $K(\cdot,\cdot)$.
          \For{$k = 1$, \ldots, $K_1$}
              \For{$i = 1$, \ldots, $N$}
              \State Input: $x_{i}(k)$, $y_{i}(k)$.
              \If{$x_{i}(k)\in \{z_{l},\ l=1,\ldots,L\}$}
              \For{$l = 1$, \ldots, $L$}
         \State $f_i(k+1)(z_{l})=f_i(k)(z_{l})+a(k) ( y_{i}(k)-f_i(k)(x_i(k))) K(x_i(k),z_{l})$
         \Statex $\ \ \ \ \ \ \ \ \ \ \ \ \ \ \ \ \ \ \ \ \ \ \ \ \ \ \ \ \ \ \ \ \ \ \ +b(k)\sum_{j\in \N_i}a_{ij} (f_j(k)(z_{l})-f_i(k)(z_{l})).$
        \EndFor
        \Else
           \State  Approximate   $f_i(k)(x_i(k))$ by the cubic spline interpolation method based on $\{f_i(k)(z_{l}),  l= 1, \ldots, L\}$.
           \For{$l = 1$, \ldots, $L$}
         \State $f_i(k+1)(z_{l})=f_i(k)(z_{l})+a(k) ( y_{i}(k)-f_i(k)(x_i(k))) K(x_i(k),z_{l})$
         \Statex $\ \ \ \ \ \ \ \ \ \ \ \ \ \ \ \ \ \ \ \ \ \ \ \ \ \ \ \ \ \ \ \ \ \ \ +b(k)\sum_{j\in \N_i}a_{ij} (f_j(k)(z_{l})-f_i(k)(z_{l})).$
        \EndFor
       \EndIf
          \EndFor
         \EndFor
         \State Output: for $i\in \mathcal V$ and $x\in \X$, if $x\in \{z_{l},\ l=1,\ldots,L\}$, then output $f_i(K_1)(x)$, and if $x\notin \{z_{l},\ l=1,\ldots,L\}$, then output  $f_i(K_1)(x)$ by the cubic spline interpolation method based on $\{f_i(K_1)(z_{l}),  l= 1, \ldots, L\}$.
        \end{algorithmic}
 \end{algorithm}

Given $\phi,\psi\in \HH_K$, denote the rank 1 tensor product operator $\phi\otimes \psi:\HH_K\to \HH_K$ by
$$
\left(\phi\otimes \psi\right)(f):=\langle f,\psi\rangle _{K}\phi,~f\in \HH_K.
$$

 For the algorithm (\ref{rkhs}), we need the following assumptions.

\begin{assumption}\label{assumption5}
$\sup_{x\in \X}K(x,x)<\infty$.
\end{assumption}

\begin{assumption}\label{assumption3}
The sequence $\{x_i(k),i\in \mathcal V,k\ge 0\}$ of random vectors with values in the Hilbert space $(\X,\tau_{\text{N}}(\X))$ and the sequence $\{v_i(k),i\in \mathcal V,k\ge 0\}$ of random variables with values in the Hilbert space $(\mathbb R,\tau_{\text{N}}(\mathbb R))$ are mutually independent.
\end{assumption}

\begin{lemma}\label{rabdomelerkhs}
  If Assumption  \ref{assumption5} and Assumption \ref{assumption3}  hold, then $H_i(k)$ is a random element with values in $(\LL(\HH_K,\mathbb R),\tau_{\text{S}}(\LL(\HH_K,\mathbb R))),\ k\ge 0,~i\in \mathcal V$.
\end{lemma}

By Lemma \ref{rabdomelerkhs} and  (\ref{section3fk}), it is known that
\begin{align}\label{nvwjoijjff}
\F(k)= \bigvee_{s=0}^k\bigg(\sigma\left(\H(s);\tau_{\text{S}}\left(\LL\left(\HH_K^N,\mathbb R^N\right)\right)\right)
 \bigvee \sigma\left(v(s);\tau_{\text{N}}\left(\mathbb R^N\right)\right)\bigg),~k\ge 0,
\end{align}
and $\F(-1)=\{\emptyset,\Omega\}$, where $v(s)=(v_1(s),\cdots,v_N(s))$ and $\mathcal H(s)=\text{diag}\{H_1(s), \cdots, H_N(s)\},\ s\geq 0$.

\begin{assumption}\label{assumption4}
The   noises  $\{v_i(k),\F(k), k\ge 0\}$, $i\in \mathcal V$  are martingale difference sequences and there exists a constant $\b>0$, such that  $ \max\limits_{i\in\mathcal V}\sup\limits_{k\ge 0}\E\big[\|v_i(k)\|_{\mathbb R}^2 |\F(k-1)\big]\leq \b~\text{a.s.}$
\end{assumption}


\begin{remark}
\rm{The existing works on RKHS online learning (\cite{Ying}-\cite{SmaleYao}) all require the random input data to be i.i.d. Notice that in Assumption  \ref{assumption3}, the sequence of random forward operators with values in $(\mathscr L(\X,\bigoplus_{i=1}^N\Y_i),\tau_{\text{S}}(\mathscr L(\X,\bigoplus_{i=1}^N\Y_i)))$ is not required to satisfy special statistical properties such as independence, stationarity, etc.}
\end{remark}



Based on the convergence results of the algorithm (\ref{algorithm}), we can obtain the following convergence results.

\begin{theorem}\label{rkhsdingli}
For the algorithm (\ref{rkhs}), suppose that $\G$ is connected, Assumptions \ref{assumption5}-\ref{assumption4} and Conditions \ref{condition1}-\ref{condition3} hold. If there exist positive self-adjoint operators $N_i\in \mathscr L(\HH_K)$, $i=1,\cdots,N$ satisfying $\sum_{i=1}^NN_i>0$ and there exists a  fixed-length time period $h>0$, such that
  \bna\label{vnknknldldklsd}
\sum_{j=1}^N\sum_{k=0}^{\infty}\E\left[\Bigg\|\Bigg(N_j-\sum_{i=kh}^{(k+1)h-1}\E\left.\left[K_{x_j(i)}\otimes K_{x_j(i)}\right|\F( kh-1)\right]\Bigg)g(k)\Bigg\|_K^2\right]<\infty,
  \ena
 for arbitrary $L_2$-bounded adaptive sequence $\{g(k),\F(kh-1),k \ge 0\}$, then the algorithm (\ref{rkhs}) is both mean square and almost surely strongly consistent.
\end{theorem}
%
\vskip 1mm

We now give some corollaries.

\vskip 1mm

\begin{corollary} \label{vnlllleleeemmem}
For the algorithm (\ref{rkhs}), suppose that $\G$ is connected,  Assumptions \ref{assumption5}-\ref{assumption4} and Conditions \ref{condition1}-\ref{condition3} hold. If there exist positive self-adjoint operators $N_i\in \mathscr L(\HH_K)$, $i=1,\cdots,N$ satisfying $\sum_{i=1}^NN_i>0$, and there exists a  fixed-length time period $h>0$, a constant $\mu_0>0$ and a nonnegative real sequence $\{\tau(k),k\ge 0\}$,  such that
\bna\label{vnkmeeeemefffff}
\max_{j\in \mathcal V}\Bigg\|N_j-\sum_{i=kh}^{(k+1)h-1}\E\left.\left[K_{x_j(i)}\otimes K_{x_j(i)}\right|\F(kh-1)\right]\Bigg\|_{\LL( \HH_K)}^2\leq \mu_0\tau(k)~\text{a.s.},
\ena
where $\sum_{k=0}^{\infty}\tau(k)<\infty$, then the algorithm (\ref{rkhs}) is both mean square and almost surely strongly consistent. Besides, the algorithm (\ref{rkhs}) is pointwisely almost surely strongly consistent, that is, $\lim\limits_{k\to \infty} f_{i}(k)(x)=f_{0}(x) $   \text{ a.s.},  $\forall \ x \in \mathscr{X}, \ i \in \mathcal{V}$.
\end{corollary}
%

\begin{remark}
Noting that Assumption \ref{assumption5} implies that $K_{x_j(i)}\otimes K_{x_j(i)}$ is a Bochner integrable random element with values in the Banach space $(\mathscr L(\HH_K),\tau_{\text{N}}(\mathscr L(\HH_K)))$, then the conditional expectations $\E[K_{x_j(i)}\otimes K_{x_j(i)}|\F(kh-1)]$, $i,k\ge 0$, $j\in \mathcal V$ uniquely exist by Lemma \ref{nvkvpeoeo}.
\end{remark}

Especially, if the input data $\{(x_1(k),\cdots,x_N(k)),k\ge 0\}$ are i.i.d, then we have the following corollary.

\begin{corollary}\label{rkhsdinglijjjjj}
For the algorithm (\ref{rkhs}), suppose that $\{(x_1(k),\cdots,x_N(k)),k\ge 0\}$ and $\{(v_1(k),\cdots,\\v_N(k)),k\ge 0\}$ are i.i.d. sequences and they are mutually independent, and $\G$ is connected. If Assumptions \ref{assumption5}-\ref{assumption4} and Conditions \ref{condition1}-\ref{condition3} hold, and
  \bna\label{nklnkle}
\E\Bigg[\sum_{j=1}^NK_{x_j(0)}\otimes K_{x_j(0)}\Bigg]>0,
  \ena
then the algorithm (\ref{rkhs}) is both mean square and almost surely strongly consistent.
\end{corollary}
%

\begin{remark}
For the centralized online learning with input data drawn independently from the probability measure $\rho_{\X}$ on $\X$,  Tarr\`{e}s  and Yao (\cite{Tarres}) defined the covariance operator of the probability measure $\rho_{\X}$ in $\HH_K$ as $L_K:\HH_K\to \HH_K$,
$$
L_K(f)(y)=\int_{\Omega}K(x,y)f(x)\dd\rho_{\X},~\forall\ f\in \HH_K.
$$
By the reproducing property and Assumption \ref{assumption5}, it follows that $L_K=\E[K_x\otimes K_x]$. This implies that for the centralized online learning problem in RKHS with i.i.d. data, the condition (\ref{nklnkle}) in Corollary \ref{rkhsdinglijjjjj} just degenerates to that in \cite{Tarres}: the covariance operator $L_K>0$.
\end{remark}

\section{Numerical Simulation}
We consider an undirected connected  graph with a node set $\mathcal{V}= \{1,2,\ldots, 10\}$  and its  weighted adjacency  matrix is given by $A=[a_{i,j}]$, where  $a_{1,2}= a_{2,1}=0.2,$ $a_{1,4}=a_{4,1}=0.4$, $a_{2,3}=a_{3,2}=0.1$, $a_{2,4}=a_{4,2}=0.3$, $a_{3,5}=a_{5,3}=0.5$, $a_{4,5}=a_{5,4}=0.6$, $a_{4,6}=a_{6,4}=0.8$, $a_{5,6}=a_{5,6}=0.7,$ $a_{6,7}=a_{7,6}=0.3$, $a_{7,8}=a_{8,7}=0.2$, $a_{8,9}=a_{9,8}=0.9$, $a_{9,10}=a_{10,9}=0.1$ and for the remaining positions, $a_{i,j}=0.$

For $i\in \mathcal{V}$, the observation data of node $i$ at instant $k$ is $(x_{i}(k),y_{i}(k))$, where
 $y_{i}(k) = f^*(x_{i}(k)) + v_{i}(k)$,  $f^*(x) =e^{-(x-1)^2},\  \forall \ x \in \mathscr{X} = [-2,4]$ is the unknown true function to be estimated, the input data $ x_{i}(k),\ i = 1,2,\ldots, 10,\ k\in \mathbb{N}$ are  independent random variables sampling according to the following rules. For  $k\in \mathbb{N}$,   $ x_{i}(2k)$ and $ x_{i}(2k+1)$ are   with uniform distributions on $\big[-2,4- \frac{3}{k+1}\big]$ and   $\big[\frac{3}{ k+1}-2,4\big]$, respectively.   Take the kernel function as $K(x,y) = e^{-(x-y)^2},\  \forall \ x,\ y \in \left[-2,4\right]$ and  we know that the unknown true function $f^* = K(x,1) \in \mathscr{H}_K$, where $\mathscr{H}_K$ is the RKHS  with kernel $K$.
	
We sample $1000$ points $\{z_{l},\ l=1,\ldots,1000\}$ with $z_{l}=-2+\frac{6(l-1)}{1000},\  l=1,\ldots,1000.$ We use the algorithm (\ref{rkhs})  to iterate the values of $f_{i}(k)$ at the sampled points, that is,
\begin{align}
&f_i(k+1)(z_{l})\notag\\=&f_i(k)(z_{l})+a(k)(y_i(k)-f_i(k)(x_i(k)))  K(x_i(k), z_{l} )
+b(k)\sum_{j\in \N_i}a_{ij}(f_j(k)(z_{l})-f_i(k)(z_{l})),\label{simu}
\end{align}
where $f_{i}(0)(z_{l})=0$, $\ l=1,\ldots,1000,\ k\ge 0,~i\in \mathcal V$.
If $x_i(k)\notin \{z_{l},\ l=1,\ldots,1000\}$, then we get the approximation of $f_i(k)(x_i(k))$ by the cubic spline interpolation method.

The measurement noises $v_{i}(k),\ i\in \mathcal{V},\  k\in \mathbb{N}$ are the i.i.d. Gaussian noises  with the normal distribution $N(0, 0.1)$ and are independent of the input data. The algorithm   gains are
	$ a(k) =   (k+1)^{-0.6},\ b(k) = (k+1)^{-1},\forall\ k\geq 1. $
Fig. 1(a) and Fig. 1(b)  show the  values of the true function and   the outputs  $f_i(k), \ i\in \mathcal V$  of (\ref{simu}), i.e.
the values of $f^{*}$ and $f_i(k), \ i\in \mathcal V$  at all sampled points in $[-2,4]$ for $k=1000$  and $k=100000$ iterations, respectively. It can be seen that the  estimations of all nodes $f_i(k), \ i\in \mathcal V$ converge to the unknown true function $f^{*}$ pointwisely almost surely as $k$ increases. This  is consistent with  Corollary \ref{vnlllleleeemmem}.

In addition to Gaussian measurement noises, Laplace measurement  noises  also appear  in some signal processing problems (\cite{Alliney1994}). We assume that $v_{i}(k),\  k\in \mathbb{N}, \ i\in \mathcal{V}$ are the  i.i.d. Laplace noises  with  mean $0$ and variance $2$ and are  independent of the input data. The algorithm   gains are
	$ a(k) =   (k+1)^{-0.8},\ b(k) = (k+1)^{-0.9},\forall\ k\geq 1.$  Fig. \ref{fig:mainfig2}(a) and Fig. \ref{fig:mainfig2}(b)  show the  values of the true function and   the outputs  $f_i(k), \ i\in \mathcal V$  of (\ref{simu}).
It can be seen that the  estimations of all nodes $f_i(k), \ i\in \mathcal V$ converge to the unknown true function $f^{*}$ pointwisely almost surely as $k$ increases. This  is consistent with  Corollary \ref{vnlllleleeemmem}.

\begin{figure}[htbp]
    \centering
    \subfigure{\includegraphics[width=0.5\linewidth, height=6cm]{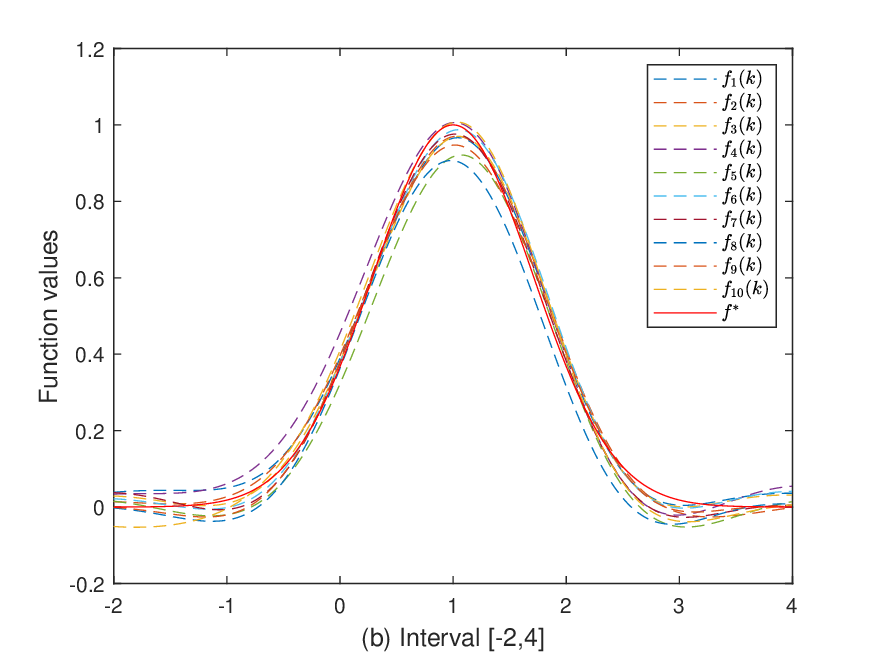}}
      \hspace{-5mm}
  \subfigure{\includegraphics[width=0.5\linewidth, height=6cm]{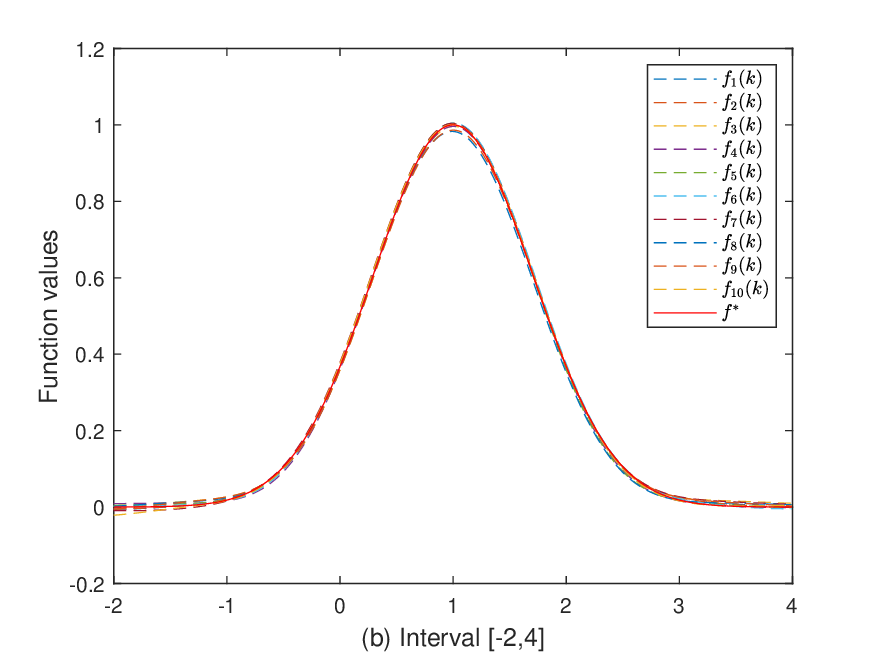}}
   \caption{ Gaussian  measurement noises: (a)  estimates  $f_{i}(k),\ i=1,\cdots,10$  of nodes   for $k=1000$; (b)   estimates  $f_{i}(k),\ i=1,\cdots,10$  of nodes for $k=100000$.}
    \label{fig:mainfig}
\end{figure}

\begin{figure}[htbp]
    \centering
    \subfigure{\includegraphics[width=0.5\linewidth, height=6cm]{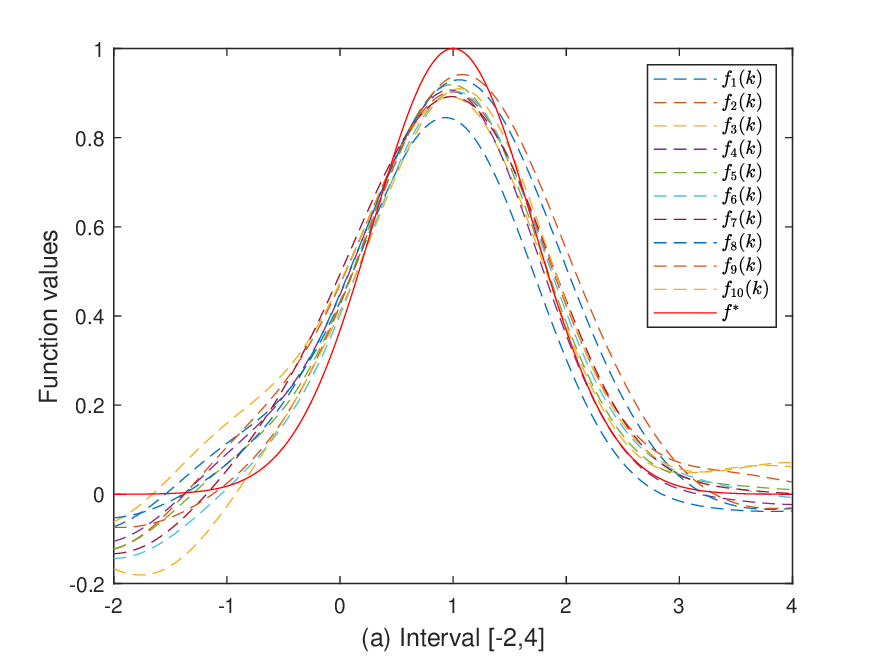}}
      \hspace{-5mm}
  \subfigure{\includegraphics[width=0.5\linewidth, height=6cm]{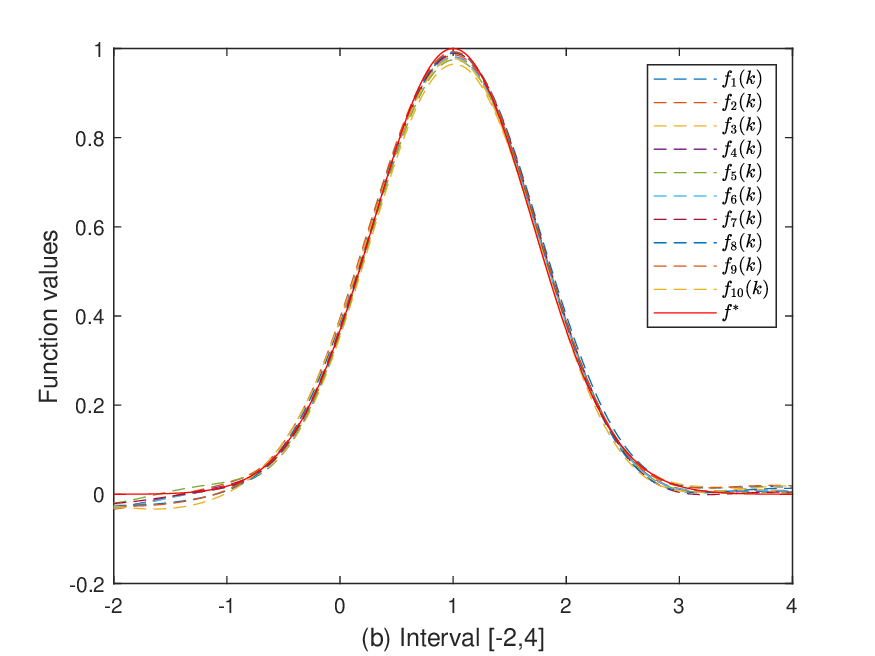}}
   \caption{Laplace measurement noises: (a)  estimates  $f_{i}(k),\ i=1,\cdots,10$  of nodes   for $k=1000$; (b)   estimates  $f_{i}(k),\ i=1,\cdots,10$  of nodes for $k=100000$.}
    \label{fig:mainfig2}
\end{figure}

\section{Conclusions}
We have established a framework of random inverse problems with online measurements over graphs, and present a decentralized online learning algorithm with online data streams.
It is not required that the random forward operators satisfy special statistical assumptions such as mutual independence, spatio-temporal independence or stationarity. Each node updates its estimate at the next instant by using its new observation and a weighted sum of its own and neighbors' estimates.
Firstly, by  exploiting the probabilistic properties of random elements with values in different topological spaces in a stochastic framework, we propose the $L_p^q$-stability condition on the sequence of operator-valued random elements, and establish the $L_2$-asymptotic stability theory of a class of   inhomogeneous random difference equations in Hilbert spaces with $L_2$-bounded martingale difference terms. Subsequently, we transform the asymptotic stability of these kinds of infinite-dimensional random difference equations into the $L_2^2$-stability condition on the operator-valued random elements. We have obtained an intuitive sufficient condition on the convergence of decentralized online learning algorithms for random inverse problems over graphs which is imposed on the random forward operators and the Laplacian matrix of the graph, i.e., the \emph{infinite-dimensional spatio-temporal persistence of excitation} condition.  We have proved that if the forward operators over connected graphs satisfy the \emph{infinite-dimensional spatio-temporal persistence of excitation} condition, then all nodes' estimates are mean square and almost surely strongly consistent. Finally, by equivalently transforming the distributed learning problem in RKHS to the random inverse problem over graphs, we propose a decentralized online learning algorithm in RKHS with non-stationary online data streams, and prove that the algorithm is mean square and almost surely strongly consistent if the operators induced by the random input data satisfy the \emph{infinite-dimensional spatio-temporal persistence of excitation} condition.

There are still many promising  directions deserving for future research. (i) The case with multiplicative  measurement noises. Only additive measurement noises are considered in (\ref{measuramentmodel}),
 while multiplicative  noises  appear in some realistic scenarios, such as synthetic aperture radar images (\cite{Varadarajan}). For this case, the estimation error equation belongs to a class of randomly time-varying difference equations that do not satisfy the independence condition required in Lemma \ref{wendingxing}.
 (ii) The case with  unbalanced and directed communication graphs. For this case, the operator $\mathcal{L}_{\mathcal{G}}\otimes I_{\mathscr{X}}$ may not be positive, which implies that   Lemma \ref{jihubiranshoulian} does not hold.  It is crucial to develop novel proof techniques or leverage the algorithm presented in \cite{Li2020}  for addressing unbalanced  directed graphs.
(iii) Extension to the distributed fixed-point finding problem (\cite{Li2024}) with measurement noises. Li \emph{et al}. (\cite{Li2024}) investigated the distributed fixed-point finding problem for a global operator. Two algorithms are proposed for two scenarios, both of which use constant algorithm gains and are proven to achieve linear convergence.  Noting that the decaying algorithm gains in (\ref{algorithm}) are introduced to attenuate the measurement noises, thus the convergence rate of (\ref{algorithm}) can not be  linear.  While \cite{Li2024} does not consider the case with the measurement noises  of the  global operator, it is worth  extending our method to address the problem in \cite{Li2024}  with  measurement noises.

\begin{appendices}

\section{Proofs of Theorems and Corollaries in Sections 3-4}
\label{appendixb}
 \setcounter{equation}{0}
\renewcommand{\theequation}{A.\arabic{equation}}
\emph{Proof of Theorem \ref{dingliyi1}.}
Denote $F(k)=a(k)\H^*(k)\H(k)+b(k)\L_{\G}\otimes I_{\X}$ and $G(k)=a(k)\mathcal H^*(k)$, respectively. Notice that $F(k)\ge 0$ and the estimation error equation (\ref{error}) can be rewritten as the following random difference equation
$
e(k+1)=\left(I_{\X^N}-F(k)\right)e(k)+G(k)v(k).
$
 We will briefly sketch the proof. Firstly, choosing  algorithm gains properly,
  we verify the conditions in    Lemma \ref{wendingxing} and   establish conclusion I. Then, by the nonnegative
supermartingale convergence theorem, we demonstrate that $\|e(k)\|^2$ converges almost surely. This together with conclusion I leads to conclusion II.

Given the initial value $e(0)\in \X^N$, by Proposition \ref{nlllwwieiie}.(a)-(c), we know that $\{e(k),k\ge 0\}$ is a random sequence with values in the Hilbert space $(\X^N,\tau_{\text{N}}(\X^N))$. On the one hand, it follows from Assumptions \ref{assumption1} and \ref{assumption2} that $\{v(k),k\ge 0\}$ is independent of $\{F(k),G(k),k\ge 0\}$, and $\sup_{k\ge 0}\E[\|v(k)\|^2]\leq \b_v$. On the other hand, by the condition (\ref{qafgs}), we get $\E[\|I_{\X^N}-F(k)\|^4|\F(k-1)]\leq 1+\gamma(k)~\text{a.s.}$, which gives
\bna\label{xlms}
&&~~~~\E\left[\|a(k)\H^*(k)\H(k)+b(k)\L_{\G}\otimes I_{\X}\|^2\right]\cr
&&\leq \E\left[\|I_{\X^N}-(I_{\X^N}-F(k))\|^2\right]\cr
&&\leq 2\left(1+\E\left[\|I_{\X^N}-F(k)\|^2\right]\right)\cr
&&\leq 2\left(1+\E\left[\|I_{\X^N}-F(k)\|^4\right]^{\frac{1}{2}}\right)\cr
&&\leq 2\left(1+\sqrt{1+\gamma(k)}\right).
\ena
Noting that $\L_{\G}$ is positive semi-definite, we have $a(k)\H^*(k)\H(k)+b(k)\L_{\G}\otimes I_{\X}\ge a(k)\H^*(k)\\ \H(k)\ge 0~\text{a.s.}$, which leads to $\|a(k)\H^*(k)\H(k)\|^2\leq \|a(k)\H^*(k)\H(k)+b(k)\L_{\G}\otimes I_{\X}\|^2$. By (\ref{xlms}), we obtain
\bna\label{xfwee}
&&~~~~\sup_{k\ge 0}\E\left[\|G(k)\|^4\right]\cr &&=\sup_{k\ge 0}\left\{a^2(k)\E\left[\left\|a(k)\H^*(k)\H(k)\right\|^2\right]\right\}\cr
&&\leq \sup_{k\ge 0}\left\{a^2(k)\E\left[\|a(k)\H^*(k)\H(k)+b(k)\L_{\G}\otimes I_{\X}\|^2\right]\right\}\cr
&&\leq 2\sup_{k\ge 0}\left\{a^2(k)\left(1+\sqrt{1+\gamma(k)}\right)\right\}\cr
&&<\infty,
\ena
where the last inequality is obtained from Condition \ref{condition2} and $\sum_{k=0}^{\infty}\gamma(k)<\infty$.\\
(I) If $\sup_{k\ge 0}\E[\|\H(k)\|^2]<\infty$, it follows from Condition \ref{condition2} that
\ban
\sum_{k=0}^{\infty}\E\left[\|G(k)\|^2\right]\leq \left\{\sup_{k\ge 0}\E\left[\|\H(k)\|^2\right]\right\}\sum_{k=0}^{\infty}a^2(k)<\infty.
\ean
By Lemma \ref{wendingxing}, the sequence of solutions to the estimation error equation (\ref{error}) is $L_2$-asymptotically stable, i.e., $\lim_{k\to\infty}\E[\|f_i(k)-f_0\|^2]=0,~i\in \mathcal V$. \\
(II) If $\E[\|\H(k)\|^2|\F(k-1)]\leq \rho_0~\text{a.s.}$, on the one hand, noting that $\sup_{k\ge 0}\E[\|\H(k)\|^2]\leq \rho_0<\infty$, by the above conclusion  I, we have $\lim_{k\to\infty}\E[\|e(k)\|^2]=0$. On the other hand, it follows from the condition (\ref{qafgs}) and (\ref{xfwee}) that
\ban
\sup_{k\ge 0}\E\left[\|(I_{\X^N}-F(k))G(k)\|^2\right]\leq \sup_{k\ge 0}\E\left[\|I_{\X^N}-F(k)\|^4+\|G(k)\|^4\right]<\infty.
\ean
Thus, by   (\ref{error}), Assumptions \ref{assumption1} and \ref{assumption2}, Proposition 2.6.31 in \cite{hy} and Proposition \ref{lemmaA6}, we get
\ban
&&~~~~\E\left.\left[\|e(k+1)\|^2\right|\F(k-1)\right]\cr
&&=\E\left.\left[\|\left(I_{\X^N}-F(k)\right)e(k)+G(k)v(k)\|^2\right|\F(k-1)\right]\cr
&&=\E\left.\left[\|(I_{\X^N}-F(k))e(k)\|^2\right|\F(k-1)\right]+\E\left.\left[\|G(k)v(k)\|^2\right|\F(k-1)\right]\cr
&&~~~+2\E\left.\left[\langle e(k),(I_{\X^N}-F(k))G(k)v(k)\rangle \right|\F(k-1)\right]\cr
&&=\E\left.\left[\|(I_{\X^N}-F(k))e(k)\|^2\right|\F(k-1)\right]+\E\left.\left[\|G(k)v(k)\|^2\right|\F(k-1)\right]\cr
&&~~~+2\langle e(k),\E\left.\left[(I_{\X^N}-F(k))G(k)v(k)\right|\F(k-1)\right]\rangle\cr
&&=\E\left.\left[\|(I_{\X^N}-F(k))e(k)\|^2\right|\F(k-1)\right]+\E\left.\left[\|G(k)v(k)\|^2\right|\F(k-1)\right]\cr
&&~~~+2\langle e(k),\E\left.\left[(I_{\X^N}-F(k))G(k)\E[v(k)|\F(k-1)]\right|\F(k-1)\right]\rangle\cr
&&=\E\left.\left[\|(I_{\X^N}-F(k))e(k)\|^2\right|\F(k-1)\right]+\E\left.\left[\|G(k)v(k)\|^2\right|\F(k-1)\right]\cr
&&\leq \E\left.\left[\|(I_{\X^N}-F(k))\|^2\right|\F(k-1)\right]\|e(k)\|^2+\E\left.\left[\|G(k)v(k)\|^2\right|\F(k-1)\right]\cr
&&\leq \E\left.\left[\|(I_{\X^N}-F(k))\|^4\right|\F(k-1)\right]^{\frac{1}{2}}\|e(k)\|^2+\E\left.\left[\|G(k)v(k)\|^2\right|\F(k-1)\right]\cr
&&\leq \left(1+\gamma(k)\right)^{\frac{1}{2}}\|e(k)\|^2+\b_v\E\left.\left[\|G(k)\|^2\right|\F(k-1)\right]\cr
&&\leq \left(1+\frac{1}{2}\gamma(k)\right)\|e(k)\|^2+\rho_0\b_va^2(k)~\text{a.s.}
\ean
Noting that $\sum_{k=0}^{\infty}\gamma(k)<\infty$ and $\sum_{k=0}^{\infty}a^2(k)<\infty$, it follows from Lemma \ref{lemmaA3} that $\|e(k)\|^2$ converges almost surely, which together with $\lim_{k\to\infty}\E[\|e(k)\|^2]=0$ gives $\lim\limits_{k\to\infty}e(k)=0~\text{a.s.}$
$\hfill\square$

\vskip 2mm

\emph{Proof of Theorem \ref{vnknoklfl}.}
By the conditions (\ref{yinlitiaojian1})-(\ref{yinlitiaojian2}) and Lemma \ref{jihubiranshoulian}, it is known that $\{I_{\X^N}-a(k)\H^*(k)\H(k)-b(k)\L_{\G}\otimes I_{\X},k\ge 0\}$ is $L_2^2$-stable w.r.t. $\{\F(k),k\ge 0\}$. Denote $D(k)=a(k)\H^*(k)\H(k)+b(k)\L_{\G}\otimes I_{\X}$. By the condition (\ref{yinlitiaojian2}), we know that
\begin{align}\label{cmllemfnn}
&\E\left[\|D(k)\|^r|\F(k-1)\right]\notag\\
 \leq & \Big(\E\big[\|a(k)\H^*(k)\H(k)+b(k)(\L_{\G}\otimes I_{\X})\|^{2^h}\big|\F(k-1)\big]\Big)^{\frac{r}{2^h}}\notag\\
 \leq & \max\{a(k),b(k)\}^r\Big(2^{2^h-1}\E\left.\left[\|\H^*(k)\H(k)\|^{2^h}\right|\F(k-1)
 \right]+2^{2^h-1}\|\L_{\G}\otimes I_{\X}\|^{2^h}\Big)^{\frac{r}{2^h}}\notag\\
 \leq & 2^r\left(a^r(k)+b^r(k)\right)\bigg(\left(\E\left.\left[\|\H^*(k)\H(k)\|^{2^h}\right|
 \F(k-1)\right]\right)^{\frac{r}{2^h}}+\|\L_{\G}\otimes I_{\X}\|^{r}\bigg)\notag\\
 \leq& 2^r (a^r(k)+b^r(k))(\rho_0^r+\|\L_{\G}\otimes I_{\X}\|^r)\notag\\
 \leq & 2^r(a^r(k)+b^r(k))\rho_1^r~\text{a.s.},~\forall \ 1\leq r\leq 4,
\end{align}
where $\rho_1=\rho_0+\|\L_{\G}\otimes I_{\X}\|$. By the condition (\ref{dinglitiaojian}) and (\ref{cmllemfnn}), we get
\begin{align}
&\E\left[\|I_{\X^N}-\left(a(k)\H^*(k)\H(k)+b(k)\L_{\G}\otimes I_{\X}\right)\|^4|\F(k-1)\right]\notag\\
 =& \E\Big[\big\|  (I_{\X^N}- (a(k)\H^*(k)\H(k)
 +b(k)\L_{\G}\otimes I_{\X} ) )^4\big\||\F(k-1)\Big]\notag\\
 =& \E\Big[\big\|I_{\X^N}-4D(k)+6D^2(k)-4D^3(k) +D^4(k)\big\|
  |\F(k-1)\Big]\notag\\
 \leq & \E\Big[\|I_{\X^N}-4D(k)\|+6\|D(k)\|^2+4\|D(k)\|^3
  +\|D(k)\|^4 |\F(k-1)\Big]
 \leq   1+\gamma(k)~\text{a.s.},\notag
\end{align}
where $\gamma(k)=\Gamma(k)+6(a^2(k)+b^2(k))\rho_1^2+4(a^3(k)+b^3(k))\rho_1^3+(a^4(k)+b^4(k))\rho_1^4$. From Condition \ref{condition2} and the condition (\ref{dinglitiaojian}), we know that $\sum_{k=0}^{\infty}\gamma(k)<\infty$, which together with Theorem \ref{dingliyi1} implies that algorithm (\ref{algorithm}) is both mean square and almost surely strongly consistent.
$\hfill\qed$

\vskip 2mm

\emph{Proof of Corollary \ref{xiaosirendetuilun}.}
We will briefly sketch the proof. By the independence theory of random elements with values in topological   spaces and the measurability theory of mappings with values in topological   spaces, we
 construct $\HH_j,\ j\in \mathcal V$ satisfying the conditions in Theorem \ref{vnknoklfl} and verify the other conditions in Theorem \ref{vnknoklfl}.

  Noting that $\{\H(k),k\ge 0\}$ is an i.i.d. sequence with values in the space $(\mathscr L(\X^N,  \bigoplus_{i=1}^N\Y_i),\\ \tau_{\text{S}}(\mathscr L(\X^N, \bigoplus_{i=1}^N\Y_i)))$, from Definition \ref{dulixing}, it can be verified that   $\{\H^*(k)\H(k)x,\ x$ $\in\X^N, k\ge 0\}$ is an i.i.d. sequence with values in the Banach space $(\X^N,\tau_{\text{N}}(\X^N))$. By Proposition E.1.10 in \cite{hy2}, we know that $\{\|\H^*(k)\H(k)\|,k\ge 0\}$ and $\{\|I_{\X^N}-(a(k)\H^*(k)\H(k)+b(k)\L_{\G}\otimes I_{\X})\|,k\ge 0\}$ are both independent random sequences. By Condition \ref{condition2}, there exists an integer $s_0>0$, such that $a(k)+b(k)\leq (4\rho^2_0+4\|\L_{\G}\otimes I_{\X}\|)^{-1}$, $\forall\ k\ge s_0$. We define the nonnegative real sequence $\{\Gamma(k),k\ge 0\}$ as
\bna
\Gamma(k)=\begin{cases}
4(a(k)+b(k))\left(\rho^2_0+\|\L_{\G}\otimes I_{\X}\|\right),& 0\leq k<s_0;\\
0, & k\ge s_0,\notag
\end{cases}
\ena
which shows that $\sum_{k=0}^{\infty}\Gamma(k)<\infty$. Noting that $\{\H(k),k\ge 0\}$ are i.i.d. and
 $P\{\|\H(0)\|\leq \rho_0\}=1$, it can be verified that $P\{\|\H(k)\|\leq \rho_0\}=1$, $k=0,1,...$, and
$\\\|I_{\X^N}-4(a(k)\H^*(k)\H(k)+b(k)\L_{\G}\otimes I_{\X})\|\leq 1+4(a(k)+b(k))(\rho^2_0+\|\L_{\G}\otimes I_{\X}\|)~\text{a.s.}$, $\forall\ k\ge 0$. Then, we obtain
\bna
&&~~~\left.\E\left[\left\|I_{\X^N}-4\left(a(k)\H^*(k)\H(k)+b(k)\L_{\G}\otimes I_{\X}\right)\right\|\right|\F(k-1)\right]\cr
&&=\E\left[\left\|I_{\X^N}-4\left(a(k)\H^*(k)\H(k)+b(k)\L_{\G}\otimes I_{\X}\right)\right\|\right]\cr
&&=\E\bigg[\sup_{\|x\|=1}\left|\left\langle I_{\X^N}-4\left(a(k)\H^*(k)\H(k)+b(k)\L_{\G}\otimes I_{\X}\right)x,x\right\rangle \right|\bigg]\cr
&&=\E\bigg[\sup_{\|x\|=1}\left|1-4\left\langle\left(a(k)\H^*(k)\H(k)+b(k)\L_{\G}\otimes I_{\X}\right)x,x\right\rangle \right|\bigg]\cr
&&=\E\bigg[1-4\inf_{\|x\|=1}\left\langle\left(a(k)\H^*(k)\H(k)+b(k)\L_{\G}\otimes I_{\X}\right)x,x\right\rangle \bigg]\cr
 &&\leq 1~\text{a.s.},~\forall \ k\ge s_0.\notag
\ena
Then, we get
\begin{align}\label{taukcnt}
  &\E\big[ \|I_{\X^N}-4 (a(k)\H^*(k)\H(k)+b(k)\L_{\G}\otimes I_{\X} ) \| |\F(k-1)\big]\notag\\
  \leq & 1+\Gamma(k)~\text{a.s.},~\forall\ k\ge 0.
\end{align}

It can be verified that
\bna\label{vnkwenvefklw}
&&~~~\sup_{k\ge 0}\left(\E\left.\left[\|\H^*(k)\H(k)\|^{2^{\max\{h,2\}}}\right|\F(k-1)\right]\right)^{\frac{1}{2^{\max\{h,2\}}}}\cr
&&=\sup_{k\ge 0}\left(\E\left[\|\H^*(k)\H(k)\|^{2^{\max\{h,2\}}}\right]\right)^{\frac{1}{2^{\max\{h,2\}}}}\cr
&&=\sup_{k\ge 0}\left(\E\left[\|\H^*(0)\H(0)\|^{2^{\max\{h,2\}}}\right]\right)^{\frac{1}
{2^{\max\{h,2\}}}}
 \leq \rho^2_0~\text{a.s.}
\ena

It follows from $\|\H(0)\|\leq \rho_0~\text{a.s.}$ and Proposition \ref{nlllwwieiie}.(a) that $H^*_j(0)    H_j(0)x\in L^1(\Omega;\X)$, $x\in \X$. For the integer $h>0$ and $j\in \mathcal V$, we define the operator $\HH_j:\X\to\X$ by
\bna\label{wpkfpkpew}
\HH_j(x)=h\E\left[H^*_j(0)H_j(0)x\right],~x\in \X,~j\in \mathcal V.
\ena
For any given $x_1,x_2\in\X$ and $c_1,c_2\in \mathbb R$, we have
\bna\label{jcknvknw}
\HH_j(c_1x_1+c_2x_2)&=&c_1h\E\left[H^*_j(0)H_j(0)x_1\right]+c_2h\E\left[H^*_j(0)H_j(0)x_2\right]\cr
&=&c_1\HH_j(x_1)+c_2\HH_j(x_2).
\ena
Noting that $H^*_j(0)H_j(0)x\in L^1(\Omega;\X)$, by Proposition 2.6.31 in \cite{hy}, we get
\bna
\left\langle \HH_j(x_1), x_2\right\rangle&=&h\E\left[\left\langle H^*_j(0)H_j(0)x_1,x_2 \right\rangle \right]\cr
&=&h\E\left[\left\langle x_1,H^*_j(0)H_j(0)x_2 \right\rangle \right]\cr
&=&\left\langle x_1, \HH_j(x_2)\right\rangle,~j\in \mathcal V.\notag
\ena
This together with   (\ref{jcknvknw}) shows that $\HH_j$ is a linear self-adjoint operator, which together with (\ref{wpkfpkpew}) and
\ban
\|\HH_j(x)\|\leq h\E\left[\left\|H^*_j(0)H_j(0)\right\|\|x\|\right]\leq h\rho_0^2\|x\|,~\forall\ x\in \X,~j\in \mathcal V
\ean
 gives that the self-adjoint operator $\HH_j\in \mathscr L(\X)$. Denote $\HH_j(x):=\HH_jx$, $\forall\ x\in \X$. Noting that $H^*_j(0)H_j(0)x\in L^1(\Omega;\X)$, it follows from Proposition 2.6.31 in \cite{hy} that
\bna\label{owwwww2}
\left\langle \HH_jx,x\right\rangle&=&h\left\langle \E\left[H^*_j(0)H_j(0)x\right],x \right\rangle\cr
&=&h\E\left[\left\langle H^*_j(0)H_j(0)x,x \right\rangle \right]\cr
&=&h\E\left[\left\|H_j(0)x\right\|^2\right]\ge 0,
\ena
from which we know that the operator $\HH_j$ defined in (\ref{wpkfpkpew}) is positive bounded linear self-adjoint, $j\in \mathcal V$.
For any given integers $k\ge 0$, $h>0$ and $x\in \X$, if $A\in \F(kh-1)$, then it follows from Proposition \ref{nlllwwieiie} and (\ref{vnkwenvefklw}) that $H^*_j(i)H_j(i)(\1_A\otimes x)\in L^1(\Omega;\X)$, $i\ge kh$, which together with Lemma \ref{nvkvpeoeo} implies that $\E[H^*_j(i)H_j(i)(\1_A\otimes x)|\F(kh-1)]$ uniquely exists. Since $\{\H(k),k\ge 0\}$ and $\{v(k),k\ge 0\}$ are both i.i.d sequences and they are mutually independent,   from Definition \ref{dulixing}, it can be verified   that $H^*_j(i)H_j(i)x, \ j\geq kh$ are independent of $\1_{F\cap A}$ with $F\in \F(kh-1)$. Thus,
\ban
&&~~~\int_F\E\left.\left[\sum_{i=kh}^{(k+1)h-1}H^*_j(i)H_j(i)(\1_A\otimes x)\right|\F(kh-1)\right]\dd\P\cr
&&=\int_F\sum_{i=kh}^{(k+1)h-1}H^*_j(i)H_j(i)(\1_A\otimes x)\dd\P\cr
&&=\int_{\Omega}\left(\sum_{i=kh}^{(k+1)h-1}H^*_j(i)H_j(i)x\right)\1_{F\cap A}\dd\P\cr
&&=\int_{\Omega}\sum_{i=kh}^{(k+1)h-1}H^*_j(i)H_j(i)x\dd\P\int_{\Omega}\1_{F\cap A}\dd\P\cr
&&=\E\left[\sum_{i=kh}^{(k+1)h-1}H^*_j(i)H_j(i)x\right]\P(F\cap A)\cr
&&=h\E\left[H^*_j(0)H_j(0)x\right]\P(F\cap A)\cr
&&=h\int_{F}\E\left[H^*_j(0)H_j(0)x\right]\1_A\dd\P,~\forall \ F\in \F(kh-1),~j\in \mathcal V,
\ean
which gives
\ban
\E\left.\left[\sum_{i=kh}^{(k+1)h-1}H^*_j(i)H_j(i)(\1_A\otimes x)\right|\F(kh-1)\right]=\HH_j(\1_A\otimes x)~\text{a.s.},~j\in \mathcal V.
\ean
 This together with the properties of the conditional expectation, the operator $\HH_j$ and the linearity of Bochner integral leads to
\bna\label{vnkwoejvnvvn}
\HH_jy=\E\Bigg[\sum_{i=kh}^{(k+1)h-1}H^*_j(i)H_j(i)y\Bigg|\F(kh-1)\Bigg]~\text{a.s.},~j\in \mathcal V,
\ena
where $y\in L^0(\Omega,\F(kh-1);\X)$ is a simple function. For $f\in L^2(\Omega,\F(kh-1);\X)$, by Lemma \ref{vnwkelel}, we know that there exists a sequence of simple functions $\{f_n \in L^0(\Omega,\F(kh-1);\X),n\ge 0\}$ satisfying $\|f_n\|\leq \|f\|\text{a.s.}$ and $\lim_{n\to \infty}f_n=f~\text{a.s.}$, which together with (\ref{vnkwenvefklw}) and Cauchy inequality gives $H^*_j(i)H_j(i)f\in L^1(\Omega;\X)$. Thus, from Lemma \ref{nvkvpeoeo}, it is known that $\E[H^*_j(i)H_j(i)f|\F(kh-1)]$ uniquely exists. Noting that $\HH_j\in \LL(\X)$, it follows from (\ref{vnkwenvefklw})-(\ref{vnkwoejvnvvn}) and the dominated convergence theorem of conditional expectation that
\bna\label{vnowlelkel}
&&~~~\HH_jf\cr
&&=\lim_{n\to\infty}\HH_jf_n\cr &&=\lim_{n\to\infty}\E\left.\left[\sum_{i=kh}^{(k+1)h-1}H^*_j(i)H_j(i)f_n\right|\F(kh-1)\right]\cr
&&=\E\left.\left[\sum_{i=kh}^{(k+1)h-1}H^*_j(i)H_j(i)\lim_{n\to\infty}f_n\right|\F(kh-1)\right]\cr
&&=\E\left.\left[\sum_{i=kh}^{(k+1)h-1}H^*_j(i)H_j(i)f\right|\F(kh-1)\right]
~\text{a.s.},~j\in \mathcal V.
\ena
Let $\{x(k),\F(kh-1),k\ge 0\}$ be a $L_2$-bounded adaptive sequence with values in the Hilbert space $\X$. Then, by (\ref{vnowlelkel}), we get
\ban
\HH_jx(k)=\E\left.\left[\sum_{i=kh}^{(k+1)h-1}H^*_j(i)H_j(i)x(k)\right|\F(kh-1)\right]~\text{a.s.},~j\in\mathcal V.
\ean
For any non-zero element $x$ in Hilbert space $x$, from Proposition 2.6.31 in \cite{hy}, the condition (\ref{tuiluntiaojian}) and (\ref{wpkfpkpew}), we have
\bna
\left\langle \sum_{j=1}^N\HH_jx,x\right\rangle=h\left\langle\E\left[\sum_{j=1}^N
H^*_j(0)H_j(0)x\right],x\right\rangle=h\sum_{j=1}^N\E\left[\|H_j(0)x\|^2\right]>0.\notag
\ena
This together with  (\ref{taukcnt})-(\ref{vnkwenvefklw})  and Theorem \ref{vnknoklfl} gives  that the algorithm (\ref{algorithm}) is both mean square and almost surely strongly consistent.
$\hfill\qed$

\vskip 2mm

\emph{Proof of Theorem \ref{rkhsdingli}.}
 We will briefly sketch the proof. Firstly, by the measurability and the integration theory of mappings with values in Banach spaces, we prove that, for any $i \geq kh$, $\E[K_{x_j(i)}\otimes K_{x_j(i)}|\F(kh-1)]$ exists. Then, we develop a property of the conditional expectations of operator-valued random elements. By the  property and the reproducing property of RKHS,  we verify  the conditions in Theorem \ref{vnknoklfl}, thereby deriving the conclusion of the theorem.


%


It follows from Assumption \ref{assumption3}  and Lemma \ref{rabdomelerkhs} that Assumption \ref{assumption1} holds. By Assumption \ref{assumption4}, it is known that $\{v(k),\F(k),k\ge 0\}$ is the martingale difference sequence and there exists a constant $\b_v:=N\b>0$, such that $ \sup\limits_{k\ge 0}\E\big[\|v(k)\|^2|\F(k-1)\big]
 \leq   N\max_{i\in \mathcal V}\sup\limits_{k\ge 0}   \E \big[\|v_i(k)\|^2 |\F(k-1)\big]\leq \b_v~\text{a.s.},$
which implies that Assumption \ref{assumption2} holds.
Let $x\in \X$. Then, for any given $f_1,f_2\in \HH_K$ and $c_1,c_2\in \mathbb R$, we get
\begin{align}\label{vmkwmefkeml1}
 &\left(K_x\otimes K_x\right)(c_1f_1+c_2f_2)\cr
 =&c_1f_1(x)K_x+c_2f_2(x)K_x\cr
 =&c_1\left(K_x\otimes K_x\right)f_1+c_2\left(K_x\otimes K_x\right)f_2.
\end{align}
From Assumption \ref{assumption5} and the reproducing property of $\HH_K$, we know that  $\left\|\left(K_x\otimes K_x\right)f\right\|_{K} \leq \|K_x\|_K^2\|f\|_K=K(x,x)\|f\|_K
\leq   \sup_{x\in\X}K(x,x)\|f\|_K,~\forall\ f\in \HH_K,~\forall\ x\in \X.$
Thus, it follows from (\ref{vmkwmefkeml1}) that $K_x\otimes K_x\in \LL(\HH_K)$. Let $x=\sum_{i=1}^n\1_{A_i}\otimes x_i$ be a random vector with values in the Hilbert space $(\X,\tau_{\text{N}}(\X))$, where $A_i\cap A_j=\emptyset$, $1\leq i\neq j\leq n$. Noting that $K_{x}=\sum_{i=1}^n\1_{A_i}\otimes K_{x_i}$, we have $K_{x}\otimes K_{x}= \left(\sum_{i=1}^n\1_{A_i}\otimes K_{x_i}\right)\otimes \left(\sum_{i=1}^n\1_{A_i}\otimes K_{x_i}\right)
= \sum_{i=1}^n\1_{A_i}\otimes \left(K_{x_i}\otimes K_{x_i}\right),$
thus, $K_{x}\otimes K_{x}$ is a simple function with values in the Banach space $\LL(\HH_K)$. For any given random vector $x$ with values in the Hilbert space $(\X,\tau_{\text{N}}(\X))$, it is known that there exists a simple function sequence $\{x_n,n\ge 0\}$ with values in $\X$, such that $\lim_{n\to\infty}\|x-x_n\|=0~\text{a.s.}$
This together with the reproducing property of $\HH_K$ and the continuity of kernel $K$ gives
\bna\label{nncknkwnkwc1}
\lim_{n\to\infty}\left\|K_x-K_{x_n}\right\|_K=0~\text{a.s.}
\ena
It follows from   the reproducing property of $\HH_K$ that
\begin{align}
&\left\|K_{x}\otimes K_{x}-K_{x_n}\otimes K_{x_n}\right\|_{\LL(\HH_K)}\cr
\leq & \left\|(K_x-K_{x_n})\otimes K_x\right\|_{\LL(\HH_K)} +\left\|K_{x_n}\otimes (K_x-K_{x_n})\right\|_{\LL(\HH_K)}\cr
 =&\sup_{\|f\|_K=1}\left\|\left((K_x-K_{x_n})\otimes K_x\right)f\right\|_K +\sup_{\|f\|_K=1}\left\|\left(K_{x_n}\otimes (K_x-K_{x_n})\right)f\right\|_K\cr
 =& \sup_{\|f\|_K=1}\left\|f(x)(K_x-K_{x_n})\right\|_K +\sup_{\|f\|_K=1}\left\|(f(x)-f(x_n))K_{x_n}\right\|_K\cr
 \leq & \|K_x\|_K\|K_x-K_{x_n}\|_K+\|K_{x_n}\|_K\|K_x-K_{x_n}\|_K
 \leq   2\sup_{x\in \X}\sqrt{K(x,x)}\|K_x-K_{x_n}\|_K~\text{a.s.}\notag
\end{align}
Then, by Assumption \ref{assumption5} and (\ref{nncknkwnkwc1}), we have $\lim_{n\to\infty}\left\|K_{x}\otimes K_{x}-K_{x_n}\otimes K_{x_n}\right\|_{\LL(\HH_K)}  =0~\text{a.s.}$
Noting that $K_{x_n}\otimes K_{x_n}$ is a simple function with values in Banach space $\LL(\HH_K)$, by Lemma \ref{vnwkelel}, it is known that $K_{x}\otimes K_{x}$ is strongly $\P$-measurable with respect to the topology $\tau_{\text{N}}(\LL(\HH_K))$, and therefore $K_{x}\otimes K_{x}$ is a random element with values in Banach space $(\LL(\HH_K),\tau_{\text{N}}(\LL(\HH_K)))$. Thus, by Assumption \ref{assumption5}, we get $K_{x_j(i)}\otimes K_{x_j(i)}\in L^1(\Omega;\mathscr L(\HH_K))$, which together with Lemma \ref{nvkvpeoeo} gives that $\E[K_{x_j(i)}\otimes K_{x_j(i)}|\F(kh-1)]$ uniquely exists. Let $\{g(k),\F(kh-1),k\ge 0\}$ be the $L_2$-bounded adaptive sequence with values in $\HH_K$, by Assumption \ref{assumption5}, Proposition \ref{tiaojianqiwangxingzhi} and the condition (\ref{vnknknldldklsd}), we obtain
\begin{align}\label{vnkwenkfffnkfmkweklf}
&\sum_{j=1}^N\sum_{k=0}^{\infty}\E\Bigg[\bigg\|
\sum_{i=kh}^{(k+1)h-1}\E\left.\left[H_j^*(i)H_j(i)g(k)\right|
\F(kh-1)\right] -N_jg(k)\bigg\|_K^2\Bigg]\cr
 =&\sum_{j=1}^N\sum_{k=0}^{\infty}\E\Bigg[\bigg\|
 \sum_{i=kh}^{(k+1)h-1}\E\left.\left[K_{x_j(i)}\otimes K_{x_j(i)}g(k)\right|\F(kh-1)\right] -N_jg(k)\bigg\|_K^2\Bigg]\cr
 =& \sum_{j=1}^N\sum_{k=0}^{\infty}\E\Bigg[\bigg\|\Big(\sum_{i=kh}^{(k+1)h-1}\E\left.\left[K_{x_j(i)}\otimes K_{x_j(i)}\right|\F(kh-1)\right] -N_j\Big)g(k)\bigg\|_K^2\Bigg]
 <\infty.
\end{align}

Denote $\rho_0=N\left(
\sup_{x\in\X}K(x,x)\right)^{2^{\max\{h,2\}-1}}$.
Given the integer $h>0$, by Assumption \ref{assumption5}, we have
\begin{align}
& \E\left.\left[\|\H^*(k)\H(k)\|_{\mathscr L\left(\HH_K^N\right)}^{2^{\max\{h,2\}}}\right|\F(k-1)\right] \cr
\leq &
 N \E\left.\left[\sup_{i\in \mathcal V}\left\|K_{x_i(k)}\otimes K_{x_i(k)}\right\|_{\LL(\HH_K)}^{2^{\max\{h,2\}}}\right|\F(k-1)\right] \cr
\leq&
 N \E\left.\left[\sup_{i\in \mathcal V}\sup_{\|f\|_K=1}\left\|f(x_i(k))K_{x_i(k)}\right\|_K^{2^{\max\{h,2\}}}
 \right|\F(k-1)\right] \cr
 \leq &
N \E \Bigg[\sup_{i\in \mathcal V}\sup_{\|f\|_K=1}|f(x_i(k))|^{2^{\max\{h,2\}}}  \left\|K_{x_i(k)}
\right\|_K^{2^{\max\{h,2\}}}\Bigg|\F(k-1)\Bigg] \cr
 \leq & N \E\Bigg[\sup_{i\in \mathcal V}\sup_{\|f\|_K=1}\left\|f\right\|^{2^{\max\{h,2\}}}_K  \left(
\sup_{x\in\X}K(x,x)\right)^{2^{\max\{h,2\}-1}}\Bigg|\F(k-1)\Bigg]\cr
 \leq &\rho_0~\text{a.s.}, \notag
\end{align}
which gives
\begin{align}
 \sup_{k\ge 0}\left(\E\left.\left[\|\H^*(k)\H(k)\|_{\mathscr L\left(\HH_K^N\right)}^{2^{\max\{h,2\}}}\right|\F(k-1)\right]
 \right)^{\frac{1}{2^{
\max\{h,2\}}}}
 \leq  \rho_0~\text{a.s.} \label{zhuanzhihk}
\end{align}

It follows from Condition \ref{condition2} that there exists a constant $j_0>0$ and an integer $t_0>0$, such that $\sup_{k\ge 0}(4\rho_0a(k)+4\|\L_{\G}\|b(k))\leq j_0$ and $\sup_{k\ge t_0}(a(k) +b(k)) (4\rho_0+4\|\L_{\G}\|)\leq 1$. Noting that $ \|4a(k)\H^*(k)\H(k)+4b(k)  \L_{\G}\otimes I_{\HH_K} \|_{\LL\left(\HH_K^N\right)}\leq j_0$ if $k<t_0$, and $ \|4a(k)\H^*(k)\H(k)+4b(k)\L_{\G}\otimes I_{\HH_K} \|_{\LL\left(\HH_K^N\right)}\leq 1$, otherwise.
Then, we get
\begin{align}\label{vnksdkjdkdddd}
&\big\|I_{\HH_K^N}-4\left(a(k)\H^*(k)\H(k)+b(k)\L_{\G}\otimes I_{\HH_K}\right)\big\|_{\LL\left(\HH_K^N\right)}
\leq
1+\Gamma(k)
~\text{a.s.},
\end{align}
where
\ban
\Gamma(k)=
\begin{cases}
j_0, & k<t_0,\\
0, & k\ge t_0,
\end{cases}
\ean
and satisfies $\sum_{k=0}^{\infty}\Gamma(k)<\infty$. It follows from (\ref{vnksdkjdkdddd}) that
\begin{align}
 &\E\Big[\big\|I_{\HH_K^N}-4(a(k)\H^*(k)\H(k)
  +b(k)\L_{\G}\otimes I_{\HH_K})\Big\|_{\LL\left(\HH_K^N\right)}\big|\F(kh-1)\Big]
 \leq 1+\Gamma(k)~\text{a.s.}\notag
\end{align}
Hence, combining (\ref{vnkwenkfffnkfmkweklf})-(\ref{zhuanzhihk}), the above inequality and Theorem \ref{vnknoklfl} gives that the algorithm (\ref{rkhs}) is mean square and almost surely strongly consistent.
$\hfill\qed$

\vskip 2mm

\emph{Proof of Corollary \ref{vnlllleleeemmem}.}
It follows from Assumption \ref{assumption5} and Theorem \ref{rkhsdingli} that $K_{x_j(i)}\otimes K_{x_j(i)}\in L^1(\Omega;\mathscr L(\HH_K))$, which together with Lemma \ref{nvkvpeoeo} implies that $\E[K_{x_j(i)}\otimes $ $K_{x_j(i)}|\F(kh-1)]$ uniquely exists. Let $\{g(k),\F(kh-1),k\ge 0\}$ be a $L_2$-bounded adaptive sequence with values in $\HH_K$, by the condition (\ref{vnkmeeeemefffff}), we get
\ban
&&~~~~\sum_{j=1}^N\sum_{k=0}^{\infty}\E\left[\Bigg\|\Bigg(N_j-\sum_{i=kh}^{(k+1)h-1}\E\left.\left[K_{x_j(i)}\otimes K_{x_j(i)}\right|\F(kh-1)\right]\Bigg)g(k)\Bigg\|_K^2\right]\cr
&&\leq \sum_{j=1}^N\sum_{k=0}^{\infty}\E\left[\Bigg\|N_j-\sum_{i=kh}^{(k+1)h-1}\E\left.\left[K_{x_j(i)}\otimes K_{x_j(i)}\right|\F(kh-1)\right]\Bigg\|_{\LL(\HH_K)}^2\|g(k)\|_K^2\right]\cr
&&\leq \sum_{j=1}^N\sum_{k=0}^{\infty}\E\left[\max_{j\in\mathcal V}\Bigg\|N_j-\sum_{i=kh}^{(k+1)h-1}\E\left.\left[K_{x_j(i)}\otimes K_{x_j(i)}\right|\F(kh-1)\right]\Bigg\|_{\LL(\HH_K)}^2\|g(k)\|_K^2\right]\cr
&&\leq N\mu_0\sup_{k\ge 0}\E\left[\|g(k)\|_K^2\right]\sum_{k=0}^{\infty}\tau(k)
  <\infty,
\ean
where the last inequality is obtained from $\sum_{k=0}^{\infty}\tau(k)<\infty$. This together with Theorem \ref{rkhsdingli} implies that the algorithm (\ref{rkhs}) is both mean square and almost surely strongly consistent.

By Cauchy-Schwarz inequality and the reproducing property of $\HH_K$, we have
\begin{align}
|f_{i}(k)(x)-f_{0}(x)|=|\langle f_{i}(k)-f_{0}, K_{x}\rangle| \leq \left\|f_{i}(k)-f_{0}\right\|_{K}  \left\|K_{x}\right\|_{K}, \ \forall \ x \in \mathscr{X}, \ i \in \mathcal{V}.\notag
\end{align}
  Noting that algorithm (\ref{rkhs}) is   almost surely strongly consistent and by the above inequality, we have  $\lim\limits_{k\to \infty} f_{i}(k)(x)=f_{0}(x) $   \text{ a.s.},  $\forall \ x \in \mathscr{X}, \ i \in \mathcal{V}$.
$\hfill\qed$

\vskip 2mm

\emph{Proof of Corollary \ref{rkhsdinglijjjjj}.}
Noting that $\{(x_1(k),\cdots,x_N(k)),k\ge 0\}$ and $\{(v_1(k),\cdots ,v_N(k)),k\ge 0\}$ are both i.i.d. sequences and they are mutually independent, we have
\begin{align*}
&\sum_{i=kh}^{(k+1)h-1}\E\left.\left[K_{x_j(i)}\otimes K_{x_j(i)}\right|\F(kh-1)\right]  =  \sum_{i=kh}^{(k+1)h-1}\E\left[K_{x_j(i)}\otimes K_{x_j(i)}\right]
 =   h\E\left[K_{x_j(0)}\otimes K_{x_j(0)}\right].
\end{align*}
Denote $N_j:\HH_K\to \HH_K$ by
$
N_j=h\E\left[K_{x_j(0)}\otimes K_{x_j(0)}\right],~j\in \mathcal V.
$
It follows from Proposition 2.6.31 in \cite{hy} that $N_j\in \mathscr L(\HH_K)$ is a positive self-adjoint operator. Noting that
\ban
\bigg\|N_j-\sum_{i=kh}^{(k+1)h-1}\E\left.\left[K_{x_j(i)}\otimes K_{x_j(i)}\right|\F(kh-1)\right]\bigg\|_{\LL(\HH_K)}^2=0~\text{a.s.},
\ean
by the condition (\ref{nklnkle}), we get
\ban
\left\langle \sum_{j=1}^NN_jf,f\right\rangle _K =h\left\langle \E\left[\sum_{j=1}^NK_{x_j(0)}\otimes K_{x_j(0)}\right]f,f\right\rangle _K >0,
\ean
where $f$ is an arbitrary non-zero function in $\HH_K$. Hence, it follows from Corollary \ref{vnlllleleeemmem} that the algorithm (\ref{rkhs}) is both mean square and almost surely strongly consistent.
$\hfill\qed$

\section{Theoretical Framework of Random Elements with Values in a Topological Space}
\label{appendixb}
 \setcounter{equation}{0}
\renewcommand{\theequation}{B.\arabic{equation}}
The proposed algorithm involves the sequences of random forward operators induced by random input data. Conventionally, a random element with values in a Banach space is required to be strongly measurable, which is almost separably valued with respect to the topology induced by the norm in the Banach space (\cite{hy2}). It is known that operator-valued mappings may not be strongly measurable since the Banach space of bounded linear operators is generally non-separable with respect to the uniform operator topology (the topology induced by the operator norm) (\cite{hy}). In this section, we develop a self-contained  theoretical framework of random elements with values in a topological space.

\subsection{Random Elements with Values in a Topological Space}\label{sectionrandonelement}

\begin{definition}\label{jihukefenzhi}
Let $(\Omega,\F,\P)$ be a complete probability space, and $(\mathscr U,\tau)$ be a topological space. Given the mapping $f:\Omega\to \mathscr U$, if there exists a separable closed subset $\mathscr U_0$ of $\mathscr U$ and a subset $\Omega_0$ of $\Omega$ with probability measure  $1$, such that $$f(\Omega_0):=\{f(x):x\in \Omega_ 0\}\subseteq \mathscr U_0,$$ then $f$ is called  almost separably valued with respect to $\tau$.
\end{definition}


\begin{definition}\label{tuopukongjian}
Let $(\Omega,\F,\P)$ be a complete probability space and $(\mathscr U,\B(\mathscr U;\tau))$ be a measurable space, where $\tau$ is the topology on $\mathscr U$, and $\B(\mathscr U;\tau)$ is the Borel $\sigma$-algebra of the topological space $(\mathscr U,\tau)$, i.e., the smallest $\sigma$-algebra containing all open sets in $\mathscr U$. A mapping $f:\Omega\to \mathscr U$ is said to be strongly  $\P$-measurable or to be  a random element  with values in the topological space $(\mathscr U,\tau)$ if it is  $\F/\B(\mathscr U;\tau)$-measurable and almost separably valued with respect to $\tau$.
\end{definition}

\begin{definition}\label{fenbudingyi}
If $f$ is a random element with values in the topological space $(\mathscr U,\tau)$, then the distribution of $f$ is defined by the Borel probability measure $\mu_{f}(B):=\P(f^{-1}(B))$ on $(\mathscr U,\tau )$, $\forall\ B\in \B(\mathscr U;\tau)$.
\end{definition}


The following lemma is derived directly from Proposition 1.1.16 in \cite{hy}.

\begin{lemma}\label{vnwkelel}
Given a mapping $f:\Omega \to \X$ with values in Banach space $(\X,\tau_{\text{N}}(\X))$, $f$ is  strongly $\P$-measurable if and only if
  there exists a sequence of simple functions $\{f_n:\Omega \to \X,n\ge 1 \}$ such that $\lim_{n\to \infty}\|f_n -f \|_{\X}=0$ with probability  $1$.
\end{lemma}


It follows from Definition \ref{tuopukongjian} that for a  separable Banach space $(\X,\tau_{\text{N}}(\X))$,
 any $\F/\B(\X;\\ \tau_{\text{N}}(\X))$-measurable mapping $f:\Omega\to \X$ is a random element with values in $(\X,\tau_{\text{N}}(\X))$.

 \vskip 0.2cm
We give the following properties of operator-valued random elements.

\begin{proposition}\label{nlllwwieiie}
Let $(\X,\tau_{N}(\X)),\ (\Y,\tau_{N}(\Y))$ and $(\Z,\tau_{N}(\Z))$ be Banach spaces,  $f_0:\Omega\to \mathscr L(\mathscr Y,\mathscr Z)$, $f_1:\Omega\to \X$, $f_2:\Omega\to \LL(\X,\Y)$, $g_1:\Omega\to \LL(\X,\Y)$ and $g_2:\Omega\to \LL(\Y,\mathscr Z)$.\\
\indent  (a). If $(\Y,\tau_{N}(\Y))$ and $(\Z,\tau_{N}(\Z))$ are separable Hilbert spaces, then  $f_0$ is a random element with values in $(\mathscr L(\mathscr Y,\mathscr Z),\tau_{\text{S}}(\mathscr L(\mathscr Y,\mathscr Z)))$ if and only if the mapping $f_0y:\omega\mapsto f_0(\omega)y$ is a random element with values in $(\mathscr Z,\tau_{\text{N}}(\Z))$ for arbitrary $y\in \Y$.\\
\indent  (b). If $f_1$ is strongly $\P$-measurable with respect to $\tau_{\text{N}}(\X)$ and $g_1x:\Omega\to \mathscr Y$ is strongly $\P$-measurable with respect to $\tau_{\text{N}}(\Y)$, $\forall\ x\in \X$, then $g_1f_1:\Omega\to \mathscr Y$ is strongly $\P$-measurable with respect to $\tau_{\text{N}}(\Y)$.\\
\indent  (c). For arbitrary $x\in\X$ and $y\in \Y$, if the mapping $f_2x:\Omega\to \Y$ is strongly $\P$-measurable with respect to $\tau_{\text{N}}(\Y)$ and the mapping $g_2y:\Omega\to \mathscr Z$ is strongly $\P$-measurable with respect to $\tau_{\text{N}}(\mathscr Z)$, then the mapping $(g_2f_2)x':\Omega\to \mathscr Z$ is strongly $\P$-measurable with respect to $\tau_{\text{N}}(\mathscr Z)$ for arbitrary $x'\in \X$.
\end{proposition}
\begin{proof}
As $(\Y,\tau_{N}(\Y))$ and $(\Z,\tau_{N}(\Z))$ are separable Hilbert spaces, it can be verified that $(\mathscr L(\mathscr Y,\mathscr Z),\tau_{\text{S}}(\mathscr L(\mathscr Y,\mathscr Z)))$ is  separable. Therefore,  $f_0$ and $f_0y, \ \forall \ y\in \Y$ are almost separably valued with respect to $\tau_{\text{S}}(\mathscr L(\mathscr Y,\mathscr Z))$ and $\tau_{N}(\Z)$.

Let $T_1,T_2 \in \mathscr{L}(\Y,\Z)$ and $T_1 \neq T_2$, then there exists $y \in \Y$ such that $T_1 y \neq T_2 y$. Let $U_1 = \{ S \in \mathscr{L}(\Y,\Z) : \Vert Sy - T_1y \Vert_{\Z} < \epsilon \},\,  U_2 = \{ S \in \mathscr{L}(\Y,\Z) : \Vert Sy - T_2y \Vert_{\Z} < \epsilon \}$, where $\epsilon = \frac{1}{3} \Vert T_1 y - T_2 y \Vert_{\Z} > 0$. Notice that $U_1$ is a neighborhood of $T_1$, $U_2$ is a neighborhood of $T_2$ and $U_1 \cap U_2 = \emptyset$, thus $(\mathscr{L}(\Y,\Z), \tau_{\text{S}}(\mathscr L(\mathscr Y,\mathscr Z)) ) $ is a Hausdorff space.

Let $\tau_{\text{C}}(\mathscr{L}(\Y,\Z))$ be the compact-open topology on $\mathscr{L}(\Y,\Z)$, by the definition in \cite{AT}, $\tau_{\text{C}}(\mathscr{L}(\Y,\Z))$ has a sub-basis consisting of the sets $M(K,O) = \{ T \in \mathscr{L}(\Y,\Z): T(K) \subseteq O \}$, where $K \subseteq \Y $ is compact, and $O \subseteq \Z$ is open.

Next we will prove that $\tau_{\text{C}}(\mathscr{L}(\Y,\Z))$ is stronger than $\tau_{\text{S}}(\mathscr{L}(\Y,\Z))$, which is equivalent to prove the identity map $id : (\mathscr L(\mathscr Y,\mathscr Z),\tau_{\text{C}}(\mathscr L(\mathscr Y,\mathscr Z))) \to (\mathscr L(\mathscr Y,\mathscr Z),\tau_{\text{S}}(\mathscr L(\mathscr Y,\mathscr Z)))$ is continuous. Fix $T \in \LL(\Y,\Z)$, and let $\mathcal{O}_{\text{C}}(T)$ and $\mathcal{O}_{\text{S}}(T)$ be the neighborhood system of $T$ under compact-open topology and strong operator topology respectively. For arbitrary $A \in \mathcal{O}_{\text{S}}(T)$, by the definition of strong operator topology, there exists $y_1,\cdots,y_n \in \Y$ and $\epsilon >0$, such that $\cap_{i=1}^{n}\{S \in \LL(\Y,\Z) : \Vert S y_i - T y_i \Vert _{\Z} < \epsilon \} \subseteq A$. Let $K_i = \{ y_i\} \subseteq\Y ,i = 1,\cdots,n$ be the compact sets in $\Y$, and let $O_i = \{ z \in \Z: \Vert z - Ty_i \Vert_{\Z} < \epsilon \} \subseteq \Z, i = 1,\cdots, n $ be the open sets in $\Z$. Then we have $ T \in \cap_{i=1}^{n}M(K_i,O_i)$ and $\cap_{i=1}^{n}M(K_i,O_i) \in \mathcal{O}_{\text{C}}(T)$. Notice that
\begin{align*}
    id( \cap_{i=1}^{n}M(K_i,O_i) ) & =  id( \cap_{i=1}^{n} \{ S \in \mathscr{L}(\Y,\Z): S y_i \in O_i \}  ) \\
    & = \cap_{i=1}^{n} \{ S \in \mathscr{L}(\Y,\Z): \Vert S y_i  - T y_i \Vert _{\Z} < \epsilon \} \\
    & \subseteq A.
\end{align*}
This together with Propsition 1.4.3 in \cite{analysisnow} leads to that the identity map $id$ is continuous at $T$. By letting $T$ range over $\LL(\Y,\Z)$ gives that the identity map is continuous. Thus, $\tau_{\text{S}}(\mathscr L(\mathscr Y,\mathscr Z)) \subseteq \tau_{\text{C}}(\mathscr L(\mathscr Y,\mathscr Z))$, i.e. $\tau_{\text{C}}(\mathscr{L}(\Y,\Z))$ is stronger than $\tau_{\text{S}}(\mathscr{L}(\Y,\Z))$.

A key fact for us \cite[Theorem 7 p.112]{schwariz1973} is that the space $(\LL(\Y,\Z), \tau_{\text{C}}(\mathscr{L}(\Y,\Z)) )$ is a Lusin space. Since the strong operator topology $\tau_{\text{S}}(\mathscr{L}(\Y,\Z))$ is Hausdorff and is weaker than the compact-open topology $\tau_{\text{C}}(\mathscr{L}(\Y,\Z))$, it follows that $(\LL(\Y,\Z), \tau_{\text{S}}(\mathscr{L}(\Y,\Z)) )$ is also a Lusin space. Analogous to the proof of Theorem 2 in \cite{Johnson}, we know that   $f_0$ is $\F/\B( \mathscr L(\mathscr Y,\mathscr Z);\tau_{\text{S}}(\mathscr L(\mathscr Y,\mathscr Z)))$-measurable if and only if $f_0y:\omega\mapsto f_0(\omega)y$ is $\F/\B(\mathscr Z,\tau_{\text{N}}(\Z))$-measurable for arbitrary $y\in \Y$. These together with Definition \ref{tuopukongjian} give  that $f_0$ is a random element with values in $(\mathscr L(\mathscr Y,\mathscr Z),\tau_{\text{S}}(\mathscr L(\mathscr Y,\mathscr Z)))$ if and only if the mapping $f_0y:\omega\mapsto f_0(\omega)y$ is a random element with values in $(\mathscr Z,\tau_{\text{N}}(\Z))$ for arbitrary $y\in \Y$, which leads to (a).

(b) and (c) are directly from Proposition 1.1.28 and Corollary 1.1.29 in \cite{hy}.
\end{proof}

If the operator-valued mapping $f:\Omega\to \LL(\mathscr Y,\mathscr Z)$ is strongly $\P$-measurable with respect to the uniform operator topology $\tau_{\text{N}}(\mathscr L(\mathscr Y,\mathscr Z))$, then for arbitrary $y\in \mathscr Y$, $fy:\Omega\to \mathscr Z$ is strongly $\P$-measurable with respect to $\tau_{\text{N}}(\mathscr Z)$. Noting that $\tau_{\text{S}}(\mathscr L(\mathscr Y,\mathscr Z))\subseteq \tau_{\text{N}}(\mathscr L(\mathscr Y,\mathscr Z))$,
 it follows that a random element with values in $(\mathscr L(\mathscr Y,\mathscr Z),\tau_{\text{N}}(\mathscr L(\mathscr Y,\mathscr Z)))$ is a random element with values in
$(\mathscr L(\mathscr Y,\mathscr Z),\tau_{\text{S}}(\mathscr L(\mathscr Y,\mathscr Z)))$. For the measurability of mappings with values in the space of operators with different topologies, one may refer to \cite{hy} and  \cite{hp}-\cite{Blasco111}.

%





For a random element $f:\Omega\to \X$ with values in $(\X,\tau(\X))$, we denote the $\sigma$-algebra generated by $f$ in the sense of the topology $\tau(\X)$ by
$$\sigma(f;\tau(\X)):=\left\{f^{-1}(B):B\in \B(\X;\tau(\X))\right\}.$$
Based on the above definitions, we have the following proposition.

\begin{proposition}\label{wenknknkn}
If $f:\Omega\rightarrow{\mathscr{X}}$ is a random element with values in a separable Hilbert  space $(\X,\langle \cdot,\cdot\rangle _{\X})$,  $T:\Omega\rightarrow{\mathscr{L}(\mathscr{X},\mathscr{Y})}$ is a random element with values in a topological
 space $(\mathscr{L}(\mathscr{X},\mathscr{Y}),\tau_{\text{S}}(\mathscr{L}(\mathscr{X},\mathscr{Y}))),$ and  $(\Y,\langle \cdot,\cdot\rangle _{\Y})$ is a separable  Hilbert  space, then $Tf:\omega\rightarrow{T(\omega)f(\omega)}$ satisfies
 \begin{align}
    \sigma(Tf;\tau_{\text{N}}(\mathscr{Y}))\subseteq\sigma(T;\tau_\text{S}(\mathscr L(\X,\Y)) \bigvee \sigma(f;\tau_{\text{N}}(\mathscr{X})).\label{sgmmmms}
\end{align}
\end{proposition}

The proof of the  above proposition needs the following  lemma.

\begin{lemma}\label{randomsepa}
If $y:\Omega\rightarrow{\mathscr{Y}}$ is a random element with values in a separable  Hilbert  space $(\Y,\langle \cdot,\cdot\rangle _{\Y})$  with an orthonormal basis $\{f_j,\ j\geq1\}$, then
\begin{align}
\sigma(y;\tau_{\text{N}}(\Y))=\bigvee\limits_{j=1}^\infty \sigma\left(\langle y,f_j\rangle _{\Y};\tau_{\text{N}}\left(\mathbb{R}\right)\right).\label{randomsepa1}
\end{align}
\end{lemma}
\begin{proof}
At first, we prove
\begin{align}
\sigma(y;\tau_{\text{N}}(\Y))\subseteq \bigvee\limits_{j=1}^\infty \sigma\left(\langle y,f_j\rangle _{\Y};\tau_{\text{N}}\left(\mathbb{R}\right)\right).\label{randomsepa2}
\end{align}
For any $r>0,\ s\in\mathscr{Y},$ we have
\begin{align*}
    \left\{\omega\in \Omega:\|y(\omega)-s\| _{\Y}< r\right\}=&\left\{\omega\in \Omega:\|y(\omega)-s\| _{\Y}^2< r^2 \right\}\\
    =& \left\{\omega\in \Omega:\sum_{j=1}^{\infty}\langle y(\omega)-s,f_j\rangle _{\Y}^2< r^2\right\}\\
    =&\left\{\omega\in \Omega:\sum_{j=1}^{\infty}\left(\langle y(\omega),f_j\rangle _{\Y}-\langle s,f_j\rangle _{\Y} \right)^2< r^2\right\}\\
   \in&\bigvee\limits_{j=1}^\infty \sigma\left(\left(\langle y(\omega),f_j\rangle _{\Y}-\langle s,f_j\rangle _{\Y} \right)^2;\tau_{\text{N}}\left(\mathbb{R}\right)\right)\\
   \subseteq &\bigvee\limits_{j=1}^\infty \sigma\left(\langle y(\omega),f_j\rangle _{\Y};\tau_{\text{N}}\left(\mathbb{R}\right)\right).
\end{align*}
Therefore, we know that for any $r>0,\ s\in\mathscr{Y},$  $\{\omega\in \Omega:\|y(\omega)-s\| _{\Y}< r\} \in \bigvee\limits_{j=1}^\infty \sigma (\langle y,f_j\rangle _{\Y};\\ \tau_{\text{N}}\left(\mathbb{R}\right) )$. This together with the definitions of $\sigma(y;\tau_{\text{N}}(\Y))$ and $\bigvee\limits_{j=1}^\infty \sigma\left(\langle y,f_j\rangle _{\Y};\tau_{\text{N}}\left(\mathbb{R}\right) \right)$
gives (\ref{randomsepa2}).
 Next, we prove
 \begin{align}
\bigvee\limits_{j=1}^\infty \sigma\left(\langle y,f_j\rangle _{\Y};\tau_{\text{N}}\left(\mathbb{R}\right)\right)  \subseteq  \sigma(y;\tau_{\text{N}}(\Y)).\label{randomsepa3}
 \end{align}
 For $j=1,2,\ldots$,  denote $h_{j}(y)=\langle y,f_j\rangle _{\Y}$, which is
  a continuous linear functional and is measurable with respect to $y$.
   Then, for any $A\in\mathscr{B}(\mathbb{R};\tau_{\text{N}}\left(\mathbb{R}\right)),$ we have
\begin{align*}
    \{\omega\in \Omega:\langle y(\omega),f_j\rangle _{\Y}\in \mathit A\}=\{\omega\in \Omega: y(\omega)\in h_{j}^{-1}(\mathit A)\}\in \sigma(y;\tau_{\text{N}}(\Y)),\ j=1,2,\ldots
\end{align*}
Then, we have $ \sigma\left(\langle y,f_j\rangle _{\Y};\tau_{\text{N}}\left(\mathbb{R}\right)\right) \subseteq \sigma(y;\tau_{\text{N}}(\Y)),\ j=1,2,\ldots$ This together with the definition of $\sigma$-algebra gives  (\ref{randomsepa3}). Combining (\ref{randomsepa2}) and (\ref{randomsepa3}), we have (\ref{randomsepa1}).
\end{proof}

\emph{Proof of Proposition \ref{wenknknkn}:}
Suppose that $\{e_i,\ i\geq1\}$ and $\{f_j,\ j\geq1\}$ are the orthonormal bases of $\X$ and $\mathscr{Y}$, respectively.
By Lemma \ref{randomsepa}, we have
\begin{align}\label{chacheng}
\sigma(Tf;\tau_{\text{N}}(\mathscr{Y}))=\bigvee\limits_{j=1}^\infty \sigma\left(\langle Tf,f_j\rangle _{\Y};\tau_{\text{N}}\left(\mathbb{R}\right)\right).
\end{align}
 By the linearity of the inner product, we have
\begin{align*}
    &\langle Tf,f_j\rangle _{\Y}\\
    =&\left\langle T\left(\sum\limits_{i=1}^{\infty}\langle f,e_i\rangle _{\X} e_i\right),f_j\right \rangle_{\Y}\\
    =&\left\langle \sum\limits_{i=1}^{\infty}\langle f,e_i\rangle _{\X} Te_i,f_j\right\rangle _{\Y}\\
    =&\sum\limits_{i=1}^{\infty} \langle f,e_i\rangle_{\X} \left\langle Te_i,f_j\right\rangle_{\Y}, \ j=1,2,\ldots
\end{align*}
Note that for any  $i=1,2,\ldots$ and $j=1,2,\ldots$, $\sigma\left(\langle f,e_i\rangle_{\X}  \left\langle Te_i,f_j\right\rangle_{\Y};\tau_{\text{N}}\left(\mathbb{R}\right)\right) \subseteq \sigma(\langle f,e_i\rangle_{\X};\\ \tau_{\text{N}} (\mathbb{R} ))  \bigvee \sigma( \left\langle Te_i,f_j\right\rangle_{\Y}; \tau_{\text{N}} (\mathbb{R}))   \subseteq  \left(\bigcup_{i=1}^{\infty} \sigma(\langle f,e_i\rangle_{\X};\tau_{\text{N}}\left(\mathbb{R}\right))\right)\bigvee \left( \bigcup_{i=1}^{\infty}\sigma( \left\langle Te_i,f_j\right\rangle_{\Y};\tau_{\text{N}}\left(\mathbb{R}\right))\right)$. This together with the above equality gives
\begin{align*}
    &\sigma(\langle Tf,f_j\rangle_{\Y};\tau_{\text{N}}\left(\mathbb{R}\right))\\
    =&\sigma\left(\sum\limits_{i=1}^{\infty} \langle f,e_i\rangle_{\X} \left\langle Te_i,f_j\right\rangle_{\Y};\tau_{\text{N}}\left(\mathbb{R}\right)\right)\\
    \subseteq&\bigvee_{i=1}^{\infty} \sigma\left(\langle f,e_i\rangle_{\X} \left\langle Te_i,f_j\right\rangle_{\Y};\tau_{\text{N}}\left(\mathbb{R}\right)\right)\\
    \subseteq& \left(\bigcup_{i=1}^{\infty} \sigma(\langle f,e_i\rangle_{\X};\tau_{\text{N}}\left(\mathbb{R}\right))\right)\bigvee \left( \bigcup_{i=1}^{\infty}\sigma( \left\langle Te_i,f_j\right\rangle_{\Y};\tau_{\text{N}}\left(\mathbb{R}\right))\right) \\
    \subseteq& \left(\bigcup_{i=1}^{\infty} \sigma(\langle f,e_i\rangle_{\X};\tau_{\text{N}}\left(\mathbb{R}\right))\right)\bigvee \left(\bigcup_{i=1}^{\infty}\sigma(Te_i;\tau_{\text{N}}\left(\mathscr{Y}\right)
    )\right) \\
    \subseteq& \sigma(f;\tau_{\text{N}}(\mathscr{X}))\bigvee \left(\bigcup_{i=1}^{\infty}\sigma(Te_i;\tau_{\text{N}}\left(\mathscr{Y}\right))\right), \ j=1,2,\ldots,
\end{align*}
which  together with (\ref{chacheng}) leads to
\begin{align}
\sigma(Tf;\tau_{\text{N}}(\mathscr{Y}))\subseteq \sigma(f;\tau_{\text{N}}(\mathscr{X}))\bigvee \left(\bigcup_{i=1}^{\infty}\sigma(Te_i;\tau_{\text{N}}\left(\mathscr{Y}\right))
\right).\label{jixigema}
\end{align}
 Now, we prove $\sigma(Te_i;\tau_{\text{N}}\left(\mathscr{Y}\right)) \subseteq\sigma(T;\tau_\text{S}(\mathscr L(\X,\Y)), \  i=1,2,\ldots$
 For any $ i=1,2,\ldots$ and $y\in \Y$, define $T_{i,y}:\X \to \Y$ as $T_{i,y}h=\langle h,e_i\rangle_{\X} y, \ h\in \X$. Notice that  $T_{i,y}e_i=y$ and $T_{i,y} \in \mathscr L(\X,\Y)$. Then, for any $r>0$, we have
\begin{align}
    &\{\omega\in \Omega:\|T(\omega)e_i-y\|_{\Y}<r\}\notag\\=&\{\omega\in \Omega:\|T(\omega)e_i-T_{i,y} e_i\|_{\Y}<r\}.\notag
\end{align}
Then, we know that
\begin{align}
\left(T e_i\right)^{-1}(\{x\in \Y:\|x-y\|_{\Y}<r\})=T^{-1}(\{H\in\mathscr L(\X,\Y):\|He_i-T_{i,y} e_i\|_{\Y}<r\}).\label{teeee}
\end{align}
 Denote $O=\{H\in\mathscr L(\X,\Y):\|He_i-T_{i,y} e_i\|_{\Y}<r\}$.
For any $H\in O$ and $S_{H} \in \big\{K \in \mathscr L(\X,\Y): \| Ke_{i}-He_{i}\|_{\Y}<\frac{r-\|He_i-T_{i,y} e_i\|_{\Y}}{2} \big\},$ we have $$\|S_{H}e_{i}-T_{i,y} e_i \|_{\Y}\leq \|S_{H}e_{i}-H e_i \|_{\Y}+  \| He_{i}-T_{i,y}e_{i}\|_{\Y}<\frac{r+\|He_i-T_{i,y} e_i\|_{\Y}}{2} <r$$ and then $S_{H} \in O$. Therefore, for any $H\in O$, there exists $\epsilon=\frac{r-\|He_i-T_{i,y} e_i\|_{\Y}}{2}$  and $e_{i}\in \X$, such that $$\left\{K \in \mathscr L(\X,\Y): \| Ke_{i}-He_{i}\|_{\Y}<\epsilon \right\} \subseteq O.$$ This together with the definition of an open set in  $(\mathscr L(\X,\Y),\tau_\text{S}(\mathscr L(\X,\Y)))$ gives that  $\{H\in\mathscr L(\X,\Y):\|He_i-T_{i,y} e_i\|_{\Y}<r\}$ is an open set  in $(\mathscr L(\X,\Y),\tau_{\text{S}}(\mathscr L(\X,\Y)))$.  Note that $\{x\in \Y:\|x-y\|_{\Y}<r\}$ is also an open set in $(\Y,\tau_{\text{N}}\left(\mathscr{Y}\right))$. Then, by  (\ref{teeee}) and the definitions of $\sigma(Te_i;\tau_{\text{N}}\left(\mathscr{Y}\right))$ and $\sigma(T;\tau_\text{S}(\mathscr L(\X,\Y))$, we have $\sigma(Te_i;\tau_{\text{N}}\left(\mathscr{Y}\right)) \subseteq\sigma(T;\tau_\text{S}(\mathscr L(\X,\Y)), i=1,2,\ldots$ This together with (\ref{jixigema}) gives  (\ref{sgmmmms}).
$\hfill\qed$

\subsection{Conditional Expectation}\label{sectionconditionexpectation}

\begin{definition}[\cite{hy2}]\label{vnsllp}
\rm{If $f\in L^1(\Omega;\X)$, then the mathematical expectation of $f$ is defined as the Bochner integral $$ \E[f]=(\text{B})\int_{\Omega}f\dd\P.$$
}
\end{definition}

Without raising ambiguity, the Bochner integral in this paper will omit the capital letter (B) in front of the integral symbol.

\begin{definition}[\cite{hy}]
\rm{Let $\Gg\subseteq \F$ be a sub-$\sigma$-algebra, $f\in L^0(\Omega;\X)$ and $g\in L^0(\Omega,\Gg;\X)$. If
\ban
\int_Fg\dd\P=\int_Ff\dd\P,~\forall\ F\in \Gg_f\cap \Gg_g,
\ean
where $\Gg_{\phi}:=\{F\in \Gg:\1_F\phi\in L^1(\Omega;\X)\}$, $\phi\in L^0(\Omega;\X)$, then $g$ is called the conditional expectation of $f$ with respect to $\Gg$.
}
\end{definition}

\begin{definition}[\cite{hy}]
Let $\{\F_k, k\geq0\}$ be a family of sub-$\sigma$-algebras of $\F$ and $\{f_k,k\geq0\}$ be a family of random elements with values in the Hilbert space $(\X,\tau_{\text{N}}(\X))$.
\begin{itemize}
\item[1.] If $\F_n\subseteq \F_m$, $\forall\ m\geq n\geq 0$, then $\{\F_k,k\geq0\}$ is called a filter in $(\Omega,\F,\P)$.
\item[2.] If $\{\F_k, k\geq0\}$ is a filter in $(\Omega,\F,\P)$, $f_k\in L^0(\Omega,\F_k;\X)$, $\forall\ k\geq0$, then $\{f_k,\F_k,k\geq0\}$ is called an adaptive sequence.
\item[3.] If $\{f_k,\F_k, k\geq0\}$ is an adaptive sequence, $f_k$ is Bochner integrable on $\F_{k-1}$ and
\[\E[f_k|\F_{k-1}]=0,\ \forall\ k\geq0,\]
then $\{f_k,\F_k,k\geq0\}$ is called a sequence of martingale differences.
\end{itemize}
\end{definition}

It can be verified that if $\{f_k,\F_k, k\geq0\}$ is an adaptive sequence with values in  a separable Hilbert space $(\X,\tau_{\text{N}}(\X))$, then  $\{\|f_k\|_{\X},\F_k, k\geq0\}$ is an adaptive sequence with values in $(\mathbb{R},\tau_{\text{N}}(\mathbb{R}))$.

\begin{lemma}[\cite{hy}]\label{nvkvpeoeo}
If $\Gg\subseteq \F$ is a sub-$\sigma$-algebra and $f\in L^1(\Omega;\X)$, then there exists a unique conditional expectation $\E[f|\Gg]\in L^0(\Omega,\Gg;\X)\cap L^1(\Omega;\X)$ satisfying
\[\int_A\E[f|\Gg]\dd\P=\int_Af\dd\P,~\forall\ A\in \Gg.\]
\end{lemma}


For a Banach space $(\X,\tau_{\text{N}}(\X))$,
by Lemma \ref{nvkvpeoeo}, it follows that there exists a unique conditional expectation $\E[f|\Gg]\in L^0(\Omega,\Gg; \X)$ of the Bochner integrable random element $f$ with values in $(\X,\tau_{\text{N}}(\X))$, and $\E[f|\Gg]$ is a random element with values in $(\X,\tau_{\text{N}}(\X))$.
We have the following propositions of conditional expectations.

\begin{proposition}\label{vnwlssfweewwfew}
If $f\in L^1(\Omega;\mathscr L(\mathscr Y,\mathscr Z))$ is a random element with values in Banach space $(\mathscr L(\mathscr Y,\mathscr Z),\tau_{\text{N}}(\mathscr L(\mathscr Y,\mathscr Z))$, then $fy\in L^1(\Omega;\mathscr Z)$ is a random element with values in Banach space $(\mathscr Z,\tau_{\text{N}}(\mathscr Z))$, and $\E[fy]=\E[f]y$, $\forall\ y\in \mathscr Y$.
\end{proposition}
\begin{proof}
Since the random elements with values in the Banach space $(\LL(\Y,\Z)$, $\tau_{\text{N}}(\LL(\Y,\Z)))$ are certainly the random elements with values in $(\LL(\Y,\Z)$, $\tau_{\text{S}}(\LL(\Y,\Z)))$, it follows that $f\in L^1(\Omega;\LL(\Y,\Z))$ implies $fy\in L^1(\Omega;\Z)$, $\forall\ y\in \Y$ by Proposition (\ref{nlllwwieiie}).(a). Considering the simple function $\sum_{i=1}^n\1_{A_i}\otimes T_i$ with values in Banach space $\LL(\Y,\Z)$, where $A_i\in \F$, $T_i\in \LL(\Y,\Z)$, $A_i\cap A_j=\emptyset$, $1\leq i\neq j \leq n$, we have
\ban
\E\left[\left(\1_{A_i}\otimes T_i\right)y\right]=\P(A)(T_iy)=(\P(A_i)T_i)y=\E\left[\1_{A_i}\otimes T_i\right]y,~1\leq i\leq n,
\ean
which leads to
\bna\label{vnkwkekeneknek}
\E\left[\left(\sum_{i=1}^n\1_{A_i}\otimes T_i\right)y\right]&=&\sum_{i=1}^n\E\left[(\1_{A_i}\otimes T_i)y\right]\cr
&=&\sum_{i=1}^n\E\left[\ 1_{A_i}\otimes T_i\right]y\cr
&=&\E\left[\sum_{i=1}^n\1_{A_i}\otimes T_i\right]y.
\ena
Since $f\in L^1(\Omega;\LL(\Y,\Z))$ is a random element with values in the Banach space $(\LL(\Y,\Z),\\\tau_{\text{N}}(\LL(\Y,\Z)))$, it follows from Lemma \ref{vnwkelel} that there exists a sequence of simple functions $\{T_n,n \ge 1\}$ such that $\lim_{n\to \infty}\|f-T_n\|=0~\text{a.s.}$ and $\|T_n\|\leq \|f\|~\text{a.s.}$ Noting that $\|f-T_n\|\leq 2\|f\|\in L^1(\Omega)$ and $\lim_{n\to\infty}\|fy-T_ny\|=0~\text{a.s.}$, $\forall\ y\in \Y$, by the dominated convergence theorem, it follows that $\lim_{n\to\infty}\E[(f-T_n)y]=0$, $\forall\ y\in \Y$, and $\lim_{n\to\infty}\E[T_n]=\E[f]$. Thus, it follows from (\ref{vnkwkekeneknek}) that
\ban
\E[fy]=\lim_{n\to\infty}\E[(f-T_n)y]+\lim_{n\to\infty}\E[T_ny]=\lim_{n\to\infty}\E[T_n]y=\E[f]y,~\forall\ y\in \Y.
\ean
\end{proof}

\begin{proposition}\label{tiaojianqiwangxingzhi}
If $f\in L^2(\Omega;\mathscr L(\mathscr Y,\mathscr Z))$ is a random element with values in the Banach space $(\mathscr L(\mathscr Y,\mathscr Z),\tau_{\text{N}}(\mathscr L(\mathscr Y,\mathscr Z)))$, and $y\in L^2(\Omega,\Gg;\Y)$ is a random element with values in the Banach space $(\Y,\tau_{\text{N}}(\Y))$, where $\Gg$ is a sub-$\sigma$-algebra of $\F$, then $fy\in L ^1(\Omega;\mathscr Z)$ is a random element with values in the Banach space $(\mathscr Z,\tau_{\text{N}}(\mathscr Z))$ and $\E[fy|\Gg]=\E[f|\Gg]y~\text{a.s.}$
\end{proposition}
\begin{proof}
Since the random elements with values in the Banach space $(\LL(\Y,\Z)$, $\tau_{\text{N}}(\LL(\Y,\Z)))$ are  the random elements with values in the topological space $(\LL(\Y,\Z)$, $\tau_{\text{S}}(\LL(\Y,\Z)))$, it follows from Proposition \ref{nlllwwieiie}.(b) that $fy$ is a random element with values in the Banach space $(\Z,\tau_{\text{N}}(\Z))$, $f\in L^2(\Omega;\LL(\Y,\Z))$, and $y\in L^2(\Omega;\Y)$ implies $fy\in L^1(\Omega ;\Z)$. Consider the simple function $\sum_{i=1}^n\1_{A_i}\otimes y_i$ with values in the Banach space $\Y$, where $A_i\in \Gg$, $y_i\in \Y$, $A_i\cap A_j=\emptyset$, $1\leq i\neq j\leq n$. For any given $F\in \Gg$ and $1\leq i\leq n$, by Lemma \ref{nvkvpeoeo}, we have
\bna\label{fcwlmlmcc}
\int_F\E[f(\1_{A_i}\otimes y_i)|\Gg]\dd\P=\int_Ff(\1_{A_i}\otimes y_i)\dd\P=\int_{F\cap A_i}fy_i\dd\P=\E[(\1_{F\cap A_i}f)y_i].
\ena
Noting that $F\cap A_i\in \Gg$, by Lemma \ref{nvkvpeoeo} and Proposition \ref{vnwlssfweewwfew}, we get
\ban
\E[(\1_{F\cap A_i}f)y_i]=\E[\1_{F\cap A_i}f]y_i=\left(\int_{F\cap A_i}f\dd\P \right)y_i=\left(\int_{F\cap A_i}\E[f|\Gg]\dd\P \right) y_i.
\ean
Noting that $\1_{F\cap A_i}\E[f|\Gg]\in L^2(\Omega,\Gg;\LL(\Y,\Z))$, it follows from Lemma \ref{nvkvpeoeo} and Proposition \ref{vnwlssfweewwfew} that
\bna\label{vmkweknfffff}
\left(\int_{F\cap A_i}\E[f|\Gg]\dd\P\right) y_i=\E\left[\1_{F\cap A_i}\E[f|\Gg]y_i\right]=\int_F\E[f|\Gg](\1_{A_i}\otimes y_i)\dd\P.
\ena
Noting that $\E[f|\Gg](\1_{A_i}\otimes y_i)\in L^0(\Omega,\Gg;\Z)$, it follows from (\ref{fcwlmlmcc})-(\ref{vmkweknfffff}) and Lemma \ref{nvkvpeoeo} that
\bna\label{vcmklwleklekle}
\E\left.\left[f\left(\sum_{i=1}^n\1_{A_i}\otimes y_i\right)\right|\Gg\right]=\E[f|\Gg]\left(\sum_{i=1}^n\1_{A_i}\otimes y_i\right)~\text{a.s.}
\ena
Given a random element $y\in L^2(\Omega,\Gg;\Y)$ with values in the Banach space $(\Y,\tau_{\text{N}}(\Y))$, there exists a sequence of simple functions $\{y_n\in L^0(\Omega,\Gg;\Y),n\ge 1\}$ with values in the Banach space $\Y$ such that $\lim_{n\to\infty}\|y-y_n\|=0~\text{a.s.}$ and $\|y_n\|\leq \|y\|~\text{a.s.}$ Noting that $\|fy_n\|\leq \|f\|^2+\|y\|^2\in L^1(\Omega)$ and $\E[f|\G]\in L^2(\Omega,\Gg;\LL(\Y,\Z))$, by the dominated convergence theorem of conditional expectation and (\ref{vcmklwleklekle}), we have
\ban
\E[fy|\Gg]=\E\left.\left[f\left(\lim_{n\to\infty}y_n\right)\right|\Gg\right]=\lim_{n\to\infty}\E\left.\left[f y_n\right|\Gg\right]= \E\left.\left[f\right|\Gg\right]\left(\lim_{n\to\infty}y_n\right)=\E\left.\left[f\right|\Gg\right]y~\text{a.s.}
\ean
\end{proof}
\subsection{Independence and Conditional Independence}\label{sectionconinde}

In terms of the independence and conditional independence of random elements with values in a topological space, we have the following definitions.

\begin{definition}\label{dulixing}
Let $\mathscr I$ be a set of indices, $f_j$, $j\in \mathscr I$ be the random elements with values in $(\X_j,\tau(\X_j))$. If
$\P\left(f_{\a_1}\in B_1,\cdots,f_{\a_n}\in B_n\right)=\prod_{j=1}^n\P\left(f_{\a_j}\in B_j\right)$, for arbitrarily different indices $\a_1,\cdots,\a_n\in \mathscr I$ and arbitrary $B_1\in \B(\X_1;\tau(\X_{\a_1})),\cdots,B_n\in \B(\X_n;\tau(\X_{\a_n}))$,
then $f_j$, $j\in \mathscr I$ are said to be mutually independent.
\end{definition}

\begin{definition} \label{tiaojiandulixing}
Let $\mathcal F_{i}$, $i \in \mathscr I$ and $\Gg$ be sub-$\sigma$-algebra of $\F$. If $\P\left.\left(\bigcap_{j=1}^nA_j\right|\Gg\right)=\prod\limits_{j=1}^n\P\left(A_j|\Gg\right)$ a.s., for arbitrarily different indices $\a_1,\cdots,\a_n\in \mathscr I$ and arbitrary $A_1\in \F_{\a_1},\cdots,A_n\in \F_{\a_n}$,
then it is said that $\mathcal F_{\alpha}$ is conditionally independent  w.r.t. $\Gg$, $\alpha \in \mathscr I$. If $\sigma(f_{i};\tau(\X_i))$ is conditionally independent  w.r.t. $\Gg$, $i \in \mathscr I$, then we say that the random elements $f_{i}$, $i \in \mathscr I$ with values in $(\X_i,\tau(\X_i))$ are conditionally independent  w.r.t. $\Gg$.
\end{definition}

By Definitions \ref{dulixing} and \ref{tiaojiandulixing}, we have the following propositions.

\begin{proposition}\label{lemmaA4}
If the family of random elements $\{f_k,k\ge 1\}$ with values in $(\X_1,\tau(\X_1))$ and the family of random elements $\{g_k,k\ge 1\}$ with values in $(\X_2,\tau(\X_2))$ are independent, then $\bigvee_{i=k}^{\infty}\sigma(f_i;\tau(\X_1))$ is conditionally independent of  $\bigvee_{i=k}^{\infty}\sigma(g_i;\tau(\X_2))$ with respect to $\bigvee_{i=1}^{k-1}(\sigma(f_i;\tau(\X_1))\bigvee \sigma(g_i;\tau(\X_2)))$, $\forall\ k\ge 2$.
\end{proposition}
\begin{proof}
Given an integer $k\ge 2$, denote $\Ff^k=\bigvee_{i=k}^{\infty}\sigma(f_i;\tau(\X_1))$, $\Gg^k=\bigvee_{i=k}^{\infty}\sigma(g_i;\tau(\X_2))$,  $\Ff_k=\bigvee_{i=1}^{k-1}\sigma(f_i;\tau(\X_1))$, $\Gg_k=\bigvee_{i=1}^{k-1}\sigma(g_i;\tau(\X_2))$ and $\Ff(k)=\bigvee_{i=1}^{k-1}(\sigma(f_i;\tau(\X_1))\bigvee \sigma(g_i;\tau(\X_2))).$
Let $E\in \Ff^k$, $F\in \Gg^k$, $A\in \Ff_k$ and $B\in \Gg_k$. On the one hand, noting that $\{f_k,k\ge 1\}$ and $\{g_k,k\ge 1\}$ are mutually independent, we know that $\bigvee_{i=1}^{\infty}\sigma(f_i;\tau(\X_1))$ and $\bigvee_{i=1}^{\infty}\sigma(g_i;\tau(\X_2))$ are also mutually independent. Since $A\cap E\in \bigvee_{i=1}^{\infty}\sigma(f_i;\tau(\X_1))$ and $B\cap F\in \bigvee_{i=1}^{\infty}\sigma(g_i;\tau(\X_2))$, we get $\P(E\cap F\cap A\cap B)=\P(E\cap A)\P(F\cap B)$, which together with $A\cap B \in \Ff(k)$ gives
\bna\label{apendix1}
\int_{A\cap B}\P(E\cap F|\Ff(k))\dd\P=\P(E\cap F\cap A\cap B)=\P(E\cap A)\P(F\cap B).
\ena
On the other hand, noting that $\1_F\in \Gg^k$, $\Ff(k)=\Ff_k\bigvee\Gg_k$ and $\Gg^k\bigvee\Gg_k=\bigvee_{i=1}^{\infty}\sigma(g_i;\tau(\X_2))$ is independent of $\Ff_k$, by Corollary $7.3.3$ in \cite{chow}, we obtain $\E[\1_F|\Ff(k)]=\E[\1_F|\Gg_k]$, which further implies
$$
\E[\1_E|\Ff(k)]\E[\1_F|\Ff(k)]=\E[\1_E\E[\1_F|\Ff(k)]|\Ff(k)]=\E[\1_E\E[\1_F|\Gg_k]|\Ff(k)],
$$
from which we get
\begin{align}\label{apendix3}
\int_{A\cap B}\P(E|\Ff(k))\P(F|\Ff(k))\dd\P = &\int_{A\cap B}\E[\1_E\E[\1_F|\Gg_k]|\Ff(k)]\dd\P \notag\\
 =& \int_{A\cap B}\1_E\E[\1_F|\Gg_k]\dd\P.
\end{align}
Noting that $\1_B\in \Gg_k$, $\1_{A\cap E}$ and $\E[\1_{B\cap F}|\Gg_k]$ are mutually independent, we have
\bna\label{apendix4}
\int_{A\cap B}\1_E\E[\1_F|\Gg_k]\dd\P&=&\E[\1_{A\cap E}\E[\1_{B\cap F}|\Gg_k]]\cr &=&\E[\1_{A\cap E}]\E[\1_{B\cap F}]\cr
&=&\P(A\cap E)\P(B\cap F).
\ena
It follows from (\ref{apendix3})-(\ref{apendix4}) that
\bna\label{apendix5}
\int_{A\cap B}\P(E|\Ff(k))\P(F|\Ff(k))\dd\P=\P(A\cap E)\P(B\cap F).
\ena
Combining (\ref{apendix1}) and (\ref{apendix5}) gives
\bna\label{abcd}
\int_{A\cap B}\P(E\cap F|\Ff(k))\dd\P=\int_{A\cap B}\P(E|\Ff(k))\P(F|\Ff(k))\dd\P.
\ena
Denote
\[\Pi=\left.\left\{\bigcup_{i=1}^n(A_i\cap B_i)\right|A_i \in \Ff_k, B_i\in \Gg_k, A_i\cap B_i\cap A_j\cap B_j=\emptyset, 1\leq i\neq j\leq n\right\}.\]
If $\bigcup_{i=1}^n(A_i\cap B_i)\in \Pi$ and $\bigcup_{i=1}^m(C_i\cap D_i)\in \Pi$, then we have
\bna\label{apendix6}
&&~~~\left(\bigcup_{i=1}^n(A_i\cap B_i)\right)\bigcap \left(\bigcup_{i=1}^m(C_i\cap D_i)\right)\cr
&&=\bigcup_{j=1}^m\left(\left(\bigcup_{i=1}^n(A_i\cap B_i)\right)\cap C_j\cap D_j\right) =\bigcup_{j=1}^m\bigcup_{i=1}^n\left(A_i\cap B_i\cap C_j\cap D_j\right).
\ena
Noting that $A_i\cap B_i\cap A_j\cap B_j=\emptyset, 1\leq i\neq j\leq n$, and $C_s\cap D_s\cap C_t\cap D_t=\emptyset, 1\leq s\neq t\leq m$, we know that $A_i\cap B_i\cap A_j\cap B_j\cap C_s\cap D_s\cap C_t\cap D_t=\emptyset$, $(i,j)\neq (s,t)$, which together with (\ref{apendix6}) gives
$$
\left(\bigcup_{i=1}^n(A_i\cap B_i)\right)\bigcap \left(\bigcup_{i=1}^m(C_i\cap D_i)\right)\in \Pi,
$$
thus, we conclude that $\Pi$ is a $\pi$-class, i.e., $\Pi$ is closed under the intersection operation of the set. Denote
\[\Lambda=\left\{ M\in \Ff(k)\bigg|\int_{M}\P(E\cap F|\Ff(k))\dd\P=\int_{M}\P(E|\Ff(k))\P(F|\Ff(k))\dd\P\right\}.\]
Noting that $\Pi$ is composed of finite disjoint union of elements in $\Ff_k\cap \Gg_k$, it follows from  (\ref{abcd}) that $\Omega \in \Lambda$. If $\bigcup_{i=1}^n(A_i\cap B_i)\in \Pi$, then we have
$$
\int_{\bigcup_{i=1}^n(A_i\cap B_i)}\P(E\cap F|\Ff(k))\dd\P=\int_{\bigcup_{i=1}^n(A_i\cap B_i)}\P(E|\Ff(k))\P(F|\Ff(k))\dd\P,
$$
which shows that $\bigcup_{i=1}^n(A_i\cap B_i)\in \Lambda$ and further gives $\Pi \subseteq \Lambda$. If $M_i\in \Lambda$ and $M_i \subseteq M_{i+1}$, then we obtain
\ban
\int_{\Omega}\1_{M_i}\P(E\cap F|\Ff(k))\dd\P=\int_{\Omega}\1_{M_i}\P(E|\Ff(k))\P(F|\Ff(k))\dd\P,~\forall
\ i \ge 1.
\ean
Noting that $M_i \subseteq M_{i+1}$ implies $\lim_{i\to\infty}\1_{M_i}=\1_{\bigcup_{i=1}^{\infty}M_i}$, by Lebesgue dominated convergence theorem, we get
\ban
\int_{\Omega}\1_{\bigcup_{i=1}^{\infty}M_i}\P(E\cap F|\Ff(k))\dd\P=\int_{\Omega}\1_{\bigcup_{i=1}^{\infty}M_i}
\P(E|\Ff(k))\P(F|\Ff(k))\dd\P.
\ean
Thus, we have $\bigcup_{i=1}^{\infty}M_i\in \Lambda$. If $M_1,M_2\in \Lambda$ and $M_1\subseteq M_2$, then $M_2=M_1\cup (M_2 \setminus M_1)$ implies
\ban
\left(\int_{M_1}+\int_{M_2\setminus M_1}\right)\P(E\cap F|\Ff(k))\dd\P=\left(\int_{M_1}+\int_{M_2\setminus M_1}\right)\P(E|\Ff(k))\P(F|\Ff(k))\dd\P.
\ean
It follows from $M_1\in \Lambda$ that
\ban
\int_{M_2\setminus M_1}\P(E\cap F|\Ff(k))\dd\P=\int_{M_2\setminus M_1}\P(E|\Ff(k))\P(F|\Ff(k))\dd\P,
\ean
which shows that $M_2\setminus M_1\in \Lambda$. From the above analysis, we know that $\Lambda$ is a $\lambda$-class, by Theorem 1.3.2 in \cite{chow}, we get $\sigma(\Pi)\subseteq \Lambda$. On the one hand, it follows from $\Pi \subseteq \Ff(k)=\Ff_k\bigvee\Gg_k$ that $\sigma(\Pi)\subseteq \Ff(k)$. On the other hand, noting that $\Ff_k\subseteq \Pi$ and $\Gg_k\subseteq \Pi$, we have $\Ff_k\cup\Gg_k\subseteq \sigma(\Pi)$, which leads to $\Ff(k)=\Ff_k\bigvee\Gg_k \subseteq \sigma(\Pi)$ and further gives $\Ff(k)=\sigma(\Pi)\subseteq \Lambda$. Noting that $E\in \Ff^k$ and $F\in \Gg^k$, we obtain
\ban
\int_{M}\P(E\cap F|\Ff(k))\dd\P=\int_{M}\P(E|\Ff(k))\P(F|\Ff(k))\dd\P,~\forall \ M\in \Ff(k),
\ean
which together with $\P(E|\Ff(k))\P(F|\Ff(k))\in \Ff(k)$ gives
\ban
\P(E\cap F|\Ff(k))=\P(E|\Ff(k))\P(F|\Ff(k))~\text{a.s.},
\ean
thus, we conclude that $\Ff^k$ is conditionally independent of $\Gg^k$ with respect to $\Ff(k)$.
\end{proof}

\begin{proposition} \label{lemmaA6}
Let $\Gg$ be the sub-$\sigma$-algebra of $\F$. The random element $T:\Omega\to \mathscr L (\X,\mathscr Y)$ with values in $(\mathscr L(\X,\mathscr Y),\tau_{\text{S}}(\mathscr L(\X,\mathscr Y)))$ satisfies $\E[\|T\|_{\mathscr L(\X,\mathscr Y)}^2]<\infty$ and $f\in L^2(\Omega;\X)$ is a random element with values in the Banach space $(\X,\tau_{\text{N}}(\X))$. If  $T$  is conditionally independent of  $f$   w.r.t. $\Gg$, then
$\E[Tf|\Gg]=\E[T\E[f|\Gg]|\Gg]$ a.s.
\end{proposition}

Before proving Proposition \ref{lemmaA6}, we introduce the following lemma.
\begin{lemma}\label{lemmaA5}
\rm{
Let $\Gg$ be a sub-$\sigma$-algebra of $\F$, $A\in \F$ and $f\in L^1(\Omega;\X)$. If $f$ is conditionally  independent of $\1_A$ with respect to $\Gg$, then
$\E[f\1_A|\Gg]=\E[f|\Gg]\E[\1_A|\Gg]~\text{a.s.}$}
\end{lemma}
\begin{proof}
Let $f=\1_B\otimes x$, where $B\in \F$ and $x\in \X$. Noting that
\bna\label{appendix3}
\sigma(\1_B\otimes x;\tau_{\text{N}}(\X))=\left\{(\1_B\otimes x)^{-1}(E):E\in \mathscr B(\X;\tau_{\text{N}}(\X))\right\},
\ena
it follows that
\bna\label{appendix4}
(\1_B\otimes x)^{-1}(E)=
\begin{cases}
\Omega, & \text{If $0\in E$, $x\in E$};\\
B, &  \text{If $0\in E$, $x\notin E$};\\
B^{\complement}, & \text{If $0\notin E$, $x\in E$};\\
\emptyset, & \text{If $0\notin E$, $x\notin E$}.
\end{cases}
,~
\forall \ E\in \mathscr B(\X;\tau_{\text{N}}(\X)),
\ena
which leads to $\sigma(\1_B\otimes x;\tau_{\text{N}}(\X))=\sigma(\1_B;\tau_{\text{N}}(\X))$. Since $\1_B\otimes x$ is conditionally independent of $\1_A$ with respect to $\Gg$, it follows from Definition \ref{tiaojiandulixing} that $\sigma(\1_B\otimes x;\tau_{\text{N}}(\X))$ is conditionally independent of  $\sigma(\1_A;\tau_{\text{N}}(\X))$ with respect to $\Gg$, from which we can further conclude that $\sigma(B)$ is conditionally independent of $\sigma(A)$ with respect to $\Gg$. Given $F\in \Gg$, on the one hand, we have
\bna\label{appendix1111}
\int_F\E[f\1_A|\Gg]\dd\P&=&x\int_F\E[\1_{A\cap B}|\Gg]\dd\P\cr
&=&x\int_F\P(A\cap B|\Gg)\dd\P\cr
&=&x\int_F\P(A|\Gg)\P(B|\Gg)\dd\P.
\ena
On the other hand, noting that $\E[\1_B\otimes x|\Gg]=\E[\1_B|\Gg]\otimes x~\text{a.s.}$, we get
\bna\label{appendix2222}
\int_F\E[f|\Gg]\E[\1_A|\Gg]\dd\P&=&\int_F\E[\1_B\otimes x|\Gg]\E[\1_A|\Gg]\dd\P
 = x\int_F\P(A|\Gg)\P(B|\Gg)\dd\P.
\ena
Then, by (\ref{appendix1111})-(\ref{appendix2222}), we obtain
\bna\label{appendix5}
\E[(\1_B\otimes x)\1_A|\Gg]=\E[\1_B\otimes x|\Gg]\E[\1_A|\Gg]~\text{a.s.}
\ena
Let $f=\sum_{i=1}^n\1_{B_i}\otimes x_i$, where $B_i \in \F$, $B_i\cap B_j=\emptyset$, $x_i\in \X$, $1\leq i\neq j\leq n$. Following the same way as (\ref{appendix3})-(\ref{appendix4}), we have
\ban
\sigma\left(\sum_{i=1}^n\1_{B_i}\otimes x_i;\tau_{\text{N}}(\X)\right)=\sigma\left(\bigcup_{i=1}^n B_i\right).
\ean
Thus, $f$ is conditionally independent of $\1_{A}$ with respect to $\Gg$, which implies that $\sigma(\bigcup_{i=1}^n B_i)$ is conditionally independent of $\sigma(A)$ with respect to $\Gg$, from which we know that $B_i$ is conditionally independent of $A$ with respect to $\Gg$, $1\leq i\leq n$. It follows from (\ref{appendix5}) that
\bna\label{appendix8}
\int_F\E\left.\left[\sum_{i=1}^n(\1_{B_i}\otimes x_i)\1_A\right|\Gg\right]\dd\P=\sum_{i=1}^n\int_F\E\left.\left[\1_{B_i}\otimes x\right|\Gg\right]\E[\1_A|\Gg]\dd\P,
\ena
which leads to
\ban
\E\left.\left[\sum_{i=1}^n(\1_{B_i}\otimes x_i)\1_A\right|\Gg\right]=\E\left.\left[\sum_{i=1}^n\1_{B_i}\otimes x\right|\Gg\right]\E[\1_A|\Gg]~\text{a.s.}
\ean
For the random element $f$ with values in Banach space $(\X,\tau_{\text{N}}(\X))$, there exists a sequence of simple functions with values in $\X$ such that $\lim_{n\to\infty}f_n=f~\text{a.s.}$ and $\|f_n\|\leq \|f\|~\text{a.s.}$ Noting that $\|f_n\1_A\|\leq \|f\|~\text{a.s.}$ together with $f\in L^1(\Omega;\X)$ implies that $\E[\|f\|]<\infty$, by the dominated convergence theorem, we get
\bna \label{appendix6}
\int_F\E\left.\left[\lim_{n\to\infty}f_n\1_A\right|\Gg\right]\dd\P=\int_F\lim_{n\to\infty}\E[f_n\1_A|\Gg]\dd\P,
\ena
and
\ban
\int_F\E\left.\left[\lim_{n\to\infty}f_n\right|\Gg\right]\E[\1_A|\Gg]\dd\P=\int_F\lim_{n\to\infty}\E[f_n|\Gg]\E[\1_A|\Gg]\dd\P.
\ean
It follows from (\ref{appendix8}) that
\bna\label{appendix9}
\int_F\E[f_n\1_A\Gg]\dd\P=\int_F\E[f_n|\Gg]\E[\1_A|\Gg]\dd\P.
\ena
Thus, by (\ref{appendix6})-(\ref{appendix9}), we have
\ban
\int_F\E[f\1_A|\Gg]\dd\P=\int_F\E[f|\Gg]\E[\1_A|\Gg]\dd\P,~\forall F\in \Gg.
\ean
\end{proof}

\begin{proof}[Proof of Proposition \ref{lemmaA6}]
We first consider the case with $f=\sum_{i=1}^n\1_{A_i}\otimes x_i$, where $A_i\in \F$, $A_i\cap A_j=\emptyset$, $x_i\in \X$, $1\leq i\neq j\leq n$. Since $T$ is conditionally independent of $f$ with respect to $\Gg$, it follows that $\sigma(Tx_i;\tau_{\text{N}}(\Y))$ is conditionally independent of $\sigma(\sum_{i=1}^n\1_{A_i}\otimes x_i;\tau_{\text{N}}(\X))$ with respect to $\Gg$, $1\leq i\leq n$, from which we know that $\sigma(Tx_i;\tau_{\text{N}}(\Y))$ is conditionally independent of $\sigma(A_i)$ with respect to $\Gg$, thus, $Tx_i$ is conditionally independent of $\1_{A_i}$ with respect to $\Gg$, then by Lemma \ref{lemmaA5}, we have
\bna\label{appendix11}
\E[T(\1_{A_i}\otimes x_i)|\Gg]&=&\E[(Tx_i)\1_{A_i}|\Gg]\cr &=&\E[Tx_i|\Gg]\E[\1_{A_i}|\Gg]\cr
&=&\E[(Tx_i)\E[\1_{A_i}|\Gg]|\Gg]~\text{a.s.}
\ena
Noting that
\bna\label{appendix12}
\E[(Tx_i)\E[\1_{A_i}|\Gg]|\Gg]=\E[T(\E[\1_{A_i}|\Gg]\otimes x_i)|\Gg]=\E[T\E[\1_{A_i}\otimes x_i|\Gg]|\Gg]~\text{a.s.}
\ena
By (\ref{appendix11})-(\ref{appendix12}), we have
\bna\label{appendix14}
\E[Tf|\Gg]&=&\sum_{i=1}^n\E[T(\1_{A_i}\otimes x_i)|\Gg]\cr &=&\sum_{i=1}^n\E[T\E[\1_{A_i}\otimes x_i|\Gg]|\Gg]\cr
&=&\E[T\E[f|\Gg]|\Gg]~\text{a.s.}
\ena
Given a random element $f\in L^2(\Omega;\X)$ with values in Banach space $(\X,\tau_{\text{N}}(\X))$, by  Lemma \ref{vnwkelel}, we know that there exists a sequence $\{f_n,n\ge 1\}$ of simple functions, such that $\lim_{n\to\infty}f_n=f~\text{a.s.}$ and $\|f_n\|\leq \|f\|~\text{a.s.}$ If $\E[\|T\|^2]<\infty$ and $f\in L^2(\Omega;\X)$, noting that $T(\omega)\in \mathscr L(\X,\mathscr Y)$, we get $\|(Tf_n)(\omega)\|=\|T(\omega)f_n(\omega)\|\leq \|T(\omega)\|\|f_n(\omega)\|\leq \|T(\omega)\|^2+\|f(\omega)\|^2\in L^1(\Omega)$, which together with conditional dominated convergence theorem gives
\bna\label{appendix13}
\E[Tf|\Gg]=\E\left.\left[T\left(\lim_{n\to\infty}f_n\right)
\right|\Gg\right]=\E\left.\left[\lim_{n\to\infty}Tf_n\right|\Gg\right]
=\lim_{n\to\infty}\E[Tf_n|\Gg]~\text{a.s.}
\ena
By (\ref{appendix14}), we obtain
\ban
\E[Tf_n|\Gg]=\E[T\E[f_n|\Gg]|\Gg]~\text{a.s.}
\ean
It follows from $\|T\E[f_n|\Gg]\|\leq \|T\|\|\E[f_n|\Gg]\|\leq \|T\|^2+\E[\|f_n\|^2|\Gg]\leq \|T\|^2+\E[\|f\|^2|\Gg]  \in L^1(\Omega)$ and conditional dominated convergence theorem that
\bna\label{appendix16}
\lim_{n\to\infty}\E[T\E[f_n|\Gg]|\Gg]=\E\left.\left[\lim_{n\to\infty}T\E[f_n|\Gg]\right|\Gg\right]=\E\left.\left[T\E\left.\left[\lim_{n\to\infty}f_n\right|\Gg\right]\right|\Gg\right]~\text{a.s.}
\ena
Hence, by (\ref{appendix13})-(\ref{appendix16}), we get
$\E[Tf|\Gg]=\E[T\E[f|\Gg]|\Gg]$ a.s.
\end{proof}





\section{Proofs of Lemmas in Sections 3-4}\label{appendixc}
 \setcounter{equation}{0}
\renewcommand{\theequation}{C.\arabic{equation}}

For the convenience of notational writing without giving rise to ambiguity, the subscripts of the parametrization in Banach spaces and the subscripts of the inner product in Hilbert spaces will be omitted in the sequel.

\begin{proof}[Proof of Lemma \ref{wendingxing}]
Given the initial value $x(0)\in \X_1$, by  Proposition \ref{nlllwwieiie}.(a)-(c), we know that $x(k)$ is a random element with values in the Hilbert space $(\X_1,\tau_{\text{N}}(\X_1))$. It follows from the random difference equation (\ref{chafen}) that
$x(k+1) = \big(\prod_{i=0}^k(I_{\X_1}-F(i))\big)x(0) +\sum_{i=0}^k\big(\prod_{j=i+1}^k(I_{\X_1}-F(j))\big)G(i)u(i),~k\ge 0.$
Then, by   Cauchy inequality, we have
\begin{align}\label{ijfwg}
&\E[\|x(k+1)\|^2]\notag\\
\leq& 2\E\left[\left\| \prod_{i=0}^k(I_{\X_1}-F(i)) x(0)
\right\|^2\right] +2\E\left[\left\|\sum_{i=0}^k\left(\prod_{j=i+1}^k(I_{\X_1}-F(j))\right)G(i)u(i)\right\|^2\right].
\end{align}
It follows from Definition \ref{tiaojiandulixing} that $\sigma(u(k);\tau_{\text{N}}(\X_2))$ is independent of $\sigma\left(F(k);\tau_{\text{S}}(\LL(\X_1))\right)\bigvee\\ \sigma\left(G(k);\tau_{\text{S}}(\LL(\X_2,\X_1))\right),$  and it can be verified that $\sigma(\|u(k)\|)$ is independent of $\sigma\left(\|F(k)\|\right)\bigvee\\ \sigma\left(\|G(k)\|\right),$ which leads to
\begin{align}\label{xxll}
 &\E\left[\left\|G(t)u(t)\right\|^2\right]\notag\\
\leq& \E\left[\|G(t)\|^2\right]\E\left[\|u(t)\|^2\right]\notag\\
\leq& \sup_{k\ge 0}\E\left[\|u(k)\|^2\right]\E\left[\|G(t)\|^2\right],~t\ge 0.
\end{align}
By the condition (\ref{ssafe}), we get $\sup_{k\ge 0}\E[\|G(k)\|^2]<\infty,$ thus, (\ref{xxll}) implies $\sup\limits_{k\ge 0}\E[\|G(k)u(k)\|^2]<\infty$. It can be verified $\|I_{\X_1}-F(j)\|\in F(k-1), \forall \ j<k-1$. Therefore, for $0\leq t<k$, by the condition (\ref{qqqqq}), we obtain
\begin{align}\label{xxoopp}
&\E\left.\left[\prod_{j=t+1}^k\left\|I_{\X_1}-F(j)\right\|^4\right|\F(t)\right]\cr
=&\E\left.\left[\E\left.\left[\prod_{j=t+1}^k\left\|I_{\X_1}-F(j)\right\|^4\right|\mathcal F(k-1)\right]\right|\F(t)\right]\cr
=&\E\bigg[\E\left[\left\|I_{\X_1}-F(k)\right\|^4\big|\mathcal F(k-1)\right]  \prod_{j=t+1}^{k-1}\|I_{\X_1}-F(j)\|^4\Big|\F(t)\bigg]\cr
\leq & (1+\gamma(k))\E\left.\left[\prod_{j=t+1}^{k-1}\|I_{\X_1}-F(j)\|^4\right|\F(t)\right]\cr
\leq & \prod_{j=t+1}^k(1+\gamma(j))~\text{a.s.}
\end{align}
It follows from $\sup_{k\ge 0}\E[\|G(k)\|^4]<\infty$ and (\ref{xxoopp}) that
\begin{align}\label{oopp}
&\E\Bigg[\bigg\|\bigg(\prod_{j=s+1}^k(I_{\X_1}-F(j))\bigg)^*
\bigg(\prod_{j=t+1}^k(I_{\X_1}-F(j))\bigg)  G(t)\bigg\|^2\Bigg]\notag\\
 \leq & \E\Bigg[\bigg(\prod_{j=t+1}^k\|I_{\X_1}-F(j)\|^4\bigg)
 \bigg(\prod_{j=s+1}^t\|I_{\X_1}-F(j)\|^2\bigg)  \|G(t)\|^2\Bigg]\notag\\
 \leq & \left(\prod_{j=t+1}^k(1+\gamma(j))\right)
 \E\Bigg[\bigg(\prod_{j=s+1}^t\|I_{\X_1}-F(j)\|^2\bigg)
 \|G(t)\|^2\Bigg]\cr
 \leq & \left(\prod_{j=t+1}^k(1+\gamma(j))\right)
 \Bigg(\E\Bigg[\prod_{j=s+1}^t\|I_{\X_1}-F(j)\|^4\Bigg]
 +\E\left[\|G(t)\|^4\right]\Bigg)\cr
 \leq & \left(\prod_{j=t+1}^k(1+\gamma(j))\right) \left(\prod_{j=s+1}^t(1+\gamma(j))+\sup_{k\ge 0}\E\left[\|G(k)\|^4\right]\right)\cr
 \leq &\left(\prod_{k=0}^{\infty}(1+\gamma(k))\right) \left(1+\sup_{k\ge 0}\E\left[\|G(k)\|^4\right]\right)
 <\infty, \  ~0\leq s<t\leq k,
\end{align}
where the last inequality is due to $\sum_{k=0}^{\infty}\gamma(k)<\infty$. By $\sup_{k\ge 0}\E[\|u(k)\|^2]<\infty$,   Proposition \ref{nlllwwieiie} and (\ref{oopp}), we have
\begin{align}\label{vnokllllll}
  \left(\prod_{j=s+1}^k(I_{\X_1}-F(j))\right)^*\left(
\prod_{j=t+1}^k(I_{\X_1}-F(j))\right)G(t)u(t)
  \in L^1(\Omega;\X_1),~0\leq s<t\leq k.
\end{align}
It can be verified that $\left(\prod_{j=s+1}^k(I_{\X_1}-F(j))\right)^*  \left(\prod_{j=t+1}^k(I_{\X_1}-F(j))\right)G(t)$ is conditionally independent of $u(t)$ w.r.t. $\F(s),\  0\leq s<t$. Then, noting that $\E[u(t)|\F(s)]=\E[\E[u(t)|\mathcal F(t-1)]|\F(s)]=0$, $0\leq s<t$, and Proposition \ref{wenknknkn}  implies $G(s)u(s)\in L^0(\Omega,\F(s);\X^N)$, it follows from (\ref{oopp})-(\ref{vnokllllll}), Proposition \ref{nlllwwieiie}.(a)-(c),  Proposition 2.6.31 in \cite{hy}  and Propositions \ref{lemmaA4}-\ref{lemmaA6} that
\begin{align}\label{fwwe}
&\E\Bigg[\Bigg\langle  \prod_{j=s+1}^k(I_{\X_1}-F(j)) G(s)u(s), \prod_{j=t+1}^k(I_{\X_1}-F(j)) G(t)u(t)\Bigg\rangle \Bigg]\notag\\
=& \E\Bigg[\Bigg\langle G(s)u(s), \left(\prod_{j=s+1}^k(I_{\X_1}-F(j))\right)^*\left(\prod_{j=t+1}^k(I_{\X_1}-F(j))
\right)G(t)u(t)\Bigg\rangle\Bigg]\notag\\
=& \E\Bigg[\E\Bigg[\Bigg\langle G(s)u(s), \left(\prod_{j=s+1}^k(I_{\X_1}-F(j))\right)^* \left(\prod_{j=t+1}^k(I_{\X_1}-F(j))\right)G(t)u(t)\Bigg\rangle\bigg|\F(s) \Bigg]\Bigg]\cr
 =& \E\Bigg[\Bigg\langle G(s)u(s), \E\Bigg[\left(\prod_{j=s+1}^k(I_{\X_1}-F(j))\right)^*  \left(\prod_{j=t+1}^k(I_{\X_1}-F(j))\right)G(t)u(t)\bigg|\F(s) \Bigg]\Bigg\rangle\Bigg]\cr
 =&\E\Bigg[\bigg\langle G(s)u(s), \E\bigg[\left(\prod_{j=s+1}^k(I_{\X_1}-F(j))\right)^*\notag\\
 &\times\left(\prod_{j=t+1}^k(I_{\X_1}-F(j))\right)G(t)\E[u(t)|\F(s)]\bigg|\F(s) \bigg]\bigg\rangle\Bigg]\notag\\
 =& 0.
\end{align}
 Denote $\Lambda=\{i\in \mathbb N:\E[\|G(i)u(i)\|^2]>0\}$. By (\ref{ijfwg}), (\ref{xxll}) and (\ref{fwwe}), we obtain
\begin{align}\label{wiiie}
&\E\left[\|x(k+1)\|^2\right]\cr
=&\E\left[\left\|\left(\prod_{i=0}^k(I_{\X_1}-F(i))\right)x(0)\right\|^2\right]+\sum_{i=0}^k\E\left[\left\|\left(\prod_{j=i+1}^k(I_{\X_1}-F(j))\right)G(i)u(i)\right\|^2\right]\cr
 \leq & \E\left[\left\|\left(\prod_{i=0}^k(I_{\X_1}-F(i))\right)x(0)\right\|^2\right]+\sum_{i=0}^{\infty}\E\left[\left\|\left(\prod_{j=i+1}^k(I_{\X_1}-F(j))\right)G(i)u(i)\right\|^2\right]\cr
 =&\E\left[\left\|\left(\prod_{i=0}^k(I_{\X_1}-F(i))
 \right)x(0)\right\|^2\right]\notag\\
 &+\sum_{i\in \Lambda}\E\left[\|G(i)u(i)\|^2\right] \E\left[\left\|\left(\prod_{j=i+1}^k(I_{\X_1}-F(j))\right)\frac{G(i)u(i)}{\left(\E\left[\|G(i)u(i)\|^2\right]\right)^{\frac{1}{2}}}\right\|^2\right]\cr
 \leq & \E\left[\left\|\left(\prod_{i=0}^k(I_{\X_1}-F(i))\right)x(0)\right\|^2\right]\notag\\
 &+\sup_{k\ge 0}\E\left[\|u(k)\|^2\right]\sum_{i\in \Lambda}\E\left[\|G(i)\|^2\right]\E\Bigg[\bigg\| \prod_{j=i+1}^k(I_{\X_1}-F(j))
 \frac{G(i)u(i)}{\left(\E\left[\|G(i)u(i)\|^2\right]\right)^{\frac{1}{2}}}
 \bigg\|^2\Bigg].
\end{align}
Noting that the operator-valued random sequence $\{I_{\X_1}-F(k),k\ge 0\}$ is $L_2^2$-stable with respect to the filter $\{\F(k),k\ge 0\}$, we get
\bna\label{oowf}
\lim_{k\to \infty}\E\left[\left\|\left(\prod_{i=0}^k(I_{\X_1}-F(i))\right)x(0)\right\|^2\right]=0.
\ena
By Proposition \ref{wenknknkn}, we know that $G(i)u(i)\in L^0(\Omega,\mathcal F(i);  \X_1)$ and $\E\left[\left\|\frac{G(i)u(i)}{\left(\E\left[\|G(i)u(i)\|^2\right]\right)^{\frac{1}{2}}}\right\|^2\right]=1,~i\in \Lambda,$
which leads to
\begin{align}\label{wwml}
&\lim_{k\to\infty}\E\left[\left\|\left(\prod_{j=i+1}^k(I_{\X_1}-F(j))\right)
\frac{G(i)u(i)}{\left(\E\left[\|G(i)u(i)\|^2\right]
\right)^{\frac{1}{2}}}\right\|^2\right]\notag\\
=&0,~ i\in \Lambda.
\end{align}
It can be verified that $\|G(i)u(i)\| \in \mathcal F(i)$  from
  $G(i)u(i)\in L^0(\Omega,\mathcal F(i);\X_1)$. Then, by (\ref{xxoopp}), we have
\begin{align}\label{llnnv}
&\sup_{k\ge 0\atop i\in \Lambda}\E\left[\left\|\left(\prod_{j=i+1}^k(I_{\X_1}-F(j))\right)\frac{G(i)u(i)}{\left(\E\left[\|G(i)u(i)\|^2\right]\right)^{\frac{1}{2}}}\right\|^2\right]\cr
 =&\sup_{k\ge 0\atop i\in \Lambda}\E\Bigg[\E\Bigg[\bigg\|
 \bigg(\prod_{j=i+1}^k(I_{\X_1}-F(j))\bigg)\frac{G(i)u(i)}{\left(\E\left[\|G(i)u(i)\|^2\right]\right)^{\frac{1}{2}}}
 \bigg\|^2\bigg|\mathcal F(i)\Bigg]\Bigg]\cr
 \leq &\sup_{k\ge 0\atop i\in \Lambda}\E\Bigg[\E\Bigg[\bigg\|\prod_{j=i+1}^k(I_{\X_1}
 -F(j))\bigg\|^2
 \bigg|\mathcal F(i)\Bigg]\left\|\frac{G(i)u(i)}{\left(\E\left[\|G(i)u(i)\|^2\right]
 \right)^{\frac{1}{2}}}\right\|^2\Bigg]\cr
 \leq & \sup_{k\ge 0\atop i\in \Lambda}\E\Bigg[\E\Bigg[\bigg\|\prod_{j=i+1}^k(I_{\X_1}
 -F(j))\bigg\|^4\bigg|\mathcal F(i)\Bigg]^{\frac{1}{2}}\left\|\frac{G(i)u(i)}{\left(\E\left[\|G(i)u(i)
 \|^2\right]\right)^{\frac{1}{2}}}\right\|^2\Bigg]\cr
 \leq & \sup_{k\ge 0\atop i\in \Lambda}\E\left[\sqrt{\prod_{j=i+1}^k(1+\gamma(j))}\left\|\frac{G(i)u(i)}{\left(\E\left[\|G(i)u(i)\|^2\right]\right)^{\frac{1}{2}}}\right\|^2\right]\cr
 \leq & \sqrt{\prod_{k=0}^{\infty}(1+\gamma(k))}\sup_{i\in \Lambda}\E\left[\left\|\frac{G(i)u(i)}{\left(\E\left[\|G(i)u(i)\|^2\right]\right)^{\frac{1}{2}}}\right\|^2\right]\cr
 =&\sqrt{\prod_{k=0}^{\infty}(1+\gamma(k))}<\infty.
\end{align}
By the condition (\ref{ssafe}-\ref{ssafe1}), (\ref{wwml})-(\ref{llnnv}) and Lemma \ref{lemma6}, we obtain
\begin{align}\label{llks}
 \lim_{k\to \infty}\sum_{i\in \Lambda}& \E\left[\|G(i)\|^2\right] \E\Bigg[\bigg\| \prod_{j=i+1}^k(I_{\X_1}-F(j))\frac{G(i)u(i)}{\left(\E\left[\|G(i)u(i)\|^2\right]\right)^{\frac{1}{2}}}
 \bigg\|^2\Bigg]=0.
\end{align}
Thus, substituting (\ref{oowf}) and (\ref{llks}) into (\ref{wiiie}) gives $\lim_{k\to\infty}\E[\|x(k)\|^2]=0$.
\end{proof}

To prove Lemma \ref{jihubiranshoulian}, we need the following lemma.
\begin{lemma}\label{yibanxingdejieguo}
For the algorithm (\ref{algorithm}), suppose that Assumptions \ref{assumption1}, \ref{assumption2}, Conditions \ref{condition1} and \ref{condition2} hold. If there exists a  fixed-length time period $h>0$ and a constant $\rho_0>0$ such that\\
\indent (i) $$ \left\{I_{\X^N}-\sum_{i=kh}^{(k+1)h-1}(a(i)\H^*(i)\H(i)+b(i)\L_{\G}\otimes I_{\X}),k\ge 0\right\}$$
is $L_2^2$-stable  w.r.t.  $\{\F((k+1)h-1), k\ge 0\}$;\\
\indent (ii) $\displaystyle \sup_{k\ge 0}\left(\E\left.\left[\|\H^*(k)\H(k)\|_{\LL\left(\X^N\right)}^{2^{\max\{h,2\}}}\right|\F(k-1)\right]\right)^{\frac{1}{2^{\max\{h,2\}}}}\leq \rho_0~\text{a.s.,}$\\
then $\{I_{\X^N}-a(k)\H^*(k)\H(k)-b(k)\L_{\G}\otimes I_{\X},k\ge 0\}$ is $L_2^2$-stable w.r.t. $\{\F(k),k\ge 0\}$,
\end{lemma}
\begin{proof}
Given the $L_2$-bounded adaptive sequence $\{x(k),\F(kh-1),k\ge 0\}$ with values in the Hilbert space $\X^N$ and the nonnegative integer $m$, we define a new sequence $\{u(k),k\ge 0\}$ by
\bna\label{diedaishi}
u(k+1)=\Phi_P((k+1)h-1,kh)u(k),~k\ge m,
\ena
where $u(m)=x(m),u(i)=0,i=0,\cdots,m-1$. It follows from Proposition \ref{nlllwwieiie}.(a)-(c) that $\{u(k),k\ge 0\}$ is a random sequence with values in the Hilbert space $(\X^N,\tau_{\text{N}}(\X^N))$. On one hand, from (\ref{diedaishi}), by iterative calculations, we get
\bna\label{wffe}
u(k+1) =& \left(\prod_{i=m}^k\Phi_P((i+1)h-1,ih)\right)u(m)\notag\\
 =& \Phi_P((k+1)h-1,mh)x(m),~k\ge m.
\ena
Noting that $x(m)\in L^0(\Omega,\F(mh-1);\X^N)$, it is known from Lemma \ref{lemma1} that
\ban
\E\left[\|u(k+1)\|^2\right]=
\E\left.\left[\E\left[\|\Phi_P((k+1)h-1,mh)x(m)\|^2\right|\F(mh-1)\right]\right]\leq d_1\E\left[\|x(m)\|^2\right],
\ean
thus, $\sup_{k\ge 0}\E[\|x(k)\|^2]<\infty$ implies $\sup_{k\ge 0}\E[\|u(k)\|^2]<\infty$. On the other hand, we can rewrite (\ref{diedaishi}) as
\ban
&&u(i+1)=\left(I_{\X^N}-\sum_{s=ih}^{(i+1)h-1}(a(s)\H^*(s)\H(s)+b(s)\L_{\G}\otimes I_{\X})\right)u(i)\cr
&&+\left(\Phi_P((i+1)h-1,ih)-\left(I_{\X^N}-\sum_{s=ih}^{(i+1)h-1}\left(a(s)\H^*(s)\H(s)+b(s)\L_{\G}\otimes I_{\X}\right)\right)\right)u(i),
\ean
which leads to
\bna\label{fwwii}
u(k+1)&=&\left(\prod_{i=m}^k\left(I_{\X^N}-\sum_{s=ih}^{(i+1)h-1}(a(s)\H^*(s)\H(s)+b(s)\L_{\G}\otimes I_{\X})\right)\right)x(m)\cr
&&+\sum_{i=m}^k\left(\prod_{j=i+1}^k\left(I_{\X^N}-\sum_{s=jh}^{(j+1)h-1}(a(s)\H^*(s)\H(s)+b(s)\L_{\G}\otimes I_{\X})\right)\right)\cr
&&\times \Bigg(\Phi_P((i+1)h-1,ih)-\Bigg(I_{\X^N}-\sum_{s=ih}^{(i+1)h-1}(a(s)\H^*(s)\H(s)\cr
&&+b(s)\L_{\G}\otimes I_{\X})\Bigg)\Bigg)u(i).
\ena
Denote the $s$-th order term in the binomial expansion of $\Phi_P((i+1)h-1,ih)$ by $M_s(i)$, $s=2,\cdots,h$. By (\ref{wffe})-(\ref{fwwii}), we have
\ban
&&~~~~\Phi_P((k+1)h-1,mh)x(m)\cr
&&=\left(\prod_{i=m}^k\left(I_{\X^N}-\sum_{s=ih}^{(i+1)h-1}(a(s)\H^*(s)\H(s)+b(s)\L_{\G}\otimes I_{\X})\right)\right)x(m)\cr
&&+\sum_{i=m}^k\left(\prod_{j=i+1}^k\left(I_{\X^N}-\sum_{s=jh}^{(j+1)h-1}(a(s)\H^*(s)\H(s)+b(s)\L_{\G}\otimes I_{\X})\right)\right)\left(\sum_{s=2}^hM_s(i)\right)u(i),
\ean
from which we get
\bna\label{ikddw}
&&~~~~\E\left[\left\|\Phi_P((k+1)h-1,mh)x(m)\right\|^2\right]\cr
&&\leq 2\E\left[\left\|\left(\prod_{i=m}^k\left(I_{\X^N}-\sum_{s=ih}^{(i+1)h-1}(a(s)\H^*(s)\H(s)+b(s)\L_{\G}\otimes I_{\X})\right)\right)x(m)\right\|^2\right]\cr &&~~~~+2\E\left[\left\|\sum_{i=m}^k\left(\prod_{j=i+1}^k\left(I_{\X^N}-\sum_{s=jh}^{(j+1)h-1}(a(s)\H^*(s)\H(s)+b(s)\L_{\G}\otimes I_{\X})\right)\right)\right.\right.\cr &&~~~~\times\left.\left.\left(\sum_{s=2}^hM_s(i)\right)u(i)\right\|^2\right].
\ena
Noting that $x(m)\in L^0(\Omega,\F(mh-1);\X^N)$ and $\{I_{\X^N}-\sum_{i=kh}^{(k+1)h-1}(a(i)\H^*(i)\H(i)+b(i)\L_{\G}\otimes I_{\X}),k\ge 0\}$ is $L_2^2$-stable with respect to the filter $\{\F((k+1)h-1),k\ge 0\}$, we obtain
\ban
\lim_{k\to\infty}\E\left[\left\|\left(\prod_{i=m}^k\left(I_{\X^N}-\sum_{s=ih}^{(i+1)h-1}(a(s)\H^*(s)\H(s)+b(s)\L_{\G}\otimes I_{\X})\right)\right)x(m)\right\|^2\right]=0.
\ean
Hereafter, we will analyze the second term on the right-hand side of the inequality in (\ref{ikddw}). Denote
$
D(s)=a(s)\H^*(s)\H(s)+b(s)(\L_{\G}\otimes I_{\X})$. On one hand, for $2\leq r\leq 2^h,ih\leq s\leq (i+1)h-1$, by Cr-inequality and the conditional Lyapunov inequality, we have
\bna\label{jssk}
&&~~~~\E\left.\left[\|D(s)\|^r\right|\F(ih-1)\right]\cr
&&\leq \left(\E\left.\left[\|a(s)\H^*(s)\H(s)+b(s)(\L_{\G}\otimes I_{\X})\|^{2^h}\right|\F(ih-1)\right]\right)^{\frac{r}{2^h}}\cr
&&\leq \max\{a(s),b(s)\}^r\left(2^{2^h-1}\E\left.\left[\|\H^*(s)\H(s)\|^{2^h}\right|\F(ih-1)\right]\right.\cr
&&\hspace{0.5cm}+\left. 2^{2^h-1}\|\L_{\G}\otimes I_{\X}\|^{2^h}\right)^{\frac{r}{2^h}}\cr
&&\leq2^r\left(a^r(s)+b^r(s)\right)\left(\left(\E\left.\left[\|\H^*(s)\H(s)\|^{2^h}\right|\F(ih-1)\right]\right)^{\frac{r}{2^h}}+\|\L_{\G}\otimes I_{\X}\|^{r}\right)\cr
&&\leq 2^r(a^r(s)+b^r(s))(\rho_0^r+\|\L_{\G}\otimes I_{\X}\|^r)\cr
&& \leq 2^r(a(s)+b(s))^r(\rho_0+\|\L_{\G}\otimes I_{\X}\|)^r~\text{a.s.}
\ena
Denote $\rho_1=\rho_0+\|\L_{\G}\otimes I_{\X}\|$. For $ih\leq n_1<\cdots< n_r\leq (i+1)h-1$, by the conditional H\"{o}lder inequality, the conditional Lyapunov inequality and (\ref{jssk}), we get
\bna\label{ijjw}
&&~~~~\E\left.\left[\left\|\prod_{j=1}^rD(n_j)\right\|^2\right|\F(ih-1)\right]\cr
&&\leq \left(\E\left.\left[\left\|\prod_{j=1}^{r-1}D(n_j)\right\|^4\right|\F(ih-1)\right]\right)^{\frac{1}{2}}\left(\E\left.\left[\left\|D(n_r)\right\|^4\right|\F(ih-1)\right]\right)^{\frac{1}{2}}\cr
&&\leq \rho_1^2(a(n_r)+b(n_r))^2\left(\E\left.\left[\left\|\prod_{j=1}^{r-1}D(n_j)\right\|^4\right|\F(ih-1)\right]\right)^{\frac{1}{2}}\cr
&&\leq \rho_1^{2r}\prod_{j=1}^r(a(n_j)+b(n_j))^2~\text{a.s.}
\ena
On the other hand, for $2\leq s\leq h$, it follows from Condition \ref{condition1} and (\ref{ijjw}) that
\bna\label{wwffw}
&&~~~~\E\left.\left[\|M_s(i)\|^2\right|\F(ih-1)\right]\cr
&&=\E\left.\left[\left\|\sum_{ih\leq n_1< \cdots< n_s\leq (i+1)h-1}\prod_{j=1}^sD(n_j)\right\|^2\right|\F(ih-1)\right]\cr
&&\leq \mathbb{C}_h^s\sum_{ih\leq n_1< \cdots< n_s\leq (i+1)h-1}\E\left.\left[\left\|\prod_{j=1}^sD(n_j)\right\|^2\right|\F(ih-1)\right]\cr
&&\leq \mathbb{C}_h^s\sum_{ih\leq n_1< \cdots< n_s\leq (i+1)h-1}\rho_1^{2s}\prod_{j=1}^s(a(n_j)+b(n_j))^2\cr
&&\leq \mathbb{C}_h^s\sum_{ih\leq n_1< \cdots< n_s\leq (i+1)h-1}\rho_1^{2s}\prod_{j=1}^s(a(i)+b(i))^2\cr
&&= \left(\mathbb{C}_h^s\right)^2\rho_1^{2s}(a(i)+b(i))^{2s}\cr
&&\leq \mathbb{C}_h^s\rho_1^{2s}(a(i)+b(i))^{2s}h!~\text{a.s.},
\ena
where the last inequality is obtained from $\mathbb{C}_h^s\leq h!$. Noting that $\{a(k),k\ge 0\}$ and $\{b(k),k\ge 0\}$ both monotonically vanish, it follows that there exists a constant $c_0>0$, such that $\sup_{k\ge 0}(a(k)+b(k))\leq c_0$. Then, by (\ref{wwffw}), we get
\bna\label{wwkks}
\sum_{s=2}^h\E\left.\left[\|M_s(i)\|^2\right|\F(ih-1)\right]&\leq& 4h!c_0^{-4}\left(a^2(i)+b^2(i)\right)^2\sum_{s=2}^h\mathbb{C}_h^s\rho_1^{2s}c_0^{2s}\cr &=&\rho_2\left(a^2(i)+b^2(i)\right)^2~\text{a.s.},
\ena
where $\rho_2=4h!c_0^{-2}((\rho_1^2c_0^2+1)^h-1-h\rho_1^2c_0^2)$. It follows from (\ref{wwkks}) that
\ban
&&~~~~\E\left[\left\|\left(\sum_{s=2}^hM_s(i)\right)u(i)\right\|^2\right]\cr
&&\leq h\E\left[\sum_{s=2}^h \E\left.\left[\|M_s(i)\|^2\right|\F(ih-1)\right]\|u(i)\|^2\right]\cr
&&\leq h\rho_2\left(a^2(i)+b^2(i)\right)^2\E\left[\|u(i)\|^2\right],
\ean
which together with $\sup_{i\ge 0}\E[\|u(i)\|^2]<\infty$ leads to
\begin{align}\label{wwooo}
&\sup_{i\ge 0}\E\left[\left\|R(i)\right\|^2\right]\cr
\leq & h\rho_2\sup_{i\ge 0}\E\left[\|u(i)\|^2\right]
<    \infty,
\end{align}
where
\ban
R(i)=\frac{1}{a^2(i)+b^2(i)}\left(\sum_{s=2}^hM_s(i)\right)u(i),~i\ge m.
\ean
By the Minkowski inequality, we obtain
\bna\label{foow}
&&\hspace{-0.2cm}\E\left[\left\|\sum_{i=m}^k\left(\prod_{j=i+1}^k\left(I_{\X^N}-\sum_{s=jh}^{(j+1)h-1}D(s)\right)\right)\left(\sum_{s=2}^hM_s(i)\right)u(i)\right\|^2\right]\cr
&&\hspace{-0.8cm}=\E\left[\left\|\sum_{i=m}^k\left(a^2(i)+b^2(i)\right)\left(\prod_{j=i+1}^k\left(I_{\X^N}-\sum_{s=jh}^{(j+1)h-1}D(s)\right)\right)R(i)\right\|^2\right]\cr
&&\hspace{-0.8cm}\leq \Bigg(\sum_{i=m}^k\left(a^2(i)+b^2(i)\right)\Bigg(\E\Bigg[\Bigg\|\Bigg(\prod_{j=i+1}^k\Bigg(I_{\X^N}-\sum_{s=jh}^{(j+1)h-1}D(s)\Bigg)\Bigg) R(i)\Bigg\|^2\Bigg]\Bigg)^{\frac{1}{2}}\Bigg)^2.
\ena
From Proposition \ref{wenknknkn}, we know that $R(i)\in L^0(\Omega,\F((i+1)h-1);\X^N)$, thus, by (\ref{wwooo}) and Lemma \ref{hhhlemma}, we know that there exists a constant $d_3>0$ such that
\bna\label{oxcnv}
&&~~~~\sup_{k\ge 0\atop i\ge 0 }\E\left[\left\|\left(\prod_{j=i+1}^k\left(I_{\X^N}-\sum_{s=jh}^{(j+1)h-1}D(s)\right)\right)R(i)\right\|^2\right]\cr
&&\leq d_3\sup_{i\ge 0}\E\left[\|R(i)\|^2\right]
 <\infty,
\ena
which together with that the operator-valued random sequence $$\left\{I_{\X^N}-\sum_{i=kh}^{(k+1)h-1}(a(i)\H^*(i)\H(i)+b(i)\L_{\G}\otimes I_{\X}),k\ge 0\right\}$$ is $L_2^2$-stable w.r.t. $\{\F((k+1)h-1),k\ge 0\}$ gives
\bna\label{xawq}
\lim_{k\to \infty}\E\left[\left\|\left(\prod_{j=i+1}^k\left(I_{\X^N}-\sum_{s=jh}^{(j+1)h-1}D(s)\right)\right)R(i)\right\|^2\right]=0,~\forall\ i\ge 0.
\ena
By Condition \ref{condition2}, (\ref{oxcnv})-(\ref{xawq}) and Lemma \ref{lemma6}, we have
\ban
\lim_{k\to\infty}\sum_{i=m}^k\left(a^2(i)+b^2(i)\right)\left(\E\left[\left\|\left(\prod_{j=i+1}^k\left(I_{\X^N}-\sum_{s=jh}^{(j+1)h-1}D(s)\right)\right)R(i)\right\|^2\right]\right)^{\frac{1}{2}}=0.
\ean
Given the $L_2$-bounded adaptive sequence $\{x(k),\F(kh-1),k\ge 0\}$ with values in the Hilbert space $\X^N$, from (\ref{ikddw}) and (\ref{foow}), we know that
\bna\label{rscsc}
\lim_{k\to \infty}\E\left[\left\|\Phi_P((k+1)h-1,mh)x(m)\right\|^2\right]=0,~\forall\ m\ge 0.
\ena
For $j\in \mathbb N$, denote $m_j=\lfloor \frac{j}{h} \rfloor,\widetilde{m}_j=\lceil \frac{j}{h} \rceil$. Let $\{y(k),\F(k),k\ge 0\}$ be a $L_2$-bounded adaptive sequence with values in the Hilbert space $\X^N$. For $0\leq i<k-3h$, noting that $0\leq k-m_kh<h$, from Propositions \ref{nlllwwieiie}-\ref{wenknknkn}, it is known that $\Phi_P(m_kh-1,\widetilde{m}_{i+1}h)\Phi_P(\widetilde{m}_{i+1}h-1,i+1)y(i)\in L^0(\Omega,\F(m_kh-1);\X^N)$, by Lemma \ref{lemma1}, we know that there exists a constant $d_2>0$ such that
\bna\label{jjkkl}
&&~~~~\mathbb E\left[\|\Phi_P(k,i+1)y(i)\|^2\right]\cr
&&=\mathbb E\left[\|\Phi_P(k,m_kh)\Phi_P(m_kh-1,\widetilde{m}_{i+1}h)\Phi_P(\widetilde{m}_{i+1}h-1,i+1)y(i)\|^2\right]\cr
&&=\mathbb E\Big[\mathbb E\Big[\Big\|\Phi_P(k,m_kh)\Phi_P(m_kh-1,\widetilde{m}_{i+1}h)\cr
&&~~~~\times \Phi_P(\widetilde{m}_{i+1}h-1,i+1)y(i)\Big\|^2\Big|\mathcal F(m_kh-1)\Big]\Big]\cr
&&\leq d_2\mathbb E\left[\|\Phi_P(m_kh-1,\widetilde{m}_{i+1}h)\Phi_P(\widetilde{m}_{i+1}h-1,i+1)y(i)\|^2\right],~0\leq i<k-3h.
\ena
Noting that $0\leq \widetilde{m}_{i+1}h-(i+1)<h$ and $y(i)\in \F(i)$, it follows from Lemma \ref{lemma1} that
\bna\label{cllll}
&&~~~\sup_{i\ge 0}\mathbb E\left[\|\Phi_P(\widetilde{m}_{i+1}h-1,i+1)y(i)\|^2\right]\cr
&&=\sup_{i\ge 0}\E\left[\mathbb E\left.\left[\|\Phi_P(\widetilde{m}_{i+1}h-1,i+1)y(i)\|^2\right|\mathcal F(i)\right]\right]\cr
&&\leq d_2\sup_{i\ge 0}\E\left[\|y(i)\|^2\right]
 <\infty.
\ena
By Propositions \ref{nlllwwieiie}-\ref{wenknknkn}, we get $\Phi_P(\widetilde{m}_{i+1}h-1,i+1)y(i)\in L^0(\Omega,\mathcal F(\widetilde{m}_{i+1}h-1);\X^N)$. Substituting (\ref{rscsc}) and (\ref{cllll}) into (\ref{jjkkl}) gives
\ban
\lim_{k\to\infty}\mathbb E\left[\|\Phi_P(k,i+1)y(i)\|^2\right]=0,~\forall\ i\ge 0,
\ean
which implies that the operator-valued random sequence $\{I_{\X^N}-a(k)\H^*(k)\H(k)-b(k)\L_{\G}\otimes I_{\X},k\ge 0\}$ is $L_2^2$-stable w.r.t. $\{\F(k),k\ge 0\}$.
\end{proof}

\begin{proof}[Proof of Lemma \ref{jihubiranshoulian}]
It follows from Condition \ref{condition3} that there exists a constant $C_1>0$, such that $|b(k)-a(k)|\leq C_1(a^2(k)+b^2(k))$, which gives
\bna\label{xmkmslff}
\sum_{i=0}^ka(i)&\leq& \sum_{i=0}^k(|a(i)-b(i)|+b(i))\cr
&\leq& C_1\sum_{i=0}^k\left(a^2(i)+b^2(i)\right)+\sum_{i=0}^kb(i),~k\ge 0.
\ena
For the integer $h>0$, denote
\bna\label{xmlwf}
c(k)=\sum_{s=kh}^{(k+1)h-1}b(s).
\ena
By (\ref{xmkmslff}) and Condition \ref{condition2}, we get
\ban
\sum_{k=0}^{\infty}c(k)=\sum_{k=0}^{\infty}\sum_{s=kh}^{(k+1)h-1}b(s)=\sum_{k=0}^{\infty}b(k)=\infty.
\ean
Denote
\bna\label{mvnjwe}
\HH=\text{diag}\left\{\frac{1}{h}\HH_1,\cdots,\frac{1}{h}\HH_N\right\}+\L_{\G}\otimes I_{\X}.
\ena
Noting that $\L_{\G}$ is positive semi-definite and $\sum_{i=1}^{N}\HH_i>0$, by Lemma \ref{lemmaA10}, we know that $\HH\in \mathscr L(\X^N)$ is a strictly positive self-adjoint operator. Let $\{x(k),\F(kh-1),k\ge 0\}$ be a $L_2$-bounded adaptive sequence with values in the Hilbert space $\X^N$, then we can write  $x(k)=(x_1(k),\cdots,x_N(k))$, where $x_i(k):\Omega\to \X$, $i=1,\cdots,N$ are the random elements with values in the Hilbert space $(\X,\tau_{\text{N}}(\X))$. Denote
\bna\label{cnlweee}
\mu(i) = c(i)\HH x(i)-\sum_{s=ih}^{(i+1)h-1}(a(s)\E[\H^*(s)\H(s)x(i)|\F(ih-1)] +b(s)(\L_{\G}\otimes I_{\X})x(i)).
\ena
By (\ref{xmlwf})-(\ref{cnlweee}), we get
\ban
&&~~~~\mu(i)\cr
&&=\sum_{s=ih}^{(i+1)h-1}b(s)\text{diag}\left\{\frac{1}{h}\HH_1,\cdots,\frac{1}{h}\HH_N\right\}x(i)-\sum_{s=ih}^{(i+1)h-1}a(s)\E[\H^*(s)\H(s)x(i)|\F(ih-1)]\cr
&&=\text{diag}\left\{\frac{1}{h}\sum_{s=ih}^{(i+1)h-1}b(s)\HH_1x_1(i)-\sum_{s=ih}^{(i+1)h-1}a(s)\E[H_1^*(s)H_1(s)x_1(i)|\F(ih-1)],\right.\cr
&&~~~~~~\left.\cdots,\frac{1}{h}\sum_{s=ih}^{(i+1)h-1}b(s)\HH_Nx_N(i)-\sum_{s=ih}^{(i+1)h-1}a(s)\E[H_N^*(s)H_N(s)x_N(i)|\F(ih-1)]\right\}.
\ean
It follows from Lemma \ref{lemmaA11} that there exists a constant $C_2>0$ such that
\bna\label{vklwmlmfm}
\max_{ih\leq s\leq (i+1)h-1}\Bigg(\frac{1}{h}\Bigg(\sum_{s=ih}^{(i+1)h-1}b(s)\Bigg)-a(s)\Bigg)^2\leq C_2\left(a^4(i)+b^4(i)\right).
\ena
Therefore, by Conditions \ref{condition1}-\ref{condition3}, the condition (\ref{yinlitiaojian2}) and (\ref{vklwmlmfm}), we have
\begin{align}\label{yuzhouwudichang}
 &\E\left[\|\mu(i)\|^2\right]\cr
 =& \sum_{j=1}^N\E\Bigg[\Bigg\|\frac{1}{h}\sum_{s=ih}^{(i+1)h-1}b(s)\HH_jx_j(i)
 -\sum_{s=ih}^{(i+1)h-1}a(s)\E[H_j^*(s)H_j(s)x_j(i)|\F(ih-1)]\Bigg\|^2\Bigg]\cr
 =&\sum_{j=1}^N\E\Bigg[\Bigg\|\Bigg(\frac{1}{h}\sum_{s=ih}^{(i+1)h-1}b(s)\Bigg)
 \Bigg(\HH_j x_j(i)-\sum_{s=ih}^{(i+1)h-1}\E\left[H_j^*(s)H_j(s)x_j(i)|\F(ih-1)\right]\Bigg)\cr & +\sum_{s=ih}^{(i+1)h-1}\Bigg(\frac{1}{h}\Bigg(\sum_{s=ih}^{(i+1)h-1}b(s)\Bigg)-a(s)\Bigg)\E\left[H_j^*(s)H_j(s)x_j(i)|\F(ih-1)\right]\Bigg\|^2\Bigg]\cr
 \leq & \sum_{j=1}^N\Bigg(2b^2(i)\E\Bigg[\Bigg\|\HH_jx_j(i)-\sum_{s=ih}^{(i+1)h-1}\E\left.\left[H_j^*(s)H_j(s)x_j(i)\right|\F(ih-1)\right]\Bigg\|^2\Bigg]\cr
& +2h\E\Bigg[\sum_{s=ih}^{(i+1)h-1}
\Bigg(\frac{1}{h}\Bigg(\sum_{s=ih}^{(i+1)h-1}b(s)\Bigg)-a(s)\Bigg)^2 \left\|\E\left[H_j^*(s)H_j(s)x_j(i)|\F(ih-1)\right]\right\|^2\Bigg]\Bigg)\cr
 \leq & \sum_{j=1}^N\left(2b^2(i)\E\left[\left\|\HH_jx_j(i)-\sum_{s=ih}^{(i+1)h-1}\E\left.\left[H_j^*(s)H_j(s)x_j(i)\right|\F(ih-1)\right]\right\|^2\right]\right.\cr
& +2h\E\Bigg[\sum_{s=ih}^{(i+1)h-1}\Bigg(\frac{1}{h}\Bigg(\sum_{s=ih}^{(i+1)h-1}b(s)\Bigg)-a(s)\Bigg)^2\cr & \times\E\left[\left\|H_j^*(s)H_j(s)\right\|^2|\F(ih-1)\right]\left\|x_j(i)\right\|^2\Bigg]\Bigg)\cr
 \leq & \sum_{j=1}^N\left(2b^2(i)\E\left[\left\|\HH_jx_j(i)-\sum_{s=ih}^{(i+1)h-1}\E\left.\left[H_j^*(s)H_j(s)x_j(i)\right|\F(ih-1)\right]\right\|^2\right]\right.\cr
& +\left.2h\rho^2_0\sup_{k\ge 0}\E\left[\|x(k)\|^2\right]\sum_{s=ih}^{(i+1)h-1}\left(\frac{1}{h}\left(\sum_{s=ih}^{(i+1)h-1}b(s)\right)-a(s)\right)^2\right)\cr
 \leq & \sum_{j=1}^N\left(2b^2(i)\E\left[\left\|\HH_jx_j(i)-\sum_{s=ih}^{(i+1)h-1}\E\left.\left[H_j^*(s)H_j(s)x_j(i)\right|\F(ih-1)\right]\right\|^2\right]\right.\cr
& +\left.2h^2\rho^2_0\sup_{k\ge 0}\E\left[\|x(k)\|^2\right]\max_{ih\leq s\leq (i+1)h-1}\left(\frac{1}{h}\left(\sum_{s=ih}^{(i+1)h-1}b(s)\right)-a(s)\right)^2\right)\cr
 \leq & \sum_{j=1}^N2b^2(i)\E\left[\left\|\HH_jx_j(i)-\sum_{s=ih}^{(i+1)h-1}\E\left.\left[H_j^*(s)H_j(s)x_j(i)\right|\F(ih-1)\right]\right\|^2\right]\cr
& +2C_2Nh^2\rho^2_0\sup_{k\ge 0}\E\left[\|x(k)\|^2\right]\left(a^4(i)+b^4(i)\right).
\end{align}
Noting that $\{x_j(k),\F(kh-1),k\ge 0\},j=1,\cdots,N$ are the $L_2$-bounded adaptive sequences with values in the Hilbert space $\X$, by (\ref{yinlitiaojian1}), we obtain
\bna\label{vnkwjofjeofe}
&&\sum_{i=0}^{\infty}\E\left[\left\|\HH_jx_j(i)-\sum_{s=ih}^{(i+1)h-1}
\E\left.\left[H_j^*(s)H_j(s)x_j(i)\right|\F(ih-1)\right]\right\|^2\right]\cr
&<&\infty,~j=1,\cdots,N.
\ena
By Condition \ref{condition2}, (\ref{yuzhouwudichang})-(\ref{vnkwjofjeofe}) and Cauchy inequality, we get
\ban
&&~~~~\sum_{i=0}^{\infty}\E\left[\|\mu(i)\|^2\right]^{\frac{1}{2}}\cr
&&\leq \sqrt{2}\sum_{i=0}^{\infty}b(i)\sum_{j=1}^N\left(\E\left[\left\|\HH_jx_j(i)-\sum_{s=ih}^{(i+1)h-1}\E\left.\left[H_j^*(s)H_j(s)x_j(i)\right|\F(ih-1)\right]\right\|^2\right]\right)^{\frac{1}{2}}\cr
&&~~~~+h\rho_0\sup_{k\ge 0}\E\left[\|x(k)\|^2\right]^{\frac{1}{2}}\sqrt{2C_2N}\sum_{i=0}^{\infty}\left(a^2(i)+b^2(i)\right)\cr
&&\leq \sqrt{2}NC_3\left(\sum_{i=0}^{\infty}\sum_{j=1}^N\E\left[\left\|\HH_jx_j(i)-\sum_{s=ih}^{(i+1)h-1}\E\left.\left[H_j^*(s)H_j(s)x_j(i)\right|\F(ih-1)\right]\right\|^2\right]\right)\cr
&&~~~~+h\rho_0\sup_{k\ge 0}\E\left[\|x(k)\|^2\right]^{\frac{1}{2}}\sqrt{2C_2N}\sum_{i=0}^{\infty}\left(a^2(i)+b^2(i)\right)\cr
&&= \sqrt{2}NC_3\left(\sum_{j=1}^N\sum_{i=0}^{\infty}\E\left[\left\|\HH_jx_j(i)-\sum_{s=ih}^{(i+1)h-1}\E\left.\left[H_j^*(s)H_j(s)x_j(i)\right|\F(ih-1)\right]\right\|^2\right]\right)\cr
&&~~~~+h\rho_0\sup_{k\ge 0}\E\left[\|x(k)\|^2\right]^{\frac{1}{2}}\sqrt{2C_2N}\sum_{i=0}^{\infty}
\left(a^2(i)+b^2(i)\right)
  <\infty,
\ean
where $C_3=\sum_{i=0}^{\infty}b^2(i)$. For any given integer $m>0$, denote $\Gamma_{m}=\{i\ge m:\E[\|\mu(i)\|^2]  >0,i\in \mathbb N\}$. Noting that $\E[\|\mu(i)\|^2]=0$ implies that $\mu(i)=0~\text{a.s.}$, we obtain
\bna\label{final2}
&&~~~~\sum_{i=m}^{k}\left(\E\left[\left\|\left(\prod_{j=i+1}^k(I_{\X^N}-c(j)\HH)\right)\mu(i)\right\|^2\right]\right)^{\frac{1}{2}}\cr
&&\leq \sum_{i=m}^{\infty}\left(\E\left[\left\|\left(\prod_{j=i+1}^k(I_{\X^N}-c(j)\HH)\right)\mu(i)\right\|^2\right]\right)^{\frac{1}{2}}\cr
&&=\sum_{i\in\Gamma_m}\left(\E\left[\left\|\left(\prod_{j=i+1}^k(I_{\X^N}-c(j)\HH)\right)\mu(i)\right\|^2\right]\right)^{\frac{1}{2}}\cr
&&=\sum_{i\in\Gamma_m}\E\left[\|\mu(i)\|^2\right]^{\frac{1}{2}}\left(\E\left[\left\|\left(\prod_{j=i+1}^k(I_{\X^N}-c(j)\HH)\right)\eta(i)\right\|^2\right]\right)^{\frac{1}{2}},
\ena
where $\eta(i)=\mu(i)\E[\|\mu(i)\|^2]^{-\frac{1}{2}},i\in \Gamma_m$. Noting that $\E[\|\eta(i)\|^2]=1$, it follows from Lemma \ref{lemmaA7} that there exist constants $M,d>0$, such that
\bna\label{final3}
&&~~~~\sup_{k\ge 0\atop i\in \Gamma_m}\left(\E\left[\left\|\left(\prod_{j=i+1}^k(I_{\X^N}-c(j)\HH)\right)\eta(i)\right\|^2\right]\right)^{\frac{1}{2}}\cr
&&\leq M^d\sup_{i\in \Gamma_m}\E\left[\|\eta(i)\|^2\right]^{\frac{1}{2}}
   =M^d.
\ena
By Lemma \ref{lemmaA7} and Lebesgue dominated convergence theorem, we get
\bna\label{final4}
\lim_{k\to \infty}\left(\E\left[\left\|\left(\prod_{j=i+1}^k(I_{\X^N}-c(j)\HH)\right)\eta(i)\right\|^2\right]\right)^{\frac{1}{2}}=0,~\forall\ i\ge 0.
\ena
Therefore, combining (\ref{final2})-(\ref{final4}) and Lemma \ref{lemma6} leads to
\ban
\lim_{k\to\infty}\sum_{i=m}^{k}\left(\E\left[\left\|\left(\prod_{j=i+1}^k(I_{\X^N}-c(j)\HH)\right)\mu(i)\right\|^2\right]\right)^{\frac{1}{2}}=0,~\forall\ m\ge 0.
\ean
It follows from Lemma \ref{henandelemma} that the operator-valued random sequence $$ \left\{I_{\X^N}-\sum_{i=kh}^{(k+1)h-1}(a(i)\H^*(i)\H(i)+b(i)\L_{\G}\otimes I_{\X}),k\ge 0\right\}$$ is $L_2^2$-stable w.r.t. $\{\F((k+1)h-1),k\ge 0\}$. Hence, from Lemma \ref{yibanxingdejieguo}, it is known that the operator-valued random sequence $\{I_{\X^N}-a(k)\H^*(k)\H(k)-b(k)\L_{\G}\otimes I_{\X},k\ge 0\}$ is $L_2^2$-stable w.r.t. $\{\F(k),k\ge 0\}$.
\end{proof}
\emph{Proof of Lemma \ref{rabdomelerkhs}.} For any given $f_1,f_2\in \HH_K$ and $c_1,c_2\in \mathbb R$, it follows from the reproducing property of $\HH_K$ that
\begin{align}\label{vlwlmmfff}
&H_i(k)(c_1f_1+c_2f_2)\notag\\
=&\langle c_1f_1+c_2f_2,K_{x_i(k)}\rangle _K\cr
=&c_1H_i(k)(f_1)+c_2H_i(k)(f_2),~k\ge 0,~i\in \mathcal V,
\end{align}
thus, $H_i(k)$ is a linear operator.
For any given sample path $\omega\in\Omega$, we get
\begin{align*}
 &|H_i(k)(\omega)(f)|\\
= & \left|\left\langle f,K_{x_i(k)(\omega)}\right\rangle _K\right|\\ \leq &\|f\|_K \|K_{x_i(k)(\omega)}\|_K\\ = & \|f\|_K  \sqrt{\langle K_{x_i(k)(\omega)},K_{x_i(k)(\omega)}\rangle_{K}}\\
=& \|f\|_K  \sqrt{ K(x_i(k)(\omega),x_i(k)(\omega))}\\
 \leq  &  \sup\limits_{x\in\X} \sqrt{K(x,x)}  \|f\|_K,   \ \forall\ f\in \HH_K,~k\ge 0,~i\in \mathcal V,
\end{align*}
then $\|H_i(k)\|_{\LL(\HH_K,\mathbb R)}\leq \sup\limits_{x\in\X}\sqrt{K(x,x)}\   \text{a.s.},$ which together with Assumption \ref{assumption5} gives $H_i(k) \in \LL(\HH_K,\mathbb R)$.
Noting that $x_{i}(k)$ is a random vector with values in the Hilbert space $(\X,\tau_{\text{N}}(\X))$, then there exists a simple function sequence $\{x_{i}^n(k),n\ge 0\}$ with values in $\X$, such that $\lim_{n\to\infty}\|x_{i}(k)-x_{i}^n(k)\|=0~\text{a.s.}, ~k\ge 0,~i\in \mathcal V.$ Denote $H_i^n(k)(f)=f(x_i^n(k)),\  f\in \HH_K, ~k\ge 0,~i\in \mathcal V.$ Then, we know that $\{H_i^n(k),n\ge 0\}$  is a sequence of simple functions with values in $\LL(\HH_K,\mathbb R).$ This together with the reproducing property of $\HH_K$ and the symmetry of Mercer kernel $K$ gives
\begin{align}
&\|H_i^n(k)-H_i(k)\|_{\LL(\HH_K,\mathbb R)}\notag\\
=&\sup_{\|f\|_{K}=1} |H_i^n(k)(f)-H_i(k)(f) |\notag\\
=& \sup_{\|f\|_{K}=1} \left|\left\langle f,K_{x_i^n(k)}-K_{x_i(k)}\right\rangle _K \right|\notag\\
\leq & \|K_{x_i^n(k)}-K_{x_i(k)}\|_{K}\notag\\
=& \sqrt{\langle K_{x_i^n(k)}-K_{x_i(k)},K_{x_i^n(k)}-K_{x_i(k)}\rangle_{K}}\notag\\
=&\sqrt{ K(x_i(k),x_i(k))-2K(x_i(k),x_i^n(k))+K(x_i^n(k),x_i^n(k))},\label{klianxuxing}
\end{align}
which together with the continuity of Mercer kernel $K$ gives that $\lim\limits_{n\to\infty}\|H_i^n(k)-H_i(k)\|_{\LL(\HH_K,\mathbb R)}=0,~k\ge 0,~i\in \mathcal V$. Combining this, Lemma \ref{vnwkelel} and Definition \ref{tuopukongjian} gives that $H_i(k)$ is a random element with values in the  space $(\LL(\HH_K,\mathbb R),\tau_{\text{N}}(\LL(\HH_K,\mathbb R)))$. Noting that $\tau_{\text{S}}(\LL(\HH_K,\mathbb R))\subseteq \tau_{\text{N}}(\LL(\HH_K,\mathbb R))$,
 it follows that a random element with values in $(\LL(\HH_K,\mathbb R),\tau_{\text{N}}(\LL(\HH_K,\mathbb R)))$ is a random element with values in
$(\LL(\HH_K,\mathbb R),\tau_{\text{S}}(\LL(\HH_K,\mathbb R)))$.  Therefore, $H_i(k)$ is a random element with values in   $(\LL(\HH_K,\mathbb R),\tau_{\text{S}}(\LL(\HH_K,\mathbb R))),~k\ge 0,~i\in \mathcal V$.
$\hfill\qed$

\section{Key Lemmas}\label{appendixee}
 \setcounter{equation}{0}
\renewcommand{\theequation}{D.\arabic{equation}}
\begin{lemma}[\cite{rb}]\label{lemmaA3}
Let $\{x(k), \mathcal F(k)\}$, $\{\a(k),\mathcal F(k)\}$, $\{\b(k),\mathcal F(k)\}$ and $\{\gamma(k), \mathcal F(k)\}$ be nonnegative  adaptive sequences satisfying
$$
\E[x(k+1)|\mathcal F(k)]\le(1+\a(k))x(k)-\beta(k)+\gamma(k),~k\ge 0~\text{a.s.},
$$
and $\sum_{k=0}^\infty(\a(k)+\gamma(k))<\infty~\text{a.s.}$ Then $x(k)$ converges to a finite random variable a.s., and $\sum_{k=0}^\infty\b(k)<\infty$ a.s.
\end{lemma}

\begin{lemma}\label{lemma6}
Let $\{a(i),i\in \Lambda\}$ be a nonnegative real sequence, where $\Lambda \subseteq \mathbb N$ and $\sum_{i\in \Lambda}a(i)<\infty$, $\{b(k,i),k\in \mathbb N,i\in \Lambda\}$ be a double index real sequence. If $\lim_{k\to \infty}b(k,i)=0$, $\forall\ i\in  \Lambda$, and there exists a constant $c>0$ such that $|b(k,i)|\leq c$, $\forall\ k\in \mathbb N$, $\forall\ i\in \Lambda$, then
\bna\label{mmxxxx}
\lim_{k\to \infty}\sum_{i\in \Lambda}a(i)b(k,i)=0.
\ena
\end{lemma}

\begin{proof}
For any given $\varepsilon>0$, it follows from $\sum_{i\in \Lambda}a(i)<\infty$ that there exists an integer $N>0$, such that $\sum_{i\in \Lambda_N}a(i)<\varepsilon$, where $\Lambda_N=\{i\in \Lambda:i\ge N+1\}$. We obtain
$
\Bigg|\sum_{i\in \Lambda_N}a(i)b(k,i)\Bigg|\leq c\varepsilon,~\forall\ k\ge 0.
$
On the other hand, noting that $\lim_{k\to \infty}b(k,i)=0$, $\forall i\in \Lambda$, it follows that there exists a constant $M_i>0$, such that $|b(k,i)|\leq \varepsilon $, $\forall k\ge M_i$. Denote $M=\max\{M_i:i\in \{1,\cdots,N\}\cap \Lambda\}$. It follows from $|b(k,i)|\leq \varepsilon,~\forall\ k\ge M$ that
\ban
\Bigg|\sum_{i\in \Lambda}a(i)b(k,i)\Bigg|\leq \Bigg|\sum_{i\in \{1,\cdots,N\}\cap \Lambda}a(i)b(k,i)\Bigg|+\Bigg|\sum_{i\in \Lambda_N}a(i)b(k,i)\Bigg|\leq \varepsilon \sum_{i\in \Lambda}a(i)+c\varepsilon,~k\ge M,
\ean
which gives (\ref{mmxxxx}).
\end{proof}

\begin{lemma}\label{lemmaA7}
Let $\X$ be a Hilbert space, $H\in \mathscr L(\X)$ be a strictly positive self-adjoint operator, and $\{\mu(k),k\ge 0\}$ be a real sequence monotonically decreasing to $0$ with $\sum_{k=0}^{\infty}\mu(k)=\infty$. Then there exist positive constants $M$ and $d$, such that
\bna
\label{lemmaaddaddadd1}
\sup_{s\ge 0}\left\|\left(\prod_{j=t}^s(I_{\X}-\mu(j)H)\right)x\right\|\leq M^d\|x\|,~\forall\ x\in \X,~\forall\ t\ge 0,
\ena
and
\bna
\label{lemmaaddaddadd2}
\lim_{k\to \infty}\left(\prod_{j=0}^k(I_{\X}-\mu(j)H)\right)x=0,~\forall\ x \in \X.
\ena
\end{lemma}

\begin{proof}
Noting that $H\in \mathscr L(\X)$ is the bounded self-adjoint operator, it is shown that $H$ has the following spectral decomposition
$$
H=\int_{-\infty}^{+\infty}\lambda \dd p_{\lambda}=\int_{\sigma (H)}\lambda \dd p_{\lambda}.
$$
From the property of the self-adjoint operator, we have
\bna\label{nvmsacc}
\prod_{j=0}^{k
}(I_{\X}-\mu(j)H)^2=\int_{\sigma(H)}\prod_{j=0}^{k}(1-\mu(j)\lambda)^2\dd p_{\lambda},
\ena
which together with (\ref{nvmsacc}) gives
\begin{align}\label{mmcc}
\Bigg\|\prod_{j=0}^{k
}(I_{\X}-\mu(j)H)x\Bigg\|^2 =& \Bigg\langle\prod_{j=0}^{k
}(I_{\X}-\mu(j)H)^2x,x\Bigg\rangle \notag\\
 =& \int_{\sigma(H)}\prod_{j=0}^{k}(1-\mu(j) \lambda)^2\dd\langle p_{\lambda}x,x\rangle,
 ~\forall\ x\in \X.
\end{align}
It follows from $H>0$ that $\sigma(H)\subset [0,\|H\|]$, which leads to
\ban
(1-\mu(j)\lambda)^2\leq \max\left\{(1-\mu(j)\|H\|)^2,1\right\},~\forall\ \lambda \in \sigma(H).
\ean
Noting that $\mu(k)\to 0,k\to \infty$, it follows that there exists an integer $d>0$ such that $\mu(j)\leq \|H\|^{-1}$, $\forall\ j>2d$, which further shows that $(1-b(j)\lambda)^2\leq 1,\forall\ j> 2d$. Denote $M=\max\{(1-\mu(j)\|H\|)^2,1:0\leq j\leq 2d\}$, we have
$$
\prod_{j=t}^s(1-\mu(j)\lambda)^2\leq M^{2d},~\forall\ \lambda \in \sigma(H),~\forall\ t\ge 0,
$$
and
\ban
\left\|\prod_{j=t}^{s
}(I_{\X}-\mu(j)H)x\right\|^2\leq M^{2d}\int_{\sigma(H)}\dd\langle p_{\lambda}x,x\rangle =M^{2d}\|x\|^2,~\forall\ x\in \X,~\forall\ t\ge 0.
\ean
This gives (\ref{lemmaaddaddadd1}).
Noting the inequality $1-a\leq \text{e}^{-a}$, $\forall\ a\ge 0$, we obtain
$
\prod_{j=0}^{k}(1-\mu(j)\lambda)^2\leq M^{2d}\text{e}^{-2\lambda\sum_{j=2d+1}^{k}\mu(j)}.
$
It follows from $\sum_{j=0}^{\infty}\mu(k)=\infty$ that
\bna\label{ttkktt}
\lim_{k\to \infty}\prod_{j=0}^{k}(1-\mu(j)\lambda)^2=0,~\forall\ \lambda\in \sigma(H)\cap \mathbb R^+.
\ena
Since $H$ is the strictly positive self-adjoint operator, it follows that $0$ is not in the point spectrum of the operator $H$, which gives
\ban
\int_{\sigma(H)}\prod_{j=0}^{k}(1-\mu(j)\lambda)^2\dd\langle p_{\lambda}x,x\rangle =\int_{\sigma(H)\cap \mathbb R^+}\prod_{j=0}^{k}(1-\mu(j)\lambda)^2\dd\langle p_{\lambda}x,x\rangle,~\forall\ x\in \X.
\ean
By (\ref{mmcc}), (\ref{lemmaaddaddadd1}), (\ref{ttkktt}) and the dominated convergence theorem, we get
\ban
\lim_{k\to \infty}\left\|\prod_{j=0}^{k
}(I_{\X}-\mu(j)H)x\right\|^2=\int_{\sigma(H)\cap \mathbb R^+}\lim_{k\to\infty}\prod_{j=0}^{k}(1-\mu(j)\lambda)^2\dd\langle p_{\lambda}x,x\rangle=0,~\forall\ x\in \X,
\ean
which gives (\ref{lemmaaddaddadd2}).
\end{proof}

\begin{lemma}\label{lemmaA8}
Let $\X$ be a Hilbert space, $H\in \mathscr L(\X)$ be a strictly positive self-adjoint operator, and $\{x(k),k\ge 0\}$ be a $L_2$-bounded random sequence with values in Hilbert space. If $\{\mu(k),k\ge 0\}$ and $\{\gamma(k),k\ge 0\}$ are both real sequences monotonically decreasing to $0$ with $\sum_{k=0}^{\infty}\mu(k)=\infty$ and $\sum_{k=0}^{\infty}\gamma(k)<\infty$, then
$$
\lim_{k\to\infty}\sum_{i=0}^k\gamma(i)\left(\E\left[\left\|\left(\prod_{j=i+1}^k(I_{\X}-\mu(j)H)\right)x(i)\right\|^2\right]\right)^{\frac{1}{2}}=0.
$$
\end{lemma}

\begin{proof}
Denote
\ban
b(k,i)=\left(\E\left[\left\|\prod_{j=i+1}^k(I_{\X}-\mu(j)H)x(i)\right\|^2\right]\right)^{\frac{1}{2}},~\forall\ k,i\ge 0,
\ean
it follows from Lemma \ref{lemmaA7} that
\bna\label{rrkkrr}
0\leq b(k,i)\leq M^d\sup_{i\ge 0}\left(\E\left[\|x(i)\|^2\right]\right)^{\frac{1}{2}}<\infty,~\forall\ k,i\ge 0,
\ena
and
\ban
\left\|\prod_{j=i+1}^k(I_{\X}-\mu(j)H)x(i)\right\|\leq M^d\|x(i)\|~\text{a.s.},~\forall\ i\ge 0.
\ean
Thus, by the Lebesgue dominated convergence theorem and Lemma \ref{lemmaA7}, we get
\bna\label{rrkkll}
\lim_{k\to\infty}\E\left[\left\|\prod_{j=i+1}^k(I_{\X}-\mu(j)H)x(i)\right\|^2\right]=0,~\forall\ i\ge 0,
\ena
which further gives $\lim_{k\to\infty}b(k,i)=0$, $\forall\ i\ge 0$. Noting that $\sum_{k=0}^{\infty}\gamma(k)<\infty$, by (\ref{rrkkrr})-(\ref{rrkkll}) and Lemma \ref{lemma6}, we obtain
$
\lim_{k\to\infty}\sum_{i=0}^k\gamma(i)\left(\E\left[\left\|\left(\prod_{j=i+1}^k(I_{\X}-\mu(j)H)\right)x(i)\right\|^2\right]\right)^{\frac{1}{2}}=0.
$
\end{proof}

\begin{lemma}\label{lemmaA10}
Let $\mathcal A=[a_{ij}]\in \mathbb R^{N\times N}$ be the adjacency matrix of an undirected connected graph $\G$, $\L_{\G}$ be the Laplacian matrix, $H_i\in \mathscr L(\X)$, $i=1,\cdots,N$ are positive self-adjoint operators. If
\bna\label{lcms}
\sum_{i=1}^NH_i>0,
\ena
then $\text{diag}\{H_1,\cdots,H_N\}+\L_{\G}\otimes I_{\X}>0$.
\end{lemma}

\begin{proof}
Given the non-zero element $x=(x_1,\cdots,x_N)$ with values in Hilbert space $\X^N$, where $x_i\in \X$, $i=1,\cdots,N$. Here, we will prove $\langle (\text{diag}\{H_1,\cdots,H_N\}+\L_{\G}\otimes I_{\X})x,x\rangle _{\X^N}>0$ in two steps.
\begin{itemize}
\item If there exists a non-zero element $a\in \X$ with $x_1=\cdots=x_N=a$, then $x=\textbf{1}_N\otimes a$. Noting that $\L_{\G}\textbf{1}_N=0$, it follows from (\ref{lcms}) that
\bna\label{ofuncc}
&&~~~\left\langle \left(\text{diag}\{H_1,\cdots,H_N\}+\L_{\G}\otimes I_{\X}\right)x,x\right\rangle _{\X^N}\cr
&&=\left\langle \left(\text{diag}\{H_1,\cdots,H_N\}+\L_{\G}\otimes I_{\X}\right)\left(\textbf{1}_N\otimes a\right),\left(\textbf{1}_N\otimes a\right)\right\rangle _{\X^N}\cr
&&=\langle \text{diag}\{H_1,\cdots,H_N\}\left(\textbf{1}_N\otimes a\right),\left(\textbf{1}_N\otimes a\right)\rangle _{\X^N}\cr
&&+\langle (\L_{\G}\otimes I_{\X})\left(\textbf{1}_N\otimes a\right),\left(\textbf{1}_N\otimes a\right)\rangle _{\X^N}\cr
&&=\left\langle \left(\sum_{i=1}^NH_i\right)a,a\right\rangle _{\X}+\langle (\L_{\G}\textbf{1}_N)\otimes a,\left(\textbf{1}_N\otimes a\right) \rangle _{\X^N}\cr
&&=\left\langle \left(\sum_{i=1}^NH_i\right)a,a\right\rangle _{\X}
 >0.
\ena
\item If there exist $1\leq i_0\neq j_0\leq N$, such that $x_{i_0}\neq x_{j_0}$. It follows from $H_i\ge 0$, $i=1,\cdots,N$ that $\text{diag}\{H_1,\cdots,H_N\}\ge 0$. Noting that the graph is undirected and connected, then there exist nodes $\{ i_{1},\ldots, i_{m}\}$ such that
      $a_{i_{0}i_{1}}>0,  a_{i_{1}i_{2}}>0, \ldots, a_{i_{m}j_{0}}>0$.
   Notice that
    \bna\label{ofunc1}
&&~~~\left\langle \left(\text{diag}\{H_1,\cdots,H_N\}+\L_{\G}\otimes I_{\X}\right)x,x\right\rangle _{\X^N}\cr
&&\ge \langle (\L_{\G}\otimes I_{\X})x,x\rangle _{\X^N}\cr
&&=\frac{1}{2}\sum_{i=1}^N\sum_{j=1}^Na_{ij}\|x_i-x_j\|^2_{\X}\cr
 && \ge \frac{1}{2}\sum_{j=0}^{m-1} a_{i_{j}i_{j+1}}\|x_{i_{j}}-x_{i_{j+1}}\|_{\X}^2 + \frac{1}{2}a_{i_{m}j_{0}}\|x_{i_{m}}-x_{j_{0}}\|_{\X}^2\geq 0.
\ena
 If  $\frac{1}{2}\sum\limits_{j=0}^{m-1} a_{i_{j}i_{j+1}}\|x_{i_{j}}-x_{i_{j+1}}\|_{\X}^2 + \frac{1}{2}a_{i_{m}j_{0}}\|x_{i_{m}}-x_{j_{0}}\|_{\X}^2=0$, by  $a_{i_{0}i_{1}}>0,  a_{i_{1}i_{2}}>0, \ldots, a_{i_{m}j_{0}}>0$, then we get $x_{i_0}=x_{i_1}=\ldots=x_{i_{m}}=x_{j_{0}}$, which leads to a contradiction with  $x_{i_0}\neq x_{j_0}$. Thus, we have $\frac{1}{2}\sum\limits_{j=0}^{m-1} a_{i_{j}i_{j+1}}\|x_{i_{j}}-x_{i_{j+1}}\|_{\X}^2 + \frac{1}{2}a_{i_{m}j_{0}}\|x_{i_{m}}-x_{j_{0}}\|_{\X}^2>0$. Then, by (\ref{ofunc1}), we have
 \begin{align}\label{ofunc}
 \langle \left(\text{diag}\{H_1,\cdots,H_N\}+\L_{\G}\otimes I_{\X}\right)x,x \rangle _{\X^N}>0.
 \end{align}
\end{itemize}
Combining (\ref{ofuncc}) and (\ref{ofunc}) gives that  $\text{diag}\{H_1,\cdots,H_N\}+\L_{\G}\otimes I_{\X}\in \mathscr L(\X^N)$ is strictly positive.
\end{proof}

\begin{lemma}\label{lemmaA11}
Let $\{a(k),k\ge 0\}$ and $\{b(k),k\ge 0\}$ be monotonically decreasing sequences of positive real numbers. If
\bna\label{tiaojian}
\max\left\{a(k)-a(k+1),b(k)-a(k)\right\}=\mathcal O\left(a^2(k)+b^2(k)\right),
\ena
then
\bna\label{fulu11}
\max_{ih\leq s\leq (i+1)h-1}\left(\frac{1}{h}\left(\sum_{s=ih}^{(i+1)h-1}b(s)\right)-a(s)\right)^2=\mathcal O \left(a^4(i)+b^4(i)\right),~\forall\ h=1,2,...
\ena
\end{lemma}

\begin{proof}
Let $h$ be any given positive integer. Noting that $\{a(k),k\ge 0\}$ and $\{b(k),k\ge 0\}$ are both monotonically decreasing sequences, it follows that
$$\frac{1}{h}\left(\sum_{s=ih}^{(i+1)h-1}b(s)\right)\in [b((i+1)h-1),b(ih)],$$
and $$a(s)\in [a((i+1)h-1),a(ih)],~ih\leq s\leq (i+1)h-1,$$
from which we obtain
\bna\label{fulu0}
&&~~~\max_{ih\leq s\leq (i+1)h-1}\left(\frac{1}{h}\left(\sum_{s=ih}^{(i+1)h-1}b(s)\right)-a(s)\right)^2\cr
&&\leq \max_{ih\leq s\leq (i+1)h-1}\left\{(b((i+1)h-1)-a(s))^2,(b(ih)-a(s))^2\right\}\cr
&&\leq \max\big\{(b((i+1)h-1)-a(ih))^2,(b((i+1)h-1)-a((i+1)h-1))^2,\cr
&&~~~(b(ih)-a(ih))^2,(b(ih)-a((i+1)h-1))^2\big\}.
\ena
It follows from (\ref{tiaojian}) that there exists a constant $C>0$, such that
\bna\label{fnlwlmv}
\max\left\{(a(k)-a(k+1))^2,(b(k)-a(k))^2\right\}\leq C\left(a^4(k)+b^4(k)\right),
\ena
which gives
\bna\label{fulu1}
\max\left\{(b((i+1)h-1)-a((i+1)h-1))^2,(b(ih)-a(ih))^2\right\}
\leq C\left(a^4(i)+b^4(i)\right).
\ena
Combining (\ref{fnlwlmv}) and the monotonicity of sequences $\{a(k),k\ge 0\}$ and $\{b(k),k\ge 0\}$ leads to
\begin{align}\label{fulu2}
 &\left(b((i+1)h-1)-a(ih)\right)^2\cr
 =& \left(b((i+1)h-1)-a((i+1)h-1)+a((i+1)h-1)-a(ih)\right)^2\cr
 =&\Bigg(b((i+1)h-1)-a((i+1)h-1)+\sum_{s=ih}^{(i+1)h-2}(a(s+1)-a(s))\Bigg)^2\cr
 \leq& h\Bigg((b((i+1)h-1)-a((i+1)h-1))^2+\sum_{s=ih}^{(i+1)h-2}(a(s)-a(s+1))^2\Bigg)\cr
\leq& h\Bigg(C\left(a^4((i+1)h-1)+b^4((i+1)h-1)
\right)+C\sum_{s=ih}^{(i+1)h-2}\left(a^4(s)+b^4(s)\right)\Bigg)\cr
\leq& Ch^2\left(a^4(i)+b^4(i)\right).
\end{align}
Similarly, we obtain
\bna\label{fulu3}
&&~~~~\left(a((i+1)h-1)-b(ih)\right)^2\cr
&&=\left(b(ih)-a((i+1)h-1)\right)^2\cr
&&=\left(b(ih)-a(ih)+a(ih)-a((i+1)h-1)\right)^2\cr
&&=\left(b(ih)-a(ih)+\sum_{s=ih}^{(i+1)h-2}(a(s)-a(s+1))\right)^2\cr
&&\leq h\left((b(ih)-a(ih))^2+\sum_{s=ih}^{(i+1)h-2}(a(s)-a(s+1))^2\right)\cr
&&\leq h\left(C\left(a^4(ih)+b^4(ih)\right)+C\sum_{s=ih}^{(i+1)h-2}
\left(a^4(s)+b^4(s)\right)\right) \cr
 &&\leq Ch^2\left(a^4(i)+b^4(i)\right).
\ena
Combining (\ref{fulu0}) and (\ref{fulu1})-(\ref{fulu3}) gives (\ref{fulu11}).
\end{proof}

\begin{lemma}\label{lemma1}
For the algorithm (\ref{algorithm}), suppose that Conditions \ref{condition1} and \ref{condition2} hold, and there exists an integer $h>0$ and a constant $\rho>0$, such that
\bna\label{nxsl}
\sup_{k\ge 0}\E\left.\left[\|\H^*(k)\H(k)\|^{2^{\max\{h,2\}}}\right|\F(k-1)\right]^{\frac{1}{2^{\max\{h,2\}}}}\leq \rho~\text{a.s.}
\ena
Then
(I) for any given integer $n\ge 0$ and $x\in L^0(\Omega,\F(nh-1);\X^N)$, there exists a constant $d_1>0$ such that
\bna\label{yinliwudianyi1}
\sup_{m\ge 0}\mathbb E\left.\left[\|\Phi_P(mh-1,nh)x\|^2\right|\mathcal F(nh-1)\right]\leq d_1\|x\|^2~\text{a.s.};
\ena
(II) for any given integer $n\ge 0$ and $x\in L^0(\Omega,\F(n-1);\X^N)$, then there exists a constant $d_2>0$ such that
\bna\label{yinliwudianyi2}
\displaystyle \sup_{0\leq m\leq n+h}\mathbb E\left[\|\Phi_P(m,n)x\|^2|\mathcal F(n-1)\right]\leq d_2\|x\|^2~\text{a.s.}
\ena
\end{lemma}

\begin{proof}
Denote $D(k)=a(k)\H^*(k)\H(k)+b(k)\L_{\G}\otimes I_{\X}$ and  $P(k)=I_{\X^N}-D(k)$. Noting that $\L_{\G}\otimes I_{\X}\ge 0$, we have $D(k)\ge 0~\text{a.s.}$ Let $x$ be a random element with values in Hilbert space $(\X^N,\tau_{\text{N}}(\X^N))$, we get
\bna\label{kkdsa}
&&~~~\mathbb \|\Phi_P(m,n)x\|^2\cr
&&=\langle\Phi_P(m,n)x,\Phi_P(m,n)x\rangle\cr
&&=\left\langle x,\Phi^*_P(m,n)\Phi_P(m,n)x\right\rangle\cr
&&=\left\langle x,(I_{\X^N}-D^*(n))\cdots(I_{\X^N}-D^*(m))(I_{\X^N}-D(m))\cdots (I_{\X^N}-D(n))x\right\rangle\cr
&&=\left\langle x,x-2\sum_{k=n}^mD(k)x+\sum_{k=2}^{2(m-n+1)}M_kx\right\rangle\cr
&&=\|x\|^2-2\sum_{k=n}^m\langle x,D(k)x\rangle+\left\langle x,\sum_{k=2}^{2(m-n+1)}M_kx\right\rangle\cr
&&\leq \|x\|^2+\left\langle x,\sum_{k=2}^{2(m-n+1)}M_kx\right\rangle
 \leq \Bigg(1+\sum_{k=2}^{2(m-n+1)}\|M_k\| \Bigg)\|x\|^2~\text{a.s.},~0\leq n \leq m,
\ena
where $M_k,k=2,\cdots,2(m-n+1)$ denote the $k$-th order terms in the binomial expansion of $\Phi_P^*(m,n)\Phi_P(m,n)$. For $0\leq m-n\leq h$, by the conditional Lyapunov inequality and (\ref{nxsl}), we obtain
\bna\label{jjkkss}
&&\hspace{-1.4cm}\sup_{k\ge 0}\mathbb E\left.\left[\|D(k)\|^i\right|\mathcal F(k-1)\right]\leq \sup_{k\ge 0}\mathbb E\left.\left[\|D(k)\|^{2^h}\right|\mathcal F(k-1) \right]^{\frac{i}{2h}}\leq \rho_0^i(a(k)+b(k))^i~\text{a.s.},\cr
&&~~~~~~~~~~~~~~~~~~~~~~~~~~~~~~~~~~~~~~~~~~~~~~~~~~~~~~~~~~~~~~~~~~~~~~~~~~~~~2\leq i\leq 2^h,
\ena
where $\rho_0:=\rho+\|\L_{\G}\|$. Note that
\bna\label{nkvvssk}
&&~~~~\mathbb E\left.\left[\|D(k)\|^i\right|\mathcal F(n-1) \right]\cr
&&=\mathbb E\left.\left.\left[\mathbb E\left[\|D(k)\|^i\right|\mathcal F(k-1)\right]\right|\mathcal F(n-1)\right],~2\leq i\leq 2^h,~k\ge n.
\ena
Since the real sequences $\{a(k),k\ge 0\}$ and $\{b(k),k\ge 0\}$ are both monotonically decreasing to $0$, it follows that there exists a constant $c_0>0$ such that $\sup_{k\ge 0}(a(k)+b(k))\leq c_0$. For $0\leq m-n\leq h$, from the definition of $M_k$ and (\ref{jjkkss})-(\ref{nkvvssk}), by termwise multiplication and using the H\"{o}lder inequality repeatedly, we have
\bna\label{pplds}
&&\hspace{-0.8cm}\mathbb E[\|M_k\||\mathcal F(n-1)]\leq \mathbb C_{2(m-n+1)}^k\rho_0^k(a(n)+b(n))^k\leq 2c_0^{-2}\left(a^2(n)+b^2(n)\right)\mathbb C_{2(m-n+1)}^k\rho_0^kc_0^k~\text{a.s.},\cr &&~~~~~~~~~~~~~~~~~~~~~~~~~~~~~~~~~~~~~~~~~~~~~~~~~~~~~~~~~~k=2,\cdots,2(m-n+1).
\ena
Denote $c_k=2c_0^{-2}(1+\rho_0c_0)^{2k}$, by (\ref{pplds}), we get
\bna\label{wwddsaf}
\sum_{k=2}^{2(m-n+1)}\mathbb E[\|M_k\||\mathcal F(n-1)]\leq c_{m-n+1}\left(a^2(n)+b^2(n)\right)~\text{a.s.}
\ena
If $x\in L^0(\Omega,\F(n-1);\X^N)$, substituting (\ref{wwddsaf}) into (\ref{kkdsa}) gives
\bna\label{kkxnms}
\mathbb E\left.\left[\|\Phi_P(m,n)x\|^2\right|\mathcal F(n-1)\right]\leq \left(1+c_{h+1}a^2(n)+c_{h+1}b^2(n)\right)\|x\|^2~\text{a.s.}
\ena
If $x\in L^0(\Omega,\F(ih-1);\X^N)$, it follows from (\ref{kkxnms}) that
\bna\label{eedd}
&&~~~~\mathbb E\left.\left[\|\Phi_P((i+1)h-1,ih)x\|^2\right|\mathcal F(ih-1)\right]\cr
&&\leq \left(1+c_ha^2(i)+c_hb^2(i)\right)\|x\|^2~\text{a.s.},~i\ge 0.
\ena
(I) For the given integer $n$ and $x\in L^0(\Omega,\F(nh-1);\X^N)$, by (\ref{eedd}), we obtain
\bna\label{vnklmef}
&&~~~~\mathbb E\left.\left[\|\Phi_P(mh-1,nh)x\|^2\right|\mathcal F(nh-1)\right]\cr
&&=\mathbb E\left.\left[\|\Phi_P(mh-1,(m-1)h)\Phi_P((m-1)h-1,nh)x\|^2\right|\mathcal F(nh-1)\right]\cr
&&=\mathbb E\big[\mathbb E\big[\|\Phi_P(mh-1,(m-1)h)\cr
&&~~~\times\Phi_P((m-1)h-1,nh)x\|^2|\mathcal F((m-1)h-1)\big]|\mathcal F(nh-1)\big]\cr
&&\leq \left(1+c_ha^2(m-1)+c_hb^2(m-1)\right)\mathbb E\left.\left[\|\Phi_P((m-1)h-1,nh)x\|^2\right|\mathcal F(nh-1)\right]\cr
&&\leq \prod_{k=n}^{m-1}\left(1+c_ha^2(k)+c_hb^2(k)\right)\|x\|^2~\text{a.s.}
\ena
Denote $d_1=\prod_{k=0}^{\infty}(1+c_ha^2(k)+c_hb^2(k))$, from (\ref{vnklmef}) and Condition \ref{condition2}, we get (\ref{yinliwudianyi1}).

(II) For the given integer $n$ and $x\in L^0(\Omega,\F(n-1);\X^N)$, it follows from Condition \ref{condition1} that there exists a constant $d_2>0$, such that $\sup_{k\ge 0}(1+c_{h+1}a^2(k)+c_{h+1}b^2(k))\leq d_2$, which together with (\ref{kkxnms}) gives (\ref{yinliwudianyi2}).
\end{proof}

\begin{lemma}\label{hhhlemma}
For the algorithm (\ref{algorithm}), if Conditions \ref{condition1} and \ref{condition2} hold, and there exists an integer $h>0$ and a constant $\rho>0$, such that
\ban
\sup_{k\ge 0}\E\left.\left[\|\H^*(k)\H(k)\|^{2^{\max\{h,2\}}}\right|\F(k-1)\right]^{\frac{1}{2^{\max\{h,2\}}}}\leq \rho~\text{a.s.},
\ean
then there exits a constant $d_3>0$, such that
\ban
&&~~~~\sup_{k\ge 0}\E\left[\left\|\left(\prod_{j=i+1}^k\left(I_{\X^N}-\sum_{s=jh}^{(j+1)h-1}(a(s)\H^*(s)\H(s)+b(s)(\L_{\G}\otimes I_{\X}))\right)\right)y\right\|^2\right]\cr
&&\leq d_3\E\left[\|y\|^2\right],~\forall\ i\ge 0,
\ean
for any given $y\in L^0(\Omega,\F((i+1)h-1);\X^N)$.
\end{lemma}

\begin{proof}
Denote $D'(k)=\sum_{s=kh}^{(k+1)h-1}(a(s)\H^*(s)\H(s)+b(s)(\L_{\G}\otimes I_{\X}))$ and $P'(k)=I_{\X^N}-D'(k)$. Noting that $\L_{\G}\otimes I_{\X}\ge 0$, we have $D'(k)\ge 0~\text{a.s.}$ Let $x\in L^0(\Omega,\F(nh^2-1);\X^N)$ be a random element with values in the Hilbert space $\X^N$. Then we get
\ban
&&\mathbb E\left.\left[\|\Phi_{P'}((n+1)h-1,nh)x\|^2\right|\mathcal F(nh^2-1)\right]
 \leq \left(1+\sum_{k=2}^{2h}\left.\mathbb E\left[\|M_k'\|\right|\mathcal F(nh^2-1)\right]\right)\|x\|^2,
\ean
where $M_k',k=2,\cdots,2h$ denote the $k$-th order terms in the binomial expansion of $\Phi_{P'}^*((n+1)h-1,nh)\Phi_{P'}((n+1)h-1,nh)$. It follows from the conditional Lyapunov inequality and Condtion \ref{condition1} that
\bna\label{jjkkssssfffw}
&&\hspace{-1.6cm}\sup_{k\ge 0}\mathbb E\left.\left[\|D'(k)\|^i\right|\mathcal F(k-1)\right]\leq  h\rho_0^i\sum_{s=kh}^{(k+1)h-1}(a(s)+b(s))^i\leq h^2\rho_0^i(a(n)+b(n))^i~\text{a.s.},\cr
&&~~~~~~~~~~~~~~~~~~~~~~~~~~~~~~~~~~~~~~~~~~~~~~~~~~~~~~~~~~~~~~~~~~~~~~~~~2\leq i\leq 2^h,
\ena
where $\rho_0=\rho+\|\L_{\G}\|$. Since $\{a(k),k\ge 0\}$ and $\{b(k),k\ge 0\}$ are both monotonically decreasing to  $0$, there exists a constant $c_0>0$ such that $\sup_{k\ge 0}(a(k)+b(k))\leq c_0$. From the definition of $M_k'$, by termwise multiplication and using the H\"{o}lder inequality of the conditional expectation, we have
\bna\label{ppldssss}
\mathbb E\left.\left[\left\|M_k'\right\|\right|\mathcal F(nh^2-1)\right]&\leq& h^2\mathbb C_{2h}^k\rho_0^k(a(n)+b(n))^k\cr
&\leq& 2h^2c_0^{-2}\left(a^2(n)+b^2(n)\right)\mathbb C_{2h}^k\rho_0^kc_0^k~\text{a.s.},~k=2,\cdots,2h.
\ena
Noting that (\ref{ppldssss}) and $x\in L^0(\Omega,\F(nh^2-1);\X^N)$, we get
\bna\label{eefffdddssd}
&&~~~~\mathbb E\left.\left[\|\Phi_{P'}((n+1)h-1,nh)x\|^2\right|\mathcal F(nh^2-1)\right]\cr
&&\leq \left(1+c'a^2(n)+c'b^2(n)\right)\|x\|^2~\text{a.s.},
\ena
where $c'=2h^2c_0^{-2}(1+\rho_0c_0)^{2h}$. By (\ref{eefffdddssd}), we obtain
\ban
&&~~~~\mathbb E\left.\left[\|\Phi_{P'}(mh-1,nh)x\|^2\right|\mathcal F(nh^2-1)\right]\cr
&&=\mathbb E\left.\left[\|\Phi_{P'}(mh-1,(m-1)h)\Phi_{P'}((m-1)h-1,nh)x\|^2\right|\mathcal F(nh^2-1)\right]\cr
&&=\mathbb E\big[\mathbb E\big[\|\Phi_{P'}(mh-1,(m-1)h)\cr
&&~~~\times\Phi_{P'}((m-1)h-1,nh)x\|^2|\mathcal F((m-1)h^2-1)\big]|\mathcal F(nh^2-1)\big]\cr
&&\leq \left(1+c'a^2(m-1)+c'b^2(m-1)\right)\mathbb E\left.\left[\|\Phi_{P'}((m-1)h-1,nh)x\|^2\right|\mathcal F(nh^2-1)\right]\cr
&&\leq \prod_{k=n}^{m-1}\left(1+c'a^2(k)+c'b^2(k)\right)\|x\|^2~\text{a.s.},~m>n\ge 0.
\ean
Denote $q_1=\prod_{k=0}^{\infty}(1+c'a^2(k)+c'b^2(k))$, it follows from Condition \ref{condition2} that
\ban
\sup_{m\ge 0}\mathbb E\left.\left[\|\Phi_{P'}(mh-1,nh)x\|^2\right|\mathcal F(nh^2-1)\right]\leq q_1\|x\|^2~\text{a.s.}
\ean
Following the same way as the proof of Lemma \ref{lemma1}, it shows that there exists a constant $q_2>0$, such that
\bna\label{leisileisi}
\sup_{0\leq m\leq n+h}\mathbb E\left.\left[\|\Phi_{P'}(m,n)x\|^2\right|\mathcal F(nh-1)\right]\leq q_2\|x\|^2~\text{a.s.}
\ena
For any given positive integer $j$, denote $m_j=\lfloor \frac{j}{h} \rfloor,\widetilde{m}_j=\lceil \frac{j}{h} \rceil$. Firstly, if $0\leq i<k-3h$, then $m_kh>\widetilde{m}_{i+1}h$. Let $y\in L^0(\Omega,\F((i+1)h-1);\X^N)$ be a random element with values in the Hilbert space $\X^N$. Noting that $0\leq k-m_kh<h$, $0\leq \widetilde{m}_{i+1}h-(i+1)<h$, $\Phi_{P'}(m_kh-1,\widetilde{m}_{i+1}h)\Phi_{P'}(\widetilde{m}_{i+1}h-1,i+1)y\in \F(m_kh^2-1)$ and $\Phi_{P'}(\widetilde{m}_{i+1}h-1,i+1)y\in \F(\widetilde{m}_{i+1}h-1)$, by (\ref{eefffdddssd})-(\ref{leisileisi}), we get
\bna\label{jjkklkkklll}
&&~~~~~~\mathbb E\left[\left\|\Phi_{P'}(k,i+1)y\right\|^2\right]\cr
&&~~=\mathbb E\left[\|\Phi_{P'}(k,m_kh)\Phi_{P'}(m_kh-1,\widetilde{m}_{i+1}h)\Phi_{P'}(\widetilde{m}_{i+1}h-1,i+1)y\|^2\right]\cr
&&~~\leq q_2\mathbb E\left[\|\Phi_{P'}(m_kh-1,\widetilde{m}_{i+1}h)\Phi_{P'}(\widetilde{m}_{i+1}h-1,i+1)y\|^2\right]\cr
&&~~=q_2\mathbb E\left[\E\left.\left[\|\Phi_{P'}(m_kh-1,\widetilde{m}_{i+1}h)\Phi_{P'}(\widetilde{m}_{i+1}h-1,i+1)y\|^2\right|\F(\widetilde{m}_{i+1}h-1)\right]\right]\cr
&&~~\leq q_1q_2\E\left[\|\Phi_{P'}(\widetilde{m}_{i+1}h-1,i+1)y\|^2\right]\cr
&&~~=q_1q_2\E\left[\E\left.\left[\|\Phi_{P'}(\widetilde{m}_{i+1}h-1,i+1)y\|^2\right|\F((i+1)h-1)\right]\right]\cr
&&~~\leq q_1q_2^2\E\left[\|y\|^2\right],~0\leq i<k-3h.
\ena
Secondly, it follows from (\ref{leisileisi}) that
\bna\label{kjifw}
&&~~~~\mathbb E\left[\|\Phi_{P'}(k,i+1)y\|^2\right]\cr
&&=\mathbb E\left[\mathbb E\left.\left[\|\Phi_{P'}(k,i+1)y\|^2\right|\mathcal F((i+1)h-1)\right]\right]
 \leq q_2\E\left[\|y\|^2\right],~k-h\leq i<k.
\ena
From (\ref{leisileisi}) and (\ref{kjifw}), it is known that
\bna\label{ikdjdf}
&&~~~~\mathbb E\left[\|\Phi_{P'}(k,i+1)y\|^2\right]\cr
&&=\mathbb E\left[\mathbb E\left.\left[\|\Phi_{P'}(k,i+1)y\|^2\right|\mathcal F((i+1)h-1)\right]\right]\cr
&&=\mathbb E\left[\mathbb E\left.\left[\|\Phi_{P'}(k,k-h+1)\Phi_{P'}(k-h,i+1)y\|^2\right|\mathcal F((i+1)h-1)\right]\right]\cr
&&\leq q_2\mathbb E\left[\mathbb E\left.\left[\|\Phi_{P'}(k-h,i+1)y\|^2\right|\mathcal F((i+1)h-1)\right]\right]\cr
&&\leq q_2^2\E\left[\|y\|^2\right],~k-2h\leq i<k-h.
\ena
Finally, it follows from (\ref{leisileisi}) and (\ref{ikdjdf}) that
\bna\label{wdkkdd}
&&~~~~\mathbb E\left[\|\Phi_{P'}(k,i+1)y\|^2\right]\cr
&&=\mathbb E\left[\mathbb E\left.\left[\|\Phi_{P'}(k,i+1)y\|^2\right|\mathcal F((i+1)h-1)\right]\right]\cr
&&=\mathbb E\left[\mathbb E\left.\left[\|\Phi_{P'}(k,k-h+1)\Phi_{P'}(k-h,i+1)y\|^2\right|\mathcal F((i+1)h-1)\right]\right]\cr
&&\leq q_2\mathbb E\left[\mathbb E\left.\left[\|\Phi_{P'}(k-h,i+1)y\|^2\right|\mathcal F((i+1)h-1)\right]\right]\cr
&&\leq q_2^3\E\left[\|y\|^2\right],~k-3h\leq i<k-2h.
\ena
Combining (\ref{kjifw})-(\ref{wdkkdd}), we get
\bna\label{llcck}
\mathbb E\left[\|\Phi_{P'}(k,i+1)y\|^2\right]\leq \max\left\{q_2,q_2^2,q_2^3\right\}\E\left[\|y\|^2\right],~0<k-i\leq 3h.
\ena
Denote $d_3=\max\{q_1q_2^2,q_2,q_2^2,q_2^3\}$. By (\ref{jjkklkkklll}) and (\ref{llcck}), we have
$
\sup_{k\ge 0}\mathbb E\left[\|\Phi_{P'}(k,i+1)y\|^2\right]\leq d_3\E\left[\|y\|^2\right],~\forall\ i\ge 0.
$
\end{proof}

\begin{lemma}\label{henandelemma}
For the algorithm (\ref{algorithm}), suppose that  Assumptions \ref{assumption1}, \ref{assumption2}, Conditions \ref{condition1} and \ref{condition2} hold, there exists an integer $h>0$, a constant $\rho_0>0$, a strictly positive self-adjoint operator $\HH\in \mathscr L(\mathscr X^N)$ and a nonnegative real sequence $\{c(k),k\ge 0\}$, respectively, satisfying the following conditions: \\
(i) for any given $L_2$-bounded adaptive sequence $\{x(k),\F(kh-1),k\ge 0\}$ with values in the Hilbert space $\X^N$,
\ban \lim_{k\to\infty}\sum_{i=m}^{k}\left(\E\left[\left\|\left(\prod_{j=i+1}^k(I_{\X^N}-c(j)\HH)\right)\mu(i)\right\|^2\right]\right)^{\frac{1}{2}}=0,~\forall\ m\in \mathbb N,\quad \sum_{k=0}^{\infty}c(k)=\infty,
\ean
where $\displaystyle\mu(i):=c(i) \mathscr{H} x(i)-\sum_{s=i h}^{(i+1) h-1}\left(a(s) \mathbb{E}\left[\mathcal{H}^{*}(s) \mathcal{H}(s) x(i) \mid \mathcal{F}(i h-1)\right]+b(s)\left(\mathcal{L}_{\mathcal{G}} \otimes I_{\mathscr{X}}\right) x(i)\right)$;\\
(ii) $\displaystyle \sup_{k\ge 0}\left(\E\left.\left[\|\H^*(k)\H(k)\|^{2^{\max\{h,2\}}}\right|\F(k-1)\right]\right)^{\frac{1}{2^{\max\{h,2\}}}}\leq \rho_0~\text{a.s.}
$\\
Then the sequence of operator-valued random elements $$ \left\{I_{\X^N}-\sum_{i=kh}^{(k+1)h-1}(a(i)\H^*(i)\H(i)+b(i)\L_{\G}\otimes I_{\X}),k\ge 0\right\}$$ is $L_2^2$-stable w.r.t.  $\{\F((k+1)h-1),k\ge 0\}$.
\end{lemma}

\begin{proof}
Let $\{x(k),\F(kh-1),k\ge 0\}$ be a $L_2$-bounded adaptive sequence with values in the Hilbert space $\X^N$. For any given integer $m\geq0$, let
\bna\label{diedaishii}
&&~~~~w(k+1)\cr
&&=\left(I_{\X^N}-\sum_{s=kh}^{(k+1)h-1}\left(a(s)\H^*(s)\H(s)+b(s)\L_{\G}\otimes I_{\X}\right)\right)w(k),~k\ge m,
\ena
where $w(m)=x(m),w(i)=0,i=0,\cdots,m-1$. It follows from Proposition \ref{nlllwwieiie}.(a)-(c) that $\{w(k),k\ge 0\}$ is a random sequence with values in the Hilbert space $(\X^N,\tau_{\text{N}}(\X^N))$. On one hand, from the definition of $w(k)$, it follows that
\bna\label{wffee}
&&~~~~w(k+1)\cr
&&=\left(\prod_{i=m}^k\left(I_{\X^N}-\sum_{s=ih}^{(i+1)h-1}(a(s)\H^*(s)\H(s)+b(s)\L_{\G}\otimes I_{\X})\right)\right)x(m),~k\ge m.
\ena
Noting that $x(m)\in L^0(\Omega,\F(mh-1);\X^N)$, it follows from Lemma \ref{hhhlemma} that there exists a constant $d_3>0$, such that
\bna\label{lsaew}
\sup_{k\ge 0}\E\left[\|w(k+1)\|^2\right]
\leq d_3\E\left[\|x(m)\|^2\right]
\leq d_3\sup_{k\ge 0}\E\left[\|x(k)\|^2\right]<\infty.
\ena
Moreover, (\ref{diedaishii}) can be rewritten as
\ban
w(i+1)&=&(I_{\X^N}-c(i)\HH)w(i)\cr
&&+\left(c(i)\HH-\sum_{s=ih}^{(i+1)h-1}(a(s)\H^*(s)\H(s)+b(s)\L_{\G}\otimes I_{\X})\right)w(i),~i\ge m,
\ean
which gives
\bna\label{fwwiii}
&&~~~~w(k+1)\cr
&&=\left(\prod_{i=m}^k(I_{\X^N}-c(i)\HH)\right)x(m)+\sum_{i=m}^k\left(\prod_{j=i+1}^k(I_{\X^N}-c(j)\HH)\right)\cr
&&~~~\times\left(c(i)\HH-\sum_{s=ih}^{(i+1)h-1}(a(s)\H^*(s)\H(s)+b(s)\L_{\G}\otimes I_{\X})\right)w(i),~k\ge m.
\ena
From (\ref{lsaew}), Proposition \ref{wenknknkn} and the condition (ii), it is known that
\ban
\sup_{s\ge 0,\ i\ge 0}\E\left[\|\H^*(s)\H(s)w(i)\|\right]\leq \sup_{s\ge 0}\E\left[\|\H^*(s)\H(s)\|^2\right]+\sup_{i\ge 0}\E\left[\|w(i)\|^2\right]<\infty,
\ean
that is, $\H^*(s)\H(s)w(i)\in L^1(\Omega;\X^N)$. Thus, by Lemma \ref{nvkvpeoeo}, we know that there exists a unique conditional expectation $\E[\H^*(s)\H(s)w(i)|\F(ih-1)]$ of $\H^*(s)\H(s)w(i)$ w.r.t. $\F(ih-1)$. Then we have
\bna\label{fewdee}
~~\Bigg(c(i)\HH-\sum_{s=ih}^{(i+1)h-1}(a(s)\H^*(s)\H(s)+b(s)\L_{\G}\otimes I_{\X})\Bigg)w(i)
=w_1(i)+w_2(i),
\ena
where
\bna\label{mlwer}
\begin{cases}
  w_1(i)=\displaystyle c(i)\HH w(i)-\sum_{s=ih}^{(i+1)h-1}(a(s)\E[\H^*(s)\H(s)w(i)|
  \F(ih-1)] +b(s)(\L_{\G}\otimes I_{\X})w(i)),\\
  w_2(i)=\displaystyle \sum_{s=ih}^{(i+1)h-1}a(s)(\E[\H^*(s)\H(s)w(i)|\F(ih-1)]-\H^*(s)\H(s)w(i)).
\end{cases}
\ena
By (\ref{wffee}) and (\ref{fwwiii})-(\ref{fewdee}), we get
\ban
&&~~~\left(\prod_{i=m}^k\left(I_{\X^N}-\sum_{s=ih}^{(i+1)h-1}(a(s)\H^*(s)\H(s)+b(s)\L_{\G}\otimes I_{\X})\right)\right)x(m)\cr &&=\left(\prod_{i=m}^k(I_{\X^N}-c(i)\HH)\right)x(m)+\sum_{i=m}^k\left(\prod_{j=i+1}^k(I_{\X^N}-c(j)\HH)\right)(w_1(i)+w_2(i)),
\ean
which together with Cauchy inequality leads to
\bna\label{ikddww}
&&~~~~\E\left[\Bigg\|\Bigg(\prod_{i=m}^k\Bigg(I_{\X^N}-\sum_{s=ih}^{(i+1)h-1}(a(s)\H^*(s)\H(s)+b(s)\L_{\G}\otimes I_{\X})\Bigg)\Bigg)x(m)\Bigg\|^2\right]\cr
&&\leq 2\E\left[\Bigg\|\Bigg(\prod_{i=m}^k(I_{\X^N}-c(i)\HH)\Bigg)x(m)\Bigg\|^2\right]\cr &&~~~+2\E\left[\left\|\sum_{i=m}^k\left(\prod_{j=i+1}^k(I_{\X^N}-c(j)\HH)\right)(w_1(i)+w_2(i))\right\|^2\right],~k\ge m.
\ena
By Lemma \ref{lemmaA7}, it is known that there exist constants $M,d>0$, such that
\ban
\left\|\left(\prod_{i=m}^k(I_{\X^N}-c(i)\HH)\right)x(m)\right\|^2\leq M^{2d}\|x(m)\|^2~\text{a.s.}
\ean
Noting that $\sup_{k\ge 0}\E[\|x(k)\|^2]<\infty$, it follows from Lebesgue dominated convergence theorem and Lemma \ref{lemmaA7} that
\bna\label{oo1}
\lim_{k\to\infty}\E\left[\left\|\left(\prod_{i=m}^k(I_{\X^N}-c(i)\HH)\right)x(m)\right\|^2\right]=0,~\forall\ m\ge 0.
\ena
By using Cauchy inequality again, we obtain
\bna\label{oooff}
&&~~~~\E\left[\left\|\sum_{i=m}^k\left(\prod_{j=i+1}^k(I_{\X^N}-c(j)\HH)\right)(w_1(i)+w_2(i))\right\|^2\right]\cr
&&\leq 2\E\left[\left\|\sum_{i=m}^k\left(\prod_{j=i+1}^k(I_{\X^N}-c(j)\HH)\right)w_1(i)\right\|^2\right]\cr &&~~~+2\E\left[\left\|\sum_{i=m}^k\left(\prod_{j=i+1}^k(I_{\X^N}-c(j)\HH)\right)w_2(i)\right\|^2\right].
\ena
We now consider the right-hand side of (\ref{oooff}) term by term. Firstly, by Minkowski inequality, we get
\bna\label{oolkf}
&&~~~~\E\left[\left\|\sum_{i=m}^k\left(\prod_{j=i+1}^k(I_{\X^N}-c(j)\HH)\right)w_1(i)\right\|^2\right]\cr
&&\leq \left(\sum_{i=m}^k\left(\E\left[\left\|\left(\prod_{j=i+1}^k(I_{\X^N}-c(j)\HH)\right)w_1(i)\right\|^2\right]\right)^{\frac{1}{2}}\right)^2.
\ena
It follows from (\ref{diedaishii}) and (\ref{lsaew}) that $\{w(k),\F(kh-1),k\ge 0\}$ is an adaptive sequence with $\sup_{k\ge 0}\E[\|w(k)\|^2]<\infty$, which together with (\ref{mlwer}), the condition (i) and (\ref{oolkf}) gives
\bna\label{vspppw}
\lim_{k\to \infty}\E\left[\left\|\sum_{i=m}^k\left(\prod_{j=i+1}^k(I_{\X^N}-c(j)\HH)\right)w_1(i)\right\|^2\right]=0.
\ena
Secondly, noting that $w_2(i)\in \F((i+1)h-1)$, $w(i)\in \F(ih-1)$, by Condition \ref{condition1}, it follows from the condition (ii) and (\ref{lsaew}) that
\bna\label{fuckd}
&&~~~~\E\left[\|w_2(i)\|^2\right]\cr
&&= \E\left[\left\|\sum_{s=ih}^{(i+1)h-1}a(s)(\E[\H^*(s)\H(s)w(i)|\F(ih-1)]-\H^*(s)\H(s)w(i))\right\|^2\right]\cr
&&\leq ha^2(i)\E\Bigg[\sum_{s=ih}^{(i+1)h-1}\|\E[\H^*(s)\H(s)w(i)|\F(ih-1)]
-\H^*(s)\H(s)w(i)\|^2\Bigg]\cr
&&\leq 2ha^2(i)\E\Bigg[\sum_{s=ih}^{(i+1)
h-1}\big(\E\big[\|\H^*(s)\H(s)w(i)\|^2|\F(ih-1)\big]
 +\|\H^*(s)\H(s)w(i)\|^2\big)\Bigg]\cr
&&\leq 2ha^2(i)\E\Bigg[\sum_{s=ih}^{(i+1)h-1}\left(\rho_0^2\|w(i)\|^2
+\|\H^*(s)\H(s)w(i)\|^2\right)\Bigg]\cr
&&\leq 2ha^2(i)\E\Bigg[\|w(i)\|^2\sum_{s=ih}^{(i+1)h-1}\left(\rho_0^2
+\|\H^*(s)\H(s)\|^2\right)\Bigg]\cr
&&= 2ha^2(i)\E\Bigg[\E\Bigg[\|w(i)\|^2\sum_{s=ih}^{(i+1)h-1}
\left(\rho_0^2+\|\H^*(s)\H(s)\|^2\right) \Bigg|\F(ih-1)\Bigg]\Bigg]\cr
&&= 2ha^2(i)\E\Bigg[\E\Bigg[\sum_{s=ih}^{(i+1)h-1}\left(\rho_0^2
+\|\H^*(s)\H(s)\|^2\right)\Bigg|\F(ih-1)\Bigg]\|w(i)\|^2\Bigg]\cr
&&\leq 4h^2\rho_0^2a^2(i)\E\left[\|w(i)\|^2\right],
\ena
which together with Condition \ref{condition2} gives
\bna\label{vlllls}
\sup_{i\ge 0}\E\left[\|w_2(i)\|^2\right]<\infty.
\ena
Thus, it follows from Lemma \ref{nvkvpeoeo} that $\E[w_2(i)|\F(ih-1)]$ exists and
\bna\label{iwsf}
&&~~~~\E[w_2(i)|\F(ih-1)]\cr
&&=\E\Bigg[\sum_{s=ih}^{(i+1)h-1}a(s)(\E[\H^*(s)\H(s)w(i)|\F(ih-1)] -\H^*(s)\H(s)w(i))\bigg|\F(ih-1)\Bigg]\cr
&&=\sum_{s=ih}^{(i+1)h-1}a(s)\E[\E[\H^*(s)\H(s)w(i)|\F(ih-1)]-\H^*(s)
\H(s)w(i)|\F(ih-1)] \cr
&& =0.
\ena
Meanwhile, from Lemma \ref{lemmaA7}, it is known that there exist constants $M,d>0$ such that
\begin{align}\label{nvweefff}
 &\sup_{\|x\|=1\atop x\in \X^N}\inf\left\{r\ge 0:\P\left(\Bigg\|\Bigg(\prod_{j=t+1}^k(I_{\X^N}-c(j)\HH)\Bigg)x\Bigg\|
 <r\right)=1\right\}\notag\\
 \leq &  \sup_{\|x\|=1\atop x\in \X^N}M^d\|x\|<\infty.
\end{align}
For $m\leq s<t\leq k$, it follows from Proposition \ref{wenknknkn}, Proposition 2.6.31 in \cite{hy}, Lemma 3.5.2 in \cite{hy}, Proposition \ref{lemmaA6} and (\ref{vlllls})-(\ref{nvweefff}) that
\bna\label{ffkk}
~~~~\E\left[\left\langle \left(\prod_{j=s+1}^k(I_{\X^N}-c(j)\HH)\right)w_2(s),\left(\prod_{j=t+1}^k(I_{\X^N}-c(j)\HH)\right)w_2(t)\right\rangle\right]
=0.
\ena
On one hand, by Lemma \ref{lemmaA7}, it is known that there exist positive constants $M$ and $d$, such that
\bna\label{wiiiwww}
\left\|\left(\prod_{j=i+1}^k(I_{\X^N}-c(j)\HH)\right)\frac{1}{a(i)}w_2(i)\right\|\leq M^d\left\|\frac{1}{a(i)}w_2(i)\right\|~\text{a.s.}
\ena
Thus, by Lemma \ref{lemmaA7} and (\ref{ffkk})-(\ref{wiiiwww}), we get
\begin{align}\label{oolll}
&\E\left[\left\|\sum_{i=m}^k\left(\prod_{j=i+1}^k(I_{\X^N}-c(j)\HH)\right)w_2(i)\right\|^2\right]\cr
 =&\sum_{i=m}^k\E\left[\left\|\left(\prod_{j=i+1}^k(I_{\X^N}-c(j)\HH)\right)
 w_2(i)\right\|^2\right]\cr & +2\sum_{m\leq s<t\leq k}\E\Bigg[\Bigg\langle \left(\prod_{j=s+1}^k(I_{\X^N}-c(j)\HH)\right)w_2(s), \left(\prod_{j=t+1}^k(I_{\X^N}-c(j)\HH)\right)w_2(t)\Bigg\rangle\Bigg]\cr
 =&\sum_{i=m}^k\E\left[\left\|\left(\prod_{j=i+1}^k(I_{\X^N}-c(j)\HH)\right)w_2(i)\right\|^2\right]\cr
=&\sum_{i=m}^ka^2(i)\E\left[\left\|\left(\prod_{j=i+1}^k(I_{\X^N}-c(j)\HH)\right)\frac{1}{a(i)}w_2(i)\right\|^2\right]\cr
 \leq& M^d\sum_{i=m}^ka^2(i)
\left[\E\left[\left\|\left(\prod_{j=i+1}^k(I_{\X^N}-c(j)\HH)\right)
\frac{1}{a(i)}w_2(i)\right\|^2\right]\right]^{\frac{1}{2}} \cr & \times\left[\E\left[\left\|\frac{1}{a(i)}w_2(i)\right\|^2\right]\right]^{\frac{1}{2}}.
\end{align}
On the other hand, by (\ref{lsaew}) and (\ref{fuckd}), we have
\bna\label{ofskj}
&&\sup_{i\ge 0}\mathbb E\left[\left\|\frac{1}{a(i)}w_2(i)\right\|^2\right]\cr
&\leq & 4h^2\rho_0^2\sup_{i\ge 0}\E\left[\|w(i)\|^2\right]<\infty.
\ena
Substituting (\ref{ofskj}) into (\ref{oolll}) gives
\ban
&&~~~~\E\left[\left\|\sum_{i=m}^k\left(\prod_{j=i+1}^k(I_{\X^N}-c(j)\HH)\right)w_2(i)\right\|^2\right]\cr
&&\leq M^d\sup_{i\ge 0}\left[\E\left[\left\|\frac{1}{a(i)}w_2(i)\right\|^2\right]\right]^{\frac{1}{2}}\cr
&&~~~~\times\sum_{i=m}^ka^2(i)\left(\E\left[\left\|\left(\prod_{j=i+1}^k(I_{\X^N}-c(j)\HH)\right)\frac{1}{a(i)}w_2(i)\right\|^2\right]\right)^{\frac{1}{2}},
\ean
which together with Condition \ref{condition2}, (\ref{ofskj}) and Lemma \ref{lemmaA8} leads to
\bna\label{iikdd}
\lim_{k\to\infty}\E\left[\left\|\sum_{i=m}^k\left(\prod_{j=i+1}^k(I_{\X^N}-c(j)\HH)\right)w_2(i)\right\|^2\right]=0,~\forall\ m\ge 0.
\ena
Hence, substituting (\ref{oo1})-(\ref{oooff}), (\ref{vspppw}) and (\ref{iikdd}) into (\ref{ikddww}) completes the proof of Lemma \ref{henandelemma}.
\end{proof}
\end{appendices}

\section*{Acknowledgment}
The authors thank Mr. Hanran Du for the related discussion and revision for Appendix B.

\end{CJK}
\end{document}